\documentclass{article}
\usepackage[accepted]{icml2019}

\usepackage{algorithm, algorithmic}
\usepackage{amsmath, amsthm, amssymb}
\usepackage{relsize}
\usepackage{color}
\usepackage{xcolor}
\usepackage{subcaption}
\usepackage{enumitem}
\usepackage[colorlinks=true,linkcolor=blue,urlcolor=blue,citecolor=blue]{hyperref} 
\usepackage{multirow}

\usepackage{microtype}
\usepackage{graphicx}
\usepackage{booktabs}

\usepackage{colortbl}

\newtheoremstyle{examplestyle}
  {1.1\topsep} 
  {1.1\topsep} 
  {\itshape} 
  {} 
  {\bfseries} 
  {.} 
  {0.5em} 
  {} 
\theoremstyle{examplestyle}

\newtheorem{theorem}{Theorem}
\newtheorem{lemma}{Lemma}
\newtheorem{proposition}{Proposition}
\newtheorem{remark}{Remark}

\newcommand{\rnum}[1]{\expandafter{\romannumeral #1\relax}}
\newcommand{\RNum}[1]{\uppercase\expandafter{\romannumeral #1\relax}}

\newcommand{\tpdv}[2]{\frac{\partial{#1}}{\partial{#2}}}
\newcommand{\pdv}[2]{\frac{\partial{#1}}{\partial{#2}}}

\def\CS{{\mathcal S}} 
\def\CP{{\mathcal P}}
\def\BR{{\mathbb R}}
\def\BN{{\mathbb N}}
\newcommand{\E}{\mathbb E}
\newcommand{\mE}{\mathbb E}

\DeclareMathOperator*{\argmin}{arg\,min}
\newcommand{\ff}{f^*\!\!~}

\newcommand\Tstrut{\rule{0pt}{2.0ex}}         
\begin{document}


\twocolumn[
\icmltitle{Lipschitz Generative Adversarial Nets}



\icmlsetsymbol{equal}{*}

\begin{icmlauthorlist}
\icmlauthor{Zhiming Zhou}{sjtu}
\icmlauthor{Jiadong Liang}{pku}
\icmlauthor{Yuxuan Song}{sjtu}
\icmlauthor{Lantao Yu}{stanford}
\icmlauthor{Hongwei Wang}{stanford}
\icmlauthor{Weinan Zhang}{sjtu}
\icmlauthor{Yong Yu}{sjtu}
\icmlauthor{Zhihua Zhang}{pku}
\end{icmlauthorlist}

\icmlaffiliation{sjtu}{Shanghai Jiao Tong University}
\icmlaffiliation{pku}{Peking University}
\icmlaffiliation{stanford}{Stanford University}

\icmlcorrespondingauthor{Zhiming Zhou}{heyohai@apex.sjtu.edu.cn}

\icmlkeywords{Lipschitz-Continuity, Generative Adversarial Nets, Convergence, Objective}

\vskip 0.3in
]

\printAffiliationsAndNotice{}

\begin{abstract}
In this paper we show that generative adversarial networks (GANs) without restriction on the discriminative function space commonly suffer from the problem that the gradient produced by the discriminator is uninformative to guide the generator. By contrast, Wasserstein GAN (WGAN), where the discriminative function is restricted to $1$-Lipschitz, does not suffer from such a gradient uninformativeness problem. We further show in the paper that the model with a compact dual form of Wasserstein distance, where the Lipschitz condition is relaxed, may also theoretically suffer from this issue. This implies the importance of Lipschitz condition and motivates us to study the general formulation of GANs with Lipschitz constraint, which leads to a new family of GANs that we call Lipschitz GANs (LGANs). We show that LGANs guarantee the existence and uniqueness of the optimal discriminative function as well as the existence of a unique Nash equilibrium. We prove that LGANs are generally capable of eliminating the gradient uninformativeness problem. According to our empirical analysis, LGANs are more stable and generate consistently higher quality samples compared with WGAN.
\end{abstract} 

\begin{table*}[ht]
\caption{Comparison of different objectives in GANs.}
\label{table1}
\centering
\vspace{-2pt}
\resizebox{1.0\textwidth}{!}{
\begin{tabular}{|c|c|c|c|c|c|c|c|}
\hline 
                    & \multirow{2}{*}{$\phi$}                & \multirow{2}{*}{$\varphi$}            & \multirow{2}{*}{$\mathcal{F}$}                   & \multirow{2}{*}{$\ff(x)$}  & Gradient   & Gradient   & $\ff(x)$  \\ 
                  & & & & & Vanishing & Uninformative & Uniqueness \\ 
\hline 
Vanilla GAN       & $-\log(\sigma(-x))$  & $-\log(\sigma(x))$  & $\{f\colon \BR^n \to \BR \}$    & $\log\frac{\CP_r(x)}{\CP_g(x)}$  
& \textcolor{red}{Yes} & \textcolor{red}{Yes} & \textcolor{blue}{Yes} \\
\hline
Least-Squares GAN       & $(x-\alpha)^2$        & $(x-\beta)^2$       & $\{f \colon  \BR^n \to  \BR \}$    & $\frac{\alpha \cdot \CP_g(x)+\beta \cdot \CP_r(x)}{\CP_r(x)+\CP_g(x)}$ 
& \textcolor{blue}{No} & \textcolor{red}{Yes} & \textcolor{blue}{Yes}  \\ 
\hline 
$\mu$-Fisher GAN    & $x$                   & $-x$                & $\{f \colon \BR^n \to \BR, ~ \E_{x\sim \mu}    |f(x)|^2 \leq 1 \}$   & $\frac{1}{\mathcal{F}_{\mu}(\CP_r,\CP_g)}\frac{\CP_r(x)-\CP_g(x)}{\mu(x)}$ 
& \textcolor{blue}{No} & \textcolor{red}{Yes} & \textcolor{blue}{Yes} \\
\hline 
Wasserstein GAN      & $x$                   & $-x$                & $\{f \colon \BR^n \to \BR, ~ k(f) \leq 1 \}$ & $\mathsmaller{N/A}$  
& \textcolor{blue}{No} & \textcolor{blue}{No} & \textcolor{red}{No} \\
\hline 
Lipschitz GAN       & \multicolumn{2}{|c|}{any $\phi$ and $\varphi$ satisfying Eq.~(\ref{eq_solvable})}  & $\{f \colon \BR^n \to \BR \}$; $k(f)$ is penalized & $\mathsmaller{N/A}$  
& \textcolor{blue}{No} & \textcolor{blue}{No} &  \textcolor{blue}{Yes} \\
\hline 
\end{tabular}
}
\end{table*}

\section{Introduction}
Generative adversarial networks (GANs) \citep{gan}, as one of the most successful generative models, have shown promising results in various challenging tasks. GANs are popular and widely used, but they are notoriously hard to train \citep{gan_tutorial}. The underlying obstacles, though have been heavily studied \citep{principled_methods,lucic2017gans,heusel2017gans,mescheder2017numerics,mescheder2018training,yadav2017stabilizing}, are still not fully understood. 

The objective of GAN is usually defined as a distance metric between the real distribution $\CP_r$ and the generative distribution $\CP_g$, which implies that $\CP_r=\CP_g$ is the unique global optimum. The nonconvergence of traditional GANs has been considered as a result of ill-behaving distance metric \cite{principled_methods}, i.e., the distance between $\CP_r$ and $\CP_g$ keeps constant when their supports are disjoint. 
\citet{wgan} accordingly suggested using the Wasserstein distance, which can properly measure the distance between two distributions no matter whether their supports are disjoint. 

In this paper, we conduct a further study on the convergence of GANs from the perspective of the informativeness of the gradient of \emph{the optimal discriminative function $\ff$}. We show that for GANs that have no restriction on the discriminative function space, e.g., the vanilla GAN and its most variants, $\ff(x)$ is only related to the densities of the local point $x$ and does not reflect any information about 
other points in the distributions. We demonstrate that under these circumstances, the gradient of the optimal discriminative function with respect to its input, on which the generator updates generated samples, usually tells nothing about the real distribution. We refer to this phenomenon as the \emph{gradient uninformativeness}, which is substantially different from the gradient vanishing and is a fundamental cause of nonconvergence of GANs. 

According to the analysis of \citet{wgangp}, Wasserstein GAN can avoid the gradient uninformativeness problem. Meanwhile, we show in the paper that the Lipschitz constraint in the Kantorovich-Rubinstein dual of the Wasserstein distance can be relaxed, leading to a new equivalent dual; and with the new dual form, the gradient may also not reflect any information about how to refine $\CP_g$ towards $\CP_r$. It suggests that Lipschitz condition would be a vital element for resolving the gradient uninformativeness problem.

Motivated by the above analysis, we investigate the general formulation of GANs with Lipschitz constraint. We show that under a mild condition, penalizing Lipschitz constant guarantees the existence and uniqueness of the optimal discriminative function as well as the existence of the unique Nash equilibrium between $\ff$ and $\CP_g$ where $\CP_r=\CP_g$. It leads to a new family of GANs that we call Lipschitz GANs (LGANs). We show that LGANs are generally capable of eliminating the gradient uninformativeness in the manner that with the optimal discriminative function, the gradient for each generated sample, if nonzero, will point towards some real sample. This process continues until the Nash equilibrium $\CP_r=\CP_g$ is reached. 

The remainder of this paper is organized as follows. In Section~\ref{sec_prelims}, we provide some preliminaries that will be used in this paper. In Section~\ref{sec_gradient_uninformative}, we study the gradient uninformativeness issue in detail. In Section~\ref{sec_lip}, we present LGANs and their theoretical analysis. We conduct the empirical analysis in Section~\ref{sec_exp}. Finally, we discuss related work in Section~\ref{sec_related_work} and conclude the paper in Section~\ref{sec_conclusion}. 

\section{Preliminaries}
\label{sec_prelims}

In this section we first give some notions and then present a general formulation for generative  adversarial networks. 

\subsection{Notation and Notions}

Given two metric spaces $(X, d_X)$ and $(Y, d_Y)$, a function $f\colon X \to Y$ is said to be Lipschitz continuous if there exists a constant $k\geq 0$ such that
\begin{equation} \label{eq_lip}
d_Y(f(x_1), f(x_2)) \leq k \cdot d_X(x_1, x_2), \forall \; x_1, x_2 \in X.
\end{equation}
In this paper and in most existing GANs, the metrics $d_X$ and $d_Y$ are by default Euclidean distance which we also denote by $\lVert \cdot \rVert$. The smallest constant $k$ is called the (best)
Lipschitz constant of  $f$,   denoted by $k(f)$. 

The first-order Wasserstein distance $W_1$ between two probability distributions is defined as
\begin{equation}\label{eq_w_primal}
W_1(\CP_r, \CP_g) =  \inf_{\pi \in \Pi(\CP_r,\CP_g)} \, \E_{(x,y) \sim \pi} \, [d(x, y)],
\end{equation}
where $\Pi(\CP_r, \CP_g)$ denotes the set of all probability measures with marginals $\CP_r$ and $\CP_g$. It can be interpreted as the minimum cost of transporting the distribution $\CP_g$ to the distribution $\CP_r$. We use $\pi^*$ to denote the optimal transport plan, and let $\CS_r$ and $\CS_g$ denote the supports of $\CP_r$ and $\CP_g$, respectively. We say two distributions are disjoint if their supports are disjoint. 

The Kantorovich-Rubinstein (KR) duality \cite{oldandnew} provides a way of more efficiently computing of Wasserstein distance. The duality states that
\begin{equation} \label{eq_w_dual_lip}
\begin{aligned}
W_1(\CP_r, \CP_g) &= {\sup}_{f} \,\, \E_{x \sim \CP_r} \, [f(x)] - \E_{x \sim \CP_g} \, [f(x)],  \, \\
&\emph{s.t.} \, f(x) - f(y) \leq d(x, y), \,\, \forall x, \forall y. 
\end{aligned}
\end{equation}
The constraint in Eq.~\eqref{eq_w_dual_lip} implies that $f$ is Lipschitz continuous with $k(f)\leq 1$. Interestingly, we have a more compact dual form of the Wasserstein distance. That is, 
\begin{equation}
\begin{aligned}
W_1(\CP_r,\CP_g) = {\sup}_{\mathsmaller{f}} \,\, \E_{x \sim \CP_r} \, [f(x)] - \E_{x \sim \CP_g} \, [f(x)], \, \\
\emph{s.t.} \, f(x) - f(y) \leq d(x, y), \,\, \forall x \in \CS_r, \forall y \in \CS_g. 
\end{aligned}
\label{eq_w_dual_form_1}
\end{equation}
The proof for this  dual form is given in Appendix~\ref{app_dual_form}.
We see that this new dual relaxes
the Lipschitz continuity condition of the  dual form in Eq.~\eqref{eq_w_dual_lip}. 

\subsection{Generative Adversarial Networks (GANs)}

Typically,  GANs can be formulated as
\begin{equation} \label{eq_gan_formulation}
\begin{aligned}
	&\min_{f \in \mathcal{F}} \ J_D \triangleq \E_{z \sim \CP_z} [ \phi(f(g(z))) ] + \E_{x \sim \CP_r} [ \varphi(f(x)) ], \\ 
	&\min_{g \in \mathcal{G}} \ J_G \triangleq \E_{z \sim \CP_z} [ \psi(f(g(z))) ],
\end{aligned}
\end{equation}
where $\CP_z$ is the source distribution of the generator in $\BR^m$ and $\CP_r$ is the target (real) distribution in $\BR^n$. The generative function $g \colon \BR^m\to  \BR^n$ learns to output samples that share the same dimension as samples in $\CP_r$, while the discriminative function $f\colon  \BR^n \to \BR$ learns to output a score indicating the authenticity of a given sample. Here
$\mathcal{F}$ and $\mathcal{G}$ denote discriminative and generative function spaces, respectively; and $\phi$, $\varphi$, $\psi$: $\BR \rightarrow \BR$ are loss metrics. 
We denote the implicit distribution of the generated samples by $\CP_g$. 

We list the choices of $\mathcal{F}$, $\phi$ and $\varphi$ in some representative GAN models in Table~\ref{table1}. In these GANs, the gradient that the generator receives from the discriminator with respect to (w.r.t.) a generated sample $x \in \CS_g$ is 
\begin{equation} \label{eq_gen_grad}
\nabla_{\!x} J_G(x) \triangleq \nabla_{\!x} \psi(f(x)) = \nabla_{\!f(x)} \psi(f(x)) \cdot \nabla_{\!x} f(x), 
\end{equation}
where the first term $\nabla_{\!f(x)} \psi(f(x))$ is a step-related scalar, and the second term $\nabla_{\!x} f(x)$ is a vector with the same dimension as $x$ which indicates the direction that the generator should follow for optimizing the generated sample $x$. 

We use $\ff$ to denote the optimal discriminative function, i.e., $\ff\triangleq\argmin_{f \in \mathcal{F}} J_D$. For further notation, we let $\mathring{J}_D(x) \triangleq \CP_g(x) \phi(f(x)) + \CP_r(x) \varphi(f(x))$. It has $J_D= \int{\mathring{J}_D(x) d x}$.  

\subsection{The Gradient Vanishing}

The gradient vanishing problem has been typically thought as a key factor for causing the nonconvergence of GANs, i.e., the gradient becomes zero when the discriminator is perfectly trained.  

\citet{gan} addressed this problem by using an alternative objective for the generator. Actually, only the scalar $\nabla_{\!f(x)} \psi(f(x))$ is changed. The Least-Squares GAN \citep{lsgan}, which aims at addressing the gradient vanishing problem, also focused on $\nabla_{\!f(x)} \psi(f(x))$. 

\citet{principled_methods} provided a new perspective for understanding the gradient vanishing. They argued that $\CS_r$ and $\CS_g$ are usually disjoint and the gradient vanishing stems from the ill-behaving of traditional distance metrics, i.e., the distance between $\CP_r$ and $\CP_g$ remains constant when they are disjoint. The Wasserstein distance was thus used \cite{wgan} as an alternative metric, which can properly measure the distance between two distributions no matter whether they are disjoint. 

\section{The Gradient Uninformativeness} \label{sec_gradient_uninformative}

In this paper we pay our main attention on the gradient direction of the optimal discriminative function, i.e., $\nabla_{\!x} \ff(x)$, along which the generated sample $x$ is updated. We show that for many distance metrics, such a gradient may fail to bring any useful information about $\CP_r$.
Consequently, $\CP_g$ is not guaranteed to converge to $\CP_r$. 
We name this pheno-menon as the \emph{gradient uninformativeness} 
and argue that it is a fundamental factor of resulting in nonconvergence and instability in the training of traditional GANs.

The gradient uninformativeness is substantially different from the gradient vanishing. The gradient vanishing is about the scalar term $\nabla_{\!f(x)} \psi(f(x))$ in $\nabla_{\!x} J_G(x)$ or the overall scale of $\nabla_{\!x} J_G(x)$, while the gradient uninformativeness is about the direction of $\nabla_{\!x} J_G(x)$, which is defined by $\nabla_{\!x} \ff(x)$. The two issues are orthogonal, though they sometimes exist simultaneously. See Table~\ref{table1} for a summary of issues for representative GANs.

Next, we discuss the gradient uninformativeness in the taxonomy of restrictions on the discriminative function space $\mathcal{F}$. We will show that for unrestricted GANs, gradient uninformativeness commonly exists; for restricted GANs, such an issue might still exist; and with Lipschitz condition, it generally does not exist.

\subsection{Unrestricted GANs}

For many GAN models, there is no restriction on $\mathcal{F}$. Typical cases include  $\!f$-divergence based GANs, such as the vanilla GAN \cite{gan}, Least-Squares GAN \cite{lsgan} and $f$-GAN \cite{fgan}. 

In these GANs, the value of the optimal discriminative function at each point $\ff(x)$ is independent of other points and only reflects the local densities $\CP_r(x)$ and $\CP_g(x)$: 
\vspace{-1pt}
\begin{equation*} 
\begin{aligned}
\ff(x) = \argmin_{f(x) \in \BR} \ \CP_g(x)  \phi(f(x)) + \CP_r(x) \varphi(f(x)), \,\, \forall x. \nonumber
\end{aligned} 
\vspace{-1pt}
\end{equation*}
Hence, for each generated sample $x$ which is not surrounded by real samples (there exists $\epsilon\!>\!0$ such that for all $y$ with $0\!<\!\lVert y-x \rVert\!<\!\epsilon$, it holds that $y \notin \CS_r$), $\ff(x)$ in the surrounding of $x$ would contain no information about $\CP_r$. Thus $\nabla_{\!x} \ff(x)$, the gradient that $x$ receives from the optimal discriminative function, does not reflect any information about $\CP_r$. 

Typical situation is that $\CS_r$ and $\CS_g$ are disjoint, which is common in practice according to \citep{principled_methods}. To further distinguish the gradient uninformativeness from the gradient vanishing, we consider an ideal case: $\CS_r$ and $\CS_g$ are totally overlapped and both consist of $n$ discrete points, but their probability masses over these points are different. In this case, $\nabla_{\!x} \ff(x)$ for each generated sample is still uninformative, but the gradient does not vanish. 

\subsection{Restricted GANs: Fisher GAN as an Instance }

Some GANs impose restrictions on $\mathcal{F}$. Typical instances are the Integral Probability Metric (IPM) based GANs \cite{fishergan,sobolevgan,bellemare2017cramer} and the Wasserstein GAN \cite{wgan}. We next show that GANs with restriction on $\mathcal{F}$ might also suffer from the gradient uninformativeness. 

The optimal discriminative function of $\mu$-Fisher IPM $\mathcal{F}_{\mu}(\CP_r, \CP_g)$, the generalized objective of the Fisher GAN \citep{sobolevgan}, has the following form:
\vspace{-1pt}
\begin{equation}
\begin{aligned}
\ff(x)= \frac{1}{\mathcal{F}_{\mu}(\CP_r, \CP_g)}\frac{\CP_r(x)-\CP_g(x)}{\mu(x)},
\end{aligned}
\vspace{-0.5pt}
\end{equation}
where $\mu$ is a distribution whose support covers $\CS_r$ and $\CS_g$, and $\frac{1}{\mathcal{F}_{\mu}(\CP_r, \CP_g)}$ is a constant. 
It can be observed that $\mu$-Fisher IPM also defines $\ff(x)$ at each point according to the local densities and does not reflect information of other locations. Similar as above, we can conclude that for each generated sample that is not surrounded by real samples, $\nabla_{\!x} \ff(x)$ is uninformative. 

\subsection{The Wasserstein GAN} \label{sec_w_also_suffer}

As shown by \citet{wgangp}, the gradient of the optimal discriminative function in the KR dual form of the Wasserstein distance has the following property: 
\begin{proposition} \label{lem:1}
Let $\pi^*$ be the optimal transport plan in Eq.~(\ref{eq_w_primal}) and $x_t=t x + (1-t) y$ with $0 \leq t \leq 1$. If the optimal discriminative function $\ff$ in Eq.~(\ref{eq_w_dual_lip}) is differentiable and $\pi^*(x,x)=0\,$ for all $x$,  then it holds that
\vspace{-2pt}
\begin{equation} \label{lipgrad}
{{\rm P}_{(x, y)\sim \pi^*}} \left[\nabla_{\!x_t} \ff(x_t) = \frac{y-x}{\lVert y - x \rVert}\right] = 1. 
\end{equation}
\end{proposition} 

This proposition indicates: (i) for each generated sample $x$, there exists a real sample $y$ such that $\nabla_{\!x_t} \ff(x_t)=\frac{y-x}{\lVert y - x \rVert}$ for all linear interpolations $x_t$ between $x$ and $y$, i.e., the gradient at any $x_t$ is pointing towards the real sample $y$; (ii) these $(x, y)$ pairs match the optimal coupling $\pi^*$ in the optimal transport perspective. It implies that WGAN is able to overcome the gradient uninformativeness as well as the gradient vanishing. 

Our concern turns to the reason why WGAN can avoid gradient uninformativeness. 
To address this question, we alternatively apply the compact dual of the Wasserstein distance in Eq.~\eqref{eq_w_dual_form_1} and study the optimal discriminative function. 

Since there is generally no closed-form solution for $\ff$ in Eq.~(\ref{eq_w_dual_form_1}), we take an illustrative example, but the conclusion is general. 
Let $Z \sim U[0,1]$ be a uniform variable on interval $[0, 1]$, $\CP_r$ be the distribution of $(1, Z)$ in $\BR^2$, and $\CP_g$ be the distribution of $(0, Z)$ in $\BR^2$. According to Eq.~(\ref{eq_w_dual_form_1}),  we have an optimal $\ff$ as follows
\vspace{-3pt}
\begin{equation}
\label{eq9}
\ff(x)=
\begin{cases}
\begin{aligned}
&1, \,\,\,\,\,\, &&\forall x \in \CS_r; \\
&0, \,\,\,\,\,\, &&\forall x \in \CS_g. 
\end{aligned}
\end{cases}
\vspace{-3pt}
\end{equation}
Though having the constraint ``$f(x) - f(y) \leq d(x, y),$ $\, \forall x \in \CS_r, \forall y \in \CS_g$,'' the Wasserstein distance in this dual form also only defines the values of $\ff(x)$ on $\CS_r$ and $\CS_g$. For each generated sample $x$ which is isolated or at the boundary (there does not exist $\epsilon\!>\!0$ such that it holds $y \in \CS_r\cup\CS_g$ for all $y$ with $0\!<\!\lVert y-x \rVert\!<\!\epsilon$), the gradient of $\ff(x)$ is theoretically undefined and thus cannot provide useful information about $\CP_r$. We can consider the more extreme case where $\CS_g$ are isolated points to make it clearer. 


These examples imply that Lipschitz condition would be critical for resolving the gradient uninformativeness problem. Motivated by this, we study the general formulation of GANs with Lipschitz constraint, which leads to a family of more general GANs that we call Lipschitz GANs. We will see that in Lipschitz GANs, the similarity measure between $\CP_r$ and $\CP_g$ might not be some Wasserstein distance, but they still perform very well.

\section{Lipschitz GANs} \label{sec_lip}

Lipschitz continuity recently becomes popular in GANs. It was observed that introducing Lipschitz continuity as a regularization of the discriminator leads to improved stability and sample quality \citep{wgan,kodali2017convergence,fedus2017many,sngan,qi2017loss}. 

In this paper, we investigate the general formulation of GANs with Lipschitz constraint, where the Lipschitz constant of discriminative function is penalized via a quadratic loss, to theoretically analyze the properties of such GANs. In particular, we define the Lipschitz Generative Adversarial Nets (LGANs) as: 
\begin{align} 
&\min_{f \in \mathcal{F}} \E_{z \sim \CP_z} [ \phi(f(g(z))) ] + \E_{x \sim \CP_r} [ \varphi(f(x)) ] + \lambda \cdot k(f)^2,  \nonumber \\ 
&\min_{g \in \mathcal{G}} \E_{z \sim \CP_z} [ \psi(f(g(z))) ]. \label{eq_lgan}
\vspace{-1pt}
\end{align}
In this work, we further assume that the loss functions $\phi$ and  $\varphi$ satisfy the following conditions: 
\vspace{-1pt}
\begin{equation}
\begin{cases} \label{eq_solvable}
\phi'(x) > 0, \varphi'(x) < 0, \\[2pt]
\phi''(x) \geq 0, \varphi''(x) \geq 0, \\[2pt]
\exists \, a, \, \phi'(a) + \varphi'(a) = 0. 
\end{cases}
\vspace{-1pt}
\end{equation}
The assumptions for the losses $\phi$ and $\varphi$ are very mild. Note that in WGAN $\phi(x)=\varphi(-x)=x$ is used, which satisfies Eq.~(\ref{eq_solvable}). There are many other instances, such as $\phi(x)=\varphi(-x)=-\log(\sigma(-x))$, $\phi(x)=\varphi(-x)=x+\sqrt{x^2+1}$ and $\phi(x)=\varphi(-x)=\exp(x)$. Meanwhile, there also exist losses used in GANs that do not satisfy Eq.~(\ref{eq_solvable}), e.g., the quadratic loss \cite{lsgan} and the hinge loss \cite{energy_based_gan,geometric_gan,sngan}. 

To devise a loss in LGANs, it is practical to let $\phi$ be be an increasing function with non-decreasing derivative and set $\phi(x)=\varphi(-x)$. Moreover, the linear combinations of such losses still satisfy Eq.~(\ref{eq_solvable}). Figure~\ref{fig_function_curve} illustrates some of these loss metrics. 

Note that $\phi(x)=\varphi(-x)=-\log(\sigma(-x))$ is the objective of vanilla GAN. As we have shown, the vanilla GAN suffers from the gradient uninformativeness problem. However, as we will show next, when imposing the Lipschitz regularization, the resulting model as a specific case of LGANs behaves very well. 

\subsection{Theoretical Analysis} \label{sec_main_theorem}

We now present the theoretical analysis of LGANs. First, we consider the existence and uniqueness of the optimal discriminative function.

\begin{theorem} \label{thm:esistence}
Under Assumption~\eqref{eq_solvable} and if $\phi$ or $\varphi$ is strictly convex, the optimal discriminative function $\ff$ of Eq.~(\ref{eq_lgan}) exists and is unique. 
\vspace{-4pt}
\end{theorem}

Note that although WGAN does not satisfy the condition in Theorem~\ref{thm:esistence}, its solution still exists but is not unique. Specifically, if $\ff$ is an optimal solution then $ \ff + \alpha$ for any $\alpha \in \BR$ is also an optimal solution. The following theorems can be regarded as a generalization of Proposition~\ref{lem:1} to LGANs. 

\begin{theorem} \label{theorem_main} 
Assume $\phi'(x)>0$, $\varphi'(x)<0$, and the optimal discriminator $\ff$ exists and is smooth. We have 
\vspace{-9pt}
\begin{enumerate}[label=(\alph*),leftmargin=16pt] 
\setlength\itemsep{0.0em} 
\item For all $x \in {\CS_r} \cup {\CS_g}$, if it holds that $\nabla_{\!\ff(x)} \mathring{J_D}(x) \neq 0$, then there exists $y \in {\CS_r}\cup {\CS_g}$ with $y\neq x$ such that $|\ff(y) - \ff(x)|=k(\ff) \cdot \lVert y - x \rVert $; 
\item For all $x \in {\CS_r} \cup {\CS_g} - {\CS_r} \cap {\CS_g}$, 
there exists $y \in {\CS_r}\cup {\CS_g}$ with $y\neq x$ such that $|\ff(y) - \ff(x)|=k(\ff) \cdot \lVert y - x \rVert $; 
\item If ${\CS_r}= {\CS_g}$ and $ \CP_r\neq \CP_g$, then there exists $(x, y)$ pair with both points in ${\CS_r} \cup {\CS_g}$ and $y\neq x$ such that $|\ff(y)-\ff(x)|=k(\ff) \cdot \lVert y - x \rVert $ and $\nabla_{\!\ff(x)} \mathring{J_D}(x) \neq 0$;
\item There is a unique Nash equilibrium between $\CP_g$ and $\ff$ under the objective $J_D + \lambda \cdot k(f)^2$, where it holds that $\CP_r=\CP_g$ and $k(\ff)=0$. 
\end{enumerate}
\vspace{-8pt}
\end{theorem}

The proof is given in Appendix~\ref{app_proof}. 
This theorem states the basic properties of LGANs, including the existence of unique Nash equilibrium where $\CP_r = \CP_g$ and the existence of \emph{bounding relationships} in the optimal discriminative function (i.e., $\exists y\neq x$ such that $|\ff(y) - \ff(x)|=k(\ff) \cdot \lVert y - x \rVert $). The former ensures that the objective is a well-defined distance metric, and the latter, as we will show next, eliminates the gradient uninformativeness problem. 

It is worth noticing that the penalty $k(f)$ is in fact necessary for Property-(c) and Property-(d). 
The reason is due to the existence of the case that $\nabla_{\!\ff(x)}\mathring{J_D}(x)=0$ for $\CP_r(x)\neq \CP_g(x)$. Minimizing $k(f)$ guarantees that the only Nash equilibrium is achieved when $\CP_r=\CP_g$. 
In WGAN, minimizing $k(f)$ is not necessary. However, if $k(f)$ is not minimized towards zero, $\nabla_{\!x}\ff(x)$ is not guaranteed to be zero at the convergence state $\CP_r=\CP_g$ where any function subject to $1$-Lipschitz constraint is an optimal $\ff$ in WGAN. It implies that minimizing $k(f)$ also benefits WGAN. 

\subsection{Refining the Bounding Relationship} \label{sec_connections}

From Theorem \ref{theorem_main}, we know that for any point $x$, as long as $\mathring{J_D}(x)$ does not hold a zero gradient with respect to $\ff(x)$, $\ff(x)$ must be bounded by another point $y$ such that $|\ff(y)-\ff(x)|=k(\ff) \cdot \lVert y - x \rVert $. We further clarify that when there is a bounding relationship, it must involve both real sample(s) and fake sample(s). More formally, we have

\begin{theorem} \label{theorem_chain}
Under the conditions in Theorem~\ref{theorem_main}, we have 
\vspace{-7pt}
\begin{enumerate}[label=\arabic*),leftmargin=16pt] 
\setlength\itemsep{0.00em} 
\item For any $x \in {\CS_g}$, if $\,\nabla_{\!\ff(x)} \mathring{J_D}(x) > 0$, then there must exist some $y \in {\CS_r}$ with $y\neq x$ such that $\ff(y)-\ff(x)=k(\ff) \cdot \lVert y - x \rVert $ and $\,\nabla_{\!\ff(y)} \mathring{J_D}(y) < 0$; 
\item For any $y \in {\CS_r}$, if $\,\nabla_{\!\ff(y)} \mathring{J_D}(y) < 0$, then there must exist some $x \in {\CS_g}$ with $y\neq x$ such that $\ff(y)-\ff(x)=k(\ff) \cdot \lVert y - x \rVert $ and $\,\nabla_{\!\ff(x)} \mathring{J_D}(x) > 0$.
\end{enumerate}
\vspace{-9pt} 
\end{theorem}

The intuition behind the above theorem is that samples from the same distribution (e.g., the fake samples) will not bound each other to violate the optimality of $\mathring{J_D}(x)$. So, when there is strict bounding relationship (i.e., it involves points that hold $\nabla_{\!\ff(x)} \mathring{J_D}(x) \neq 0$), it must involve both real and fake samples. It is worth noticing that if only it is not the overlapping case, all fake samples hold $\nabla_{\!\ff(x)} \mathring{J_D}(x) > 0$, while all real samples hold $\nabla_{\!\ff(y)} \mathring{J_D}(y) < 0$. 



Note that there might exist a dozen real and fake samples that bound each other. Under the Lipschitz continuity condition, the bounding relationship on the value surface of $\ff$ is the basic building block that connects $\CP_r$ and $\CP_g$, and each fake sample with $\nabla_{\!\ff(x)} \mathring{J_D}(x) \neq 0$ lies in at least one of these bounded relationships. Next we will further interpret the implication of bounding relationship and show that it guarantees meaningful $\nabla_{\!x} \ff(x)$ for all involved points.

\subsection{The Implication of Bounding Relationship} \label{sec_convergence}

Recall that the Proposition~\ref{lem:1} states that $\nabla_{\!x_t} \ff(x_t)=\frac{y-x}{\lVert y - x \rVert}$. We next show that it is actually a direct consequence of bounding relationship between $x$ and $y$. We formally state it as follows:

\begin{theorem} \label{theorem_gradient}
Assume function $f$ is differentiable and its Lipschitz constant is $k$, then for all $x$ and $y$ which satisfy $y\neq x$ and $f(y)-f(x)=k \cdot \lVert y - x \rVert$, we have $\nabla_{\!x_t} f(x_t)=k \cdot \frac{y-x}{\rVert y-x \rVert}\,$ for all $x_t=tx+(1-t)y$ with $0 \leq t \leq 1$. 
\vspace{-2pt}
\end{theorem}

In other words, if two points $x$ and $y$ bound each other in terms of $f(y)-f(x)=k \cdot \lVert y - x \rVert$, there is a straight line between $x$ and $y$ on the value surface of $f$. Any point in this line holds the maximum gradient slope $k$, and the gradient direction at any point in this line is pointing towards the $x \rightarrow y$ direction. The proof is provided in Appendix~\ref{lip_direction}. 

Combining Theorems~\ref{theorem_main} and \ref{theorem_chain}, we can conclude that when $\CS_r$ and $\CS_g$ are disjoint, the gradient $\nabla_{\!x} \ff(x)$ for each generated sample $x \in \CS_g$  points towards some real sample $y \in \CS_r$, which guarantees that $\nabla_{\!x} \ff(x)$-based updating would pull $\CP_g$ towards $\CP_r$ at every step. 

In fact, Theorem~\ref{theorem_main} provides further guarantee on the convergence. Property-(b) implies that for any generated sample $x \in \CS_g$ that does not lie in $\CS_r$, its gradient $\nabla_{\!x} \ff(x)$ must point towards some real sample $y \in \CS_r$. And in the fully overlapped case, according to Property-(c), unless $\CP_r=\CP_g$, there must exist at least one pair of $(x, y)$ in strict bounding relationship and $\nabla_{\!x} \ff(x)$ pulls $x$ towards $y$. Finally, Property-(d) guarantees that the only Nash equilibrium is $\CP_r=\CP_g$ where $\nabla_{\!x} \ff(x)=0$ for all generated samples. 


\section{Empirical Analysis} \label{sec_exp}

In this section, we empirically study the gradient uninformativeness problem and the performance of various objectives of Lipschitz GANs. The anonymous code is provided in the supplemental material. 

\begin{figure*}
\begin{subfigure}{0.33\linewidth}
    \centering
    \includegraphics[width=0.99\columnwidth]{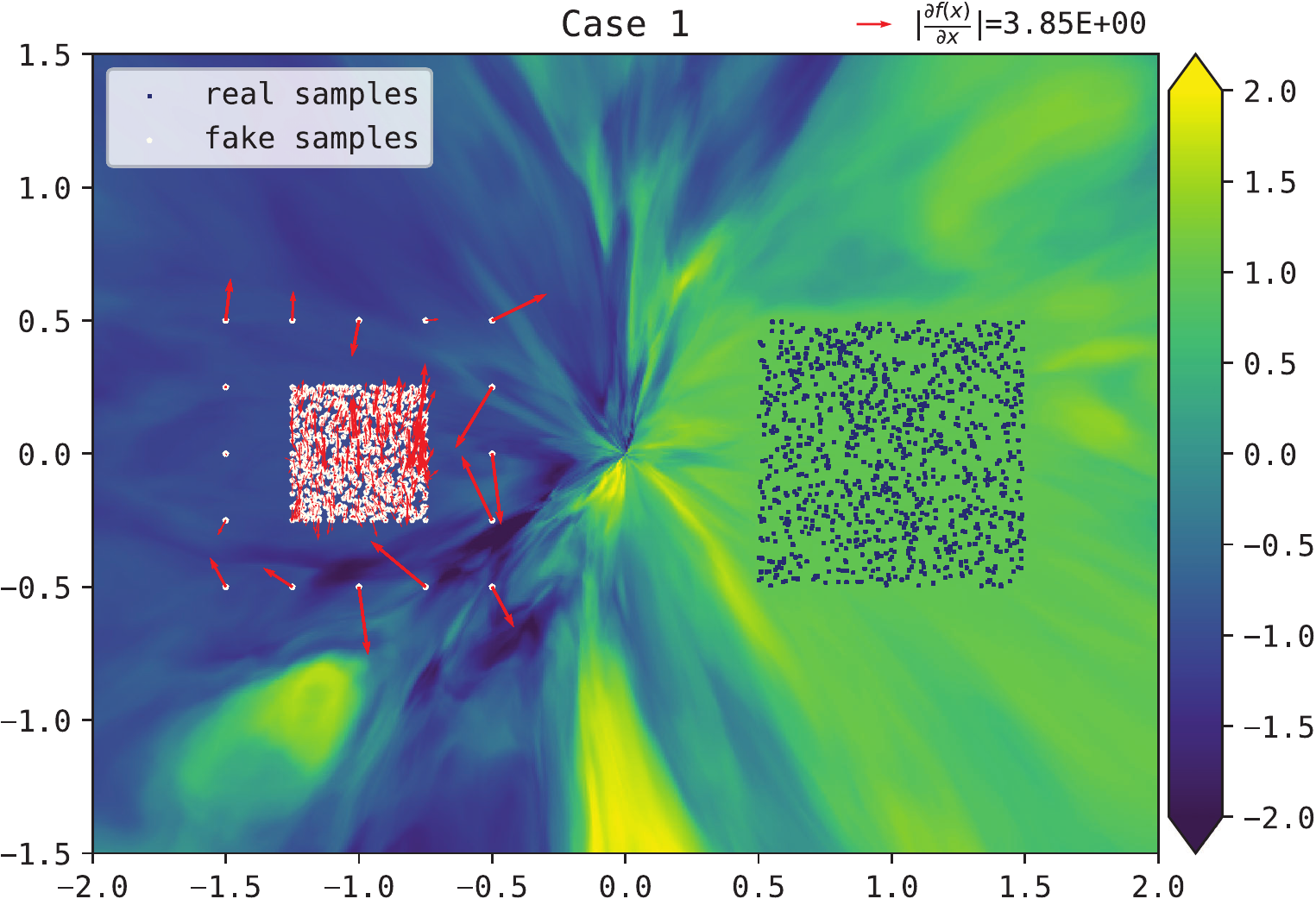}
    \caption{Disjoint Case}
    \label{fig_case1}
\end{subfigure}
\begin{subfigure}{0.33\linewidth}
    \centering
    \includegraphics[width=0.99\columnwidth]{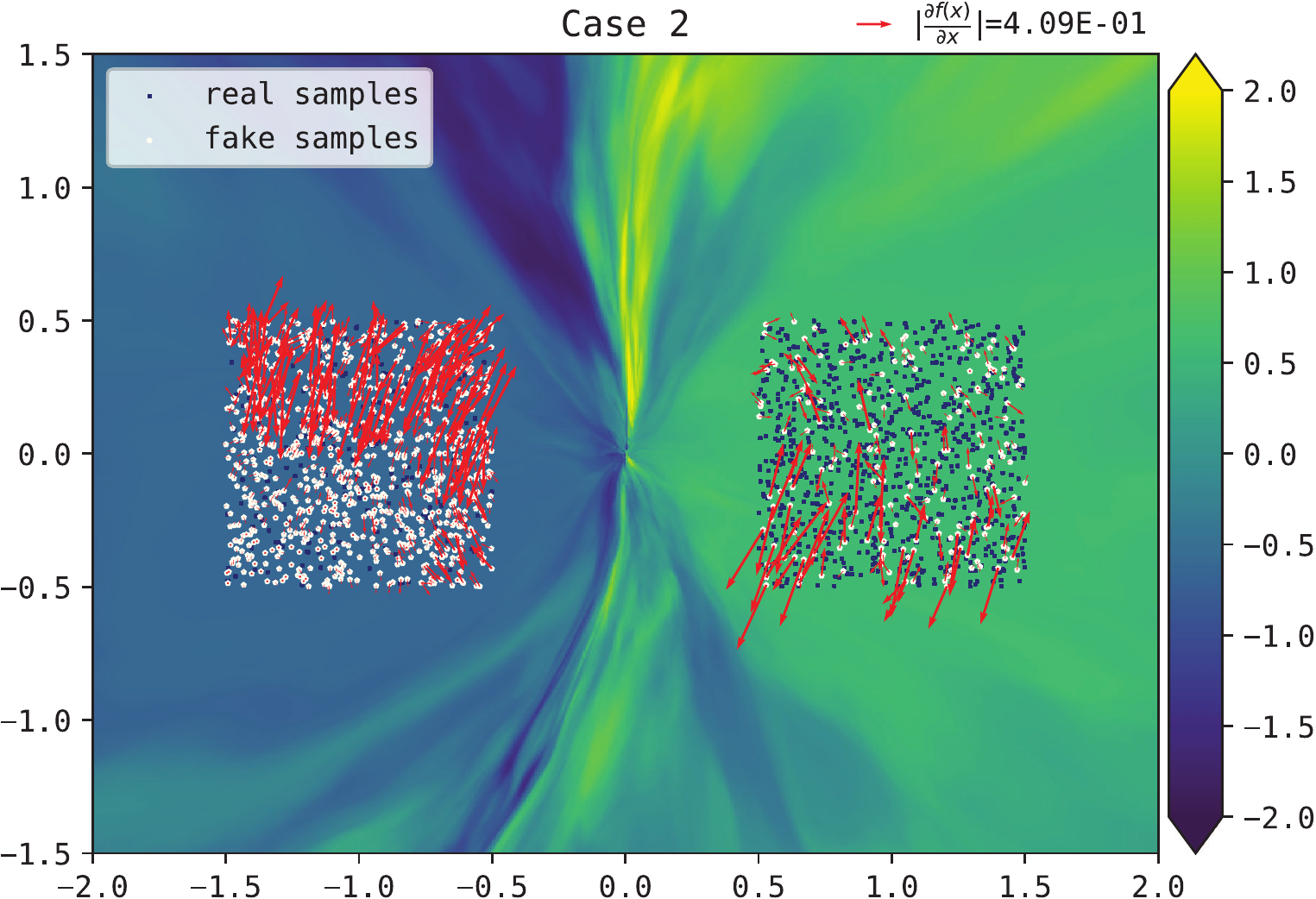}
    \caption{Overlapping Case}
    \label{fig_case2}
\end{subfigure}	
\begin{subfigure}{0.33\linewidth}
    \centering
    \includegraphics[width=0.99\columnwidth]{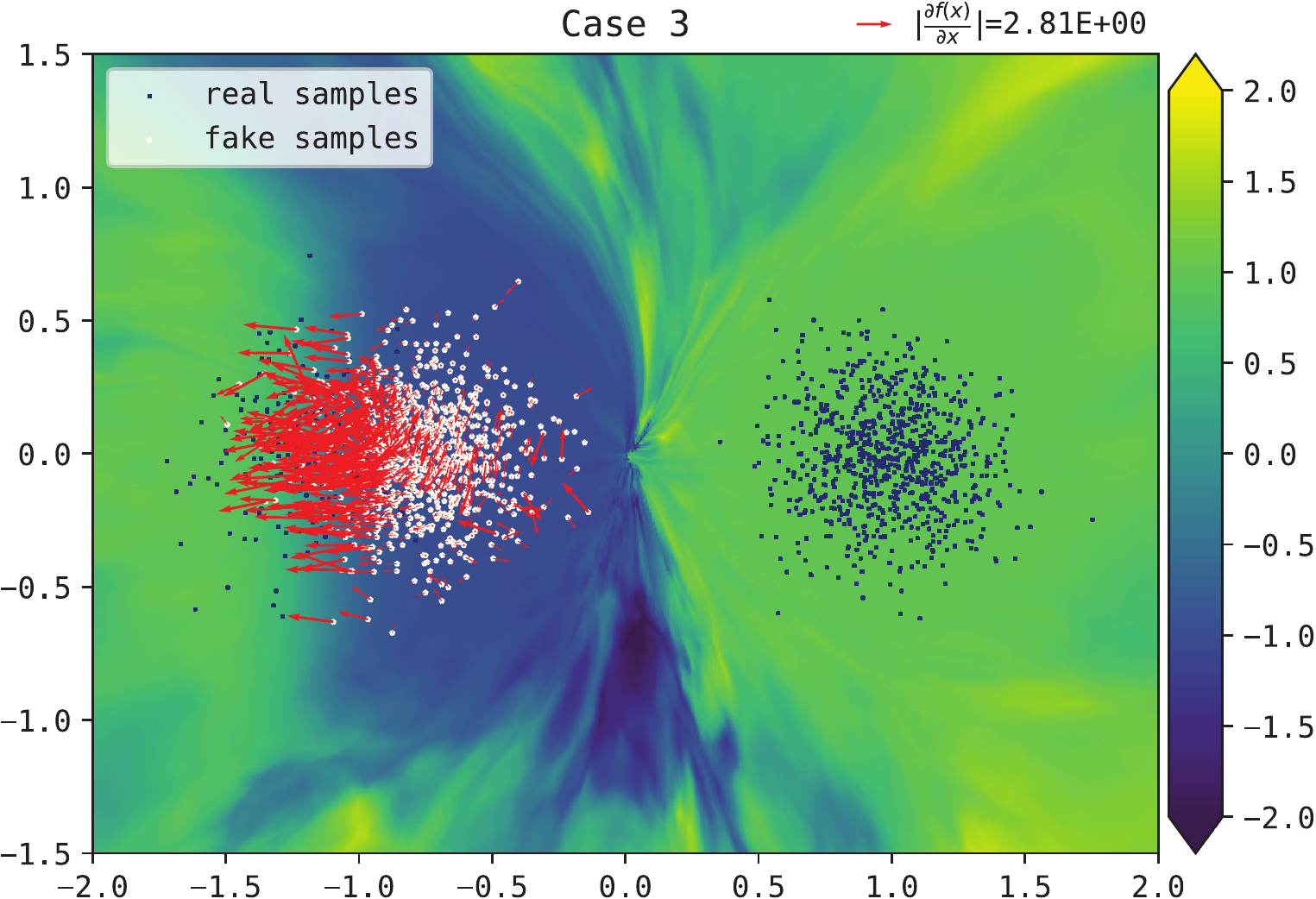}
    \caption{Mode Collapse}
    \label{fig_case3}
\end{subfigure}	
\vspace{-8pt}
\caption{Practical behaviors of gradient uninformativeness: noisy gradient. Local greedy gradient leads to mode collapse.}
\label{fig_case123}
\vspace{-5pt}
\end{figure*}

\subsection{Gradient Uninformativeness in Practice}

According to our analysis, $\nabla_{\!x} \ff(x)$ for most traditional GANs is uninformative. Here we investigate the practical behaviors of the gradient uninformativeness. Note that the behaviors of GANs without restriction on $\mathcal{F}$ are essentially identical. We choose the Least-Squares GAN whose $\ff$ is relatively simple as the representative and study it with a set of synthetic experiments which benefits the visualization. 

The results are shown in Figure~\ref{fig_case123}. We find that the gradient is very random, which we believe is the typical practical behavior of the gradient uninformativeness. Given the nondeterministic property of $\ff(x)$ for points out of $\CS_r\cup\CS_g$, $\nabla_{\!x} \ff(x)$ is highly sensitive to the hyper-parameters. We actually conduct the same experiments with a set of different hyper-parameters. The rest is provided in Appendix~\ref{hyper_para}. 

In Section~\ref{sec_gradient_uninformative}, we discussed the gradient uninformativeness under the circumstances that the fake sample is not surrounded by real samples. Actually, the problem of $\nabla_{\!x} \ff(x)$ in traditional GANs is more general, which can also be regarded as the gradient uninformativeness. For example, in the case of Figure~\ref{fig_case2} where the real and fake samples are both evenly distributed in the two regions with different densities, $\ff(x)$ is constant in each region and undefined outside. It theoretically has zero $\nabla_{\!x} \ff(x)$ for inner points and undefined $\nabla_{\!x} \ff(x)$ for boundary points. They in practice also behave as noisy gradient. We note that in the totally overlapping and continuous case, $\nabla_{\!x} \ff(x)$ is also ill-behaving, which seems to be an intrinsic cause of mode collapse, as illustrated in Figure~\ref{fig_case3} where $\CP_r$ and $\CP_g$ are both devised to be Gaussian(s). 

\subsection{Verifying $\nabla_{\!x} \ff(x)$ of LGANs}

\begin{figure*}[thbp]
\begin{minipage}{.5\textwidth}
    \begin{subfigure}{0.49\linewidth}
        \includegraphics[width=0.99\columnwidth]{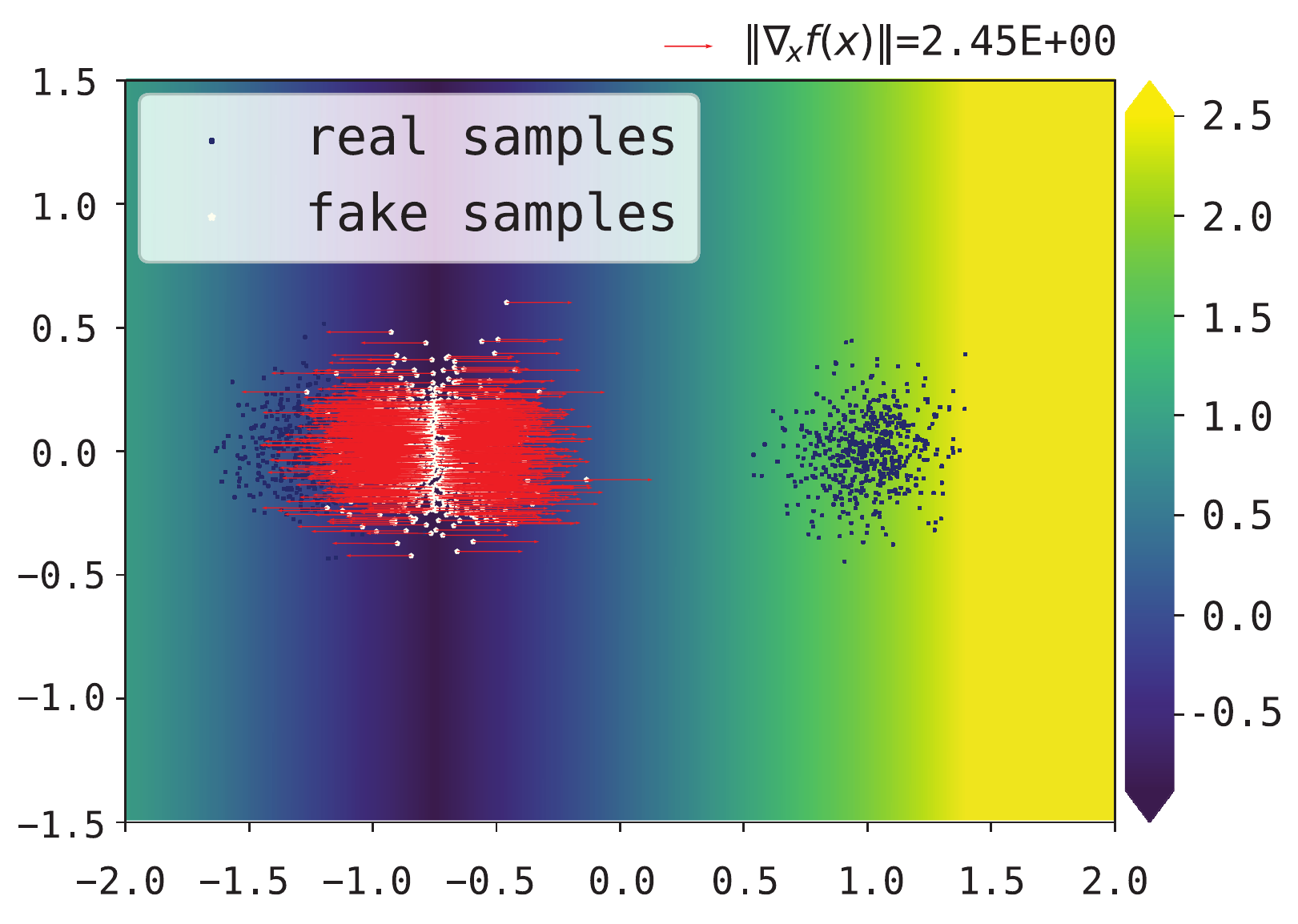}
        \vspace{-17pt}
        \caption{$x$}
        \label{a}
    \end{subfigure}	
    \begin{subfigure}{0.49\linewidth}
        \includegraphics[width=0.99\columnwidth]{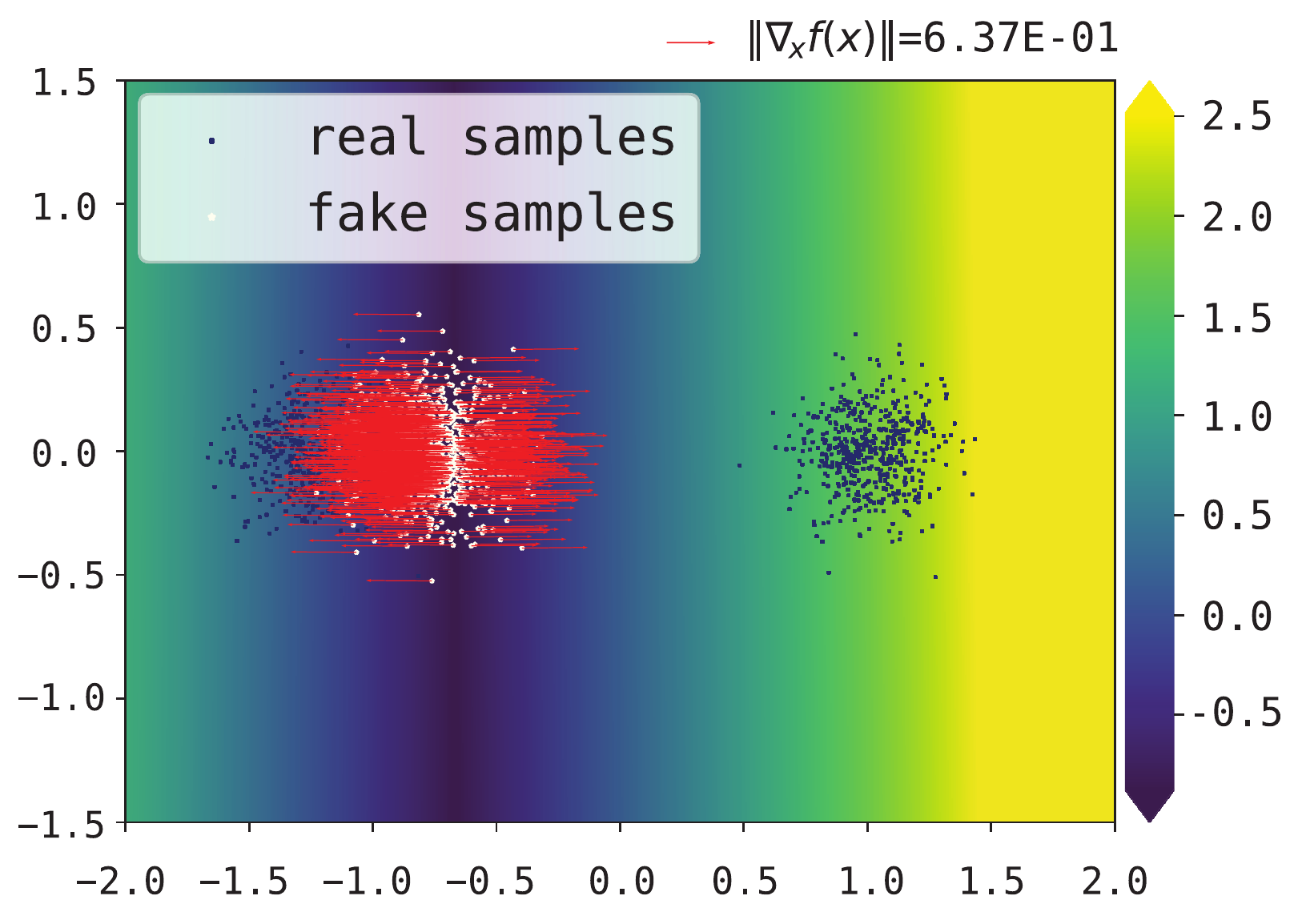}
        \vspace{-17pt}
        \caption{$-\log(\sigma(-x))$}
        \label{b}
    \end{subfigure}	
    \begin{subfigure}{0.49\linewidth}
        \vspace{3pt}
        \includegraphics[width=0.99\columnwidth]{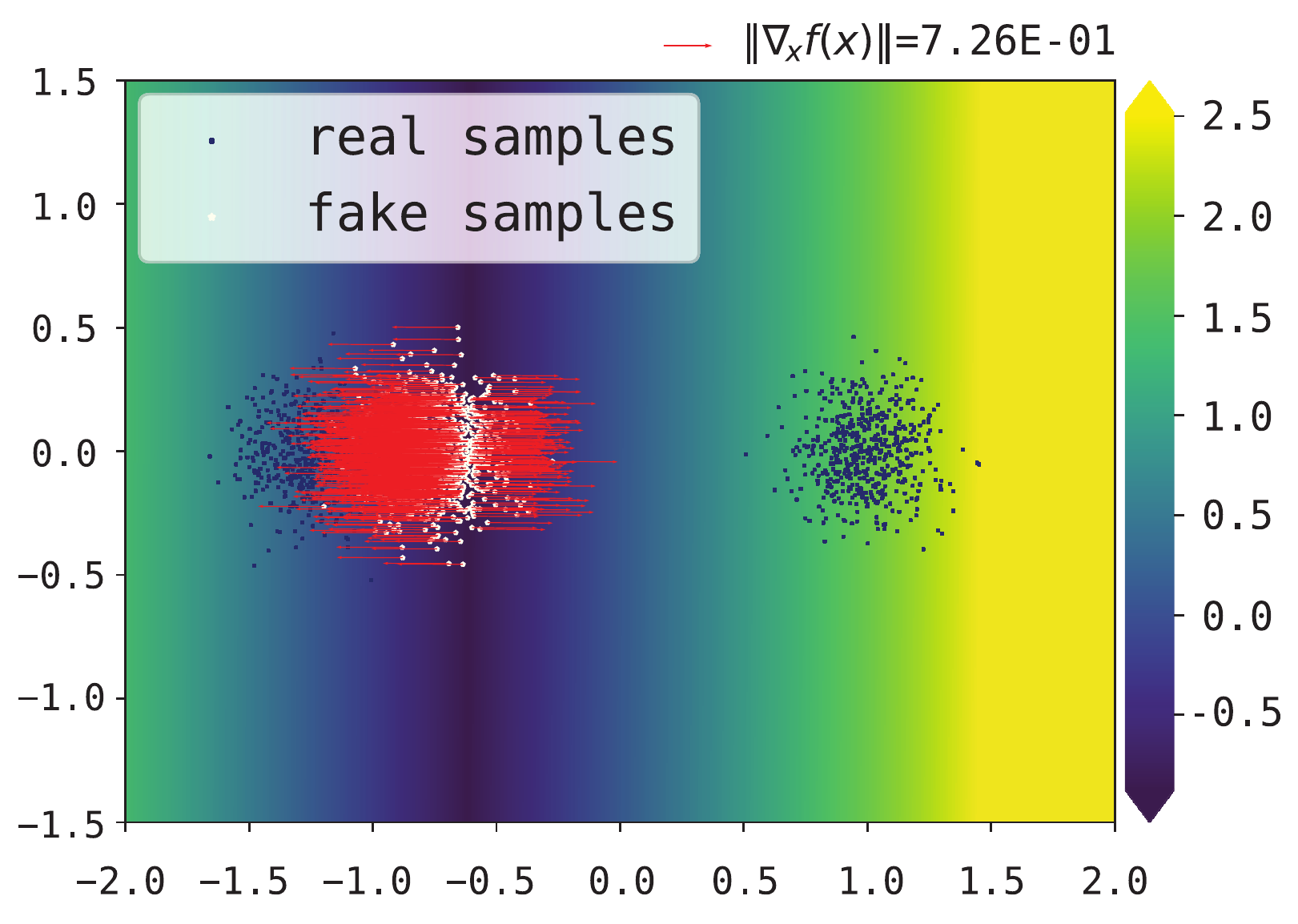}
        \vspace{-17pt}
        \caption{$x+\sqrt{x^2+1}$}
        \label{c}
    \end{subfigure}	
    \begin{subfigure}{0.49\linewidth}
        \vspace{3pt}
        \includegraphics[width=0.99\columnwidth]{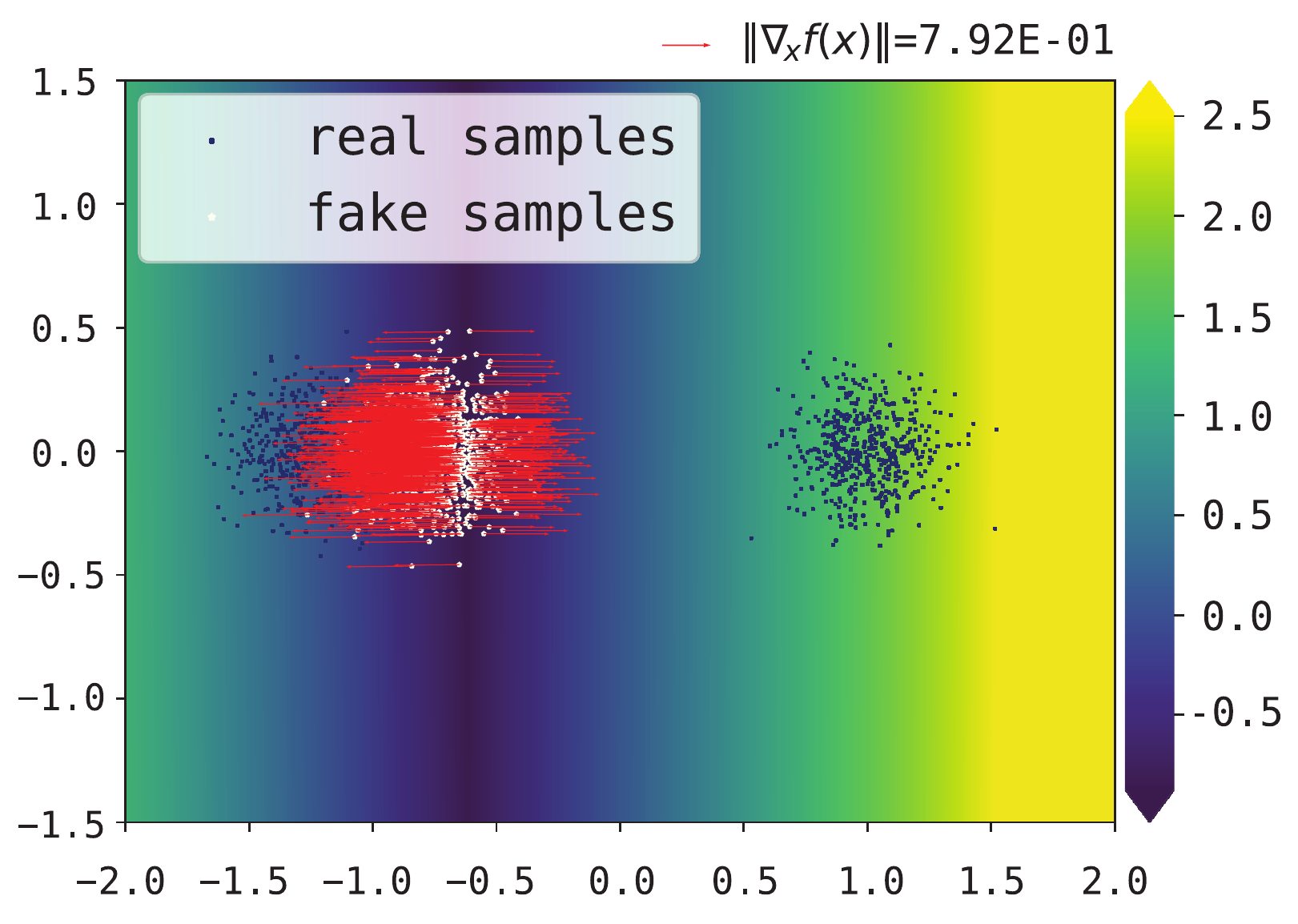}
        \vspace{-17pt}
        \caption{$\exp(x)$}
        \label{d}
    \end{subfigure}	
    \vspace{-5pt} 
    \caption{$\nabla_{\!x} \ff(x)$ in LGANs point towards real samples.}
\label{verifing_f_syn}
\end{minipage}
\begin{minipage}{.5\textwidth}
    \centering
    \includegraphics[width=0.95\linewidth]{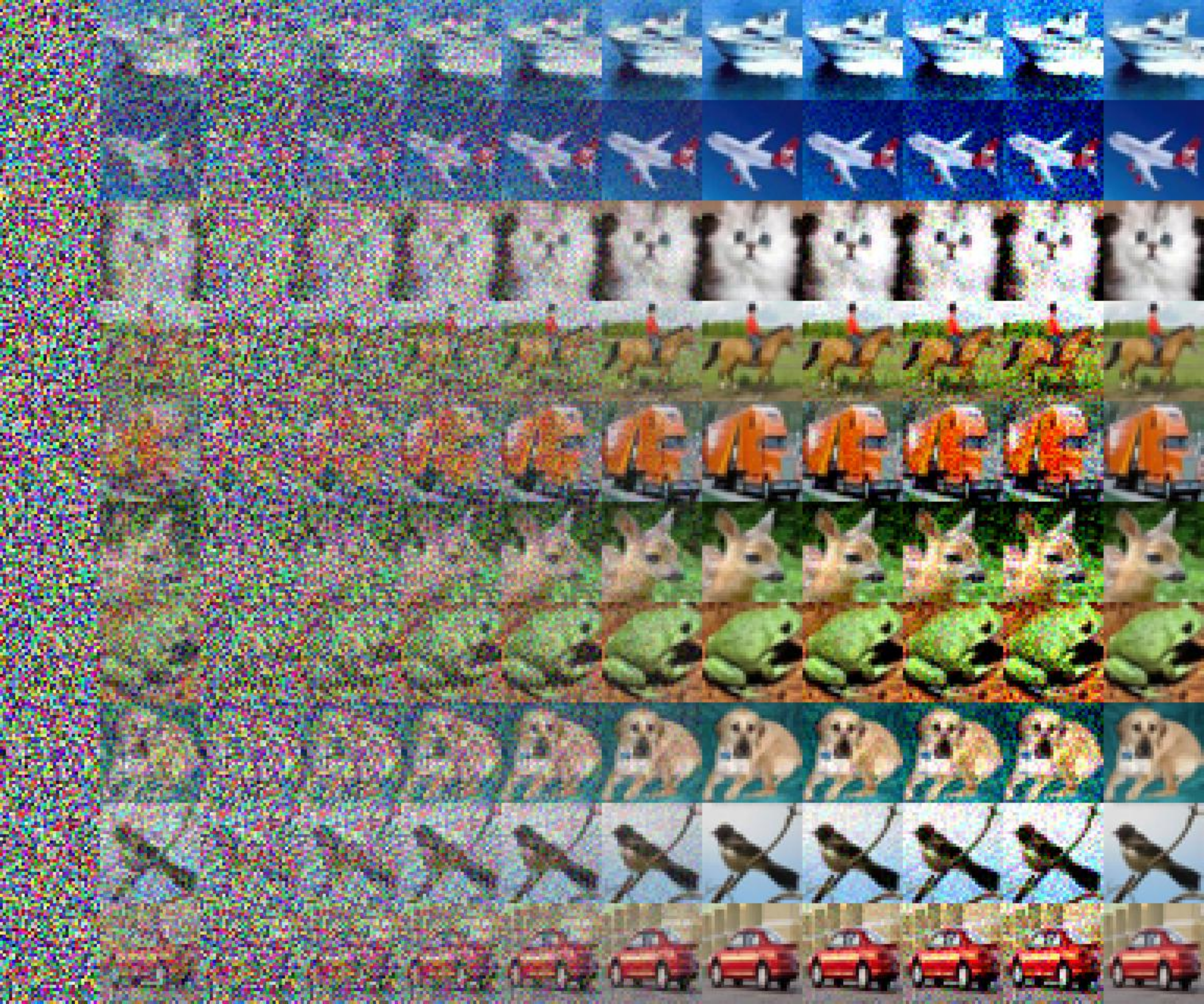}
    \vspace{-2pt}
    \caption{$\nabla_{\!x} \ff(x)$ gradation with CIFAR-10.}
\label{verifing_f}
\end{minipage}
\vspace{-5pt}
\end{figure*}

One important theoretical benefit of LGANs is that $\nabla_{\!x} \ff(x)$ for each generated sample is guaranteed to point towards some real sample. We here verify the gradient direction of $\nabla_{\!x} \ff(x)$ with a set of $\phi$ and $\varphi$ that satisfy Eq.~(\ref{eq_solvable}). 

The tested objectives include: (a) $\phi(x)=\varphi(-x)=x$; (b) $\phi(x)=\varphi(-x)=-\log(\sigma(-x))$; (c) $\phi(x)=\varphi(-x)=x+\sqrt{x^2+1}$; (d) $\phi(x)=\varphi(-x)=\exp(x)$. And they are tested in two scenarios: two-dimensional toy data and real-world high-dimensional data. In the two-dimensional case, $\CP_r$ consists of two Gaussians and $\CP_g$ is fixed as one Gaussian which is close to one of the two real Gaussians, as illustrated in Figure~\ref{verifing_f_syn}. For the latter case, we use the CIFAR-10 training set. To make  solving  $\ff$ feasible, we use ten CIFAR-10 images as $\CP_r$ and ten fixed noise images as $\CP_g$. Note that we fix $\CP_g$ on purpose because to verify the direction of $\nabla_{\!x} \ff(x)$, learning $\CP_g$ is not necessary. 

The results are shown in Figures~\ref{verifing_f_syn} and \ref{verifing_f}, respectively. In Figure~\ref{verifing_f_syn}, we can see that the gradient of each generated sample is pointing towards some real sample. 
For the high dimensional case, visualizing the gradient direction is nontrivial. Hence, we plot the gradient and corresponding increments. In Figure~\ref{verifing_f}, the leftmost in each row is a sample $x$ from $\CP_g$ and the second is its gradient $\nabla_{\!x} f(x)$. The interiors are $x+\epsilon\cdot\nabla_{\!x} f(x)$ with increasing $\epsilon$ and the rightmost is the nearest real sample $y$ from $\CP_r$. This result visually demonstrates that the gradient of a generated sample is towards 
a real sample. Note that the final results of Figure~\ref{verifing_f} keep almost identical when varying the loss metric $\phi$ and $\varphi$ in the family of LGANs. 

\subsection{Stabilized Discriminative Functions}

The Wasserstein distance is a very special case that has solution under Lipschitz constraint. It is the only case where both $\phi$ and $\varphi$ have constant derivative. As a result, $\ff$ under the Wasserstein distance has a free offset, i.e., given some $\ff$, $\ff+\alpha$ with any $\alpha \in \BR$ is also an optimal. In practice, it behaves as oscillations in $f(x)$ during training. The oscillations affect the practical performance of WGAN;  \citet{progressive_growing_gan} and \citet{adler2018banach} introduced regularization to the discriminative function to prevent $f(x)$ drifting during the training. By contrast, any other instance of LGANs does not have this problem. We illustrate the practical difference in Figure~\ref{f_oscillations}. 

\subsection{Max Gradient Penalty (MaxGP)}

LGANs impose penalty on the Lipschitz constant of the discriminative function. There are works that investigate different implementations of Lipschitz continuity in GANs, such as gradient penalty (GP) \citep{wgangp}, Lipschitz penalty (LP) \cite{wganlp} and spectral normalization (SN) \citep{sngan}. However, the existing regularization methods do not directly penalize the Lipschitz constant. According to \cite{adler2018banach}, Lipschitz constant $k(f)$ is equivalent to the maximum scale of $\lVert\nabla_{\!x} f(x)\rVert$. Both GP and LP penalize all gradients whose scales are larger than the given target Lipschitz constant $k_0$. SN directly restricts the Lipschitz constant via normalizing the network weights by their largest eigenvalues. However, it is currently unclear how to effectively penalize the Lipschitz constant with SN. 

To directly penalize Lipschitz constant, we approximate $k(f)$ in Eq.~(\ref{eq_lgan}) with the maximum sampled gradient scale:
\begin{equation}
k(f) \simeq \max_x \big\lVert \nabla_{\!x} f(x) \big\rVert.
\end{equation}
Practically, we follow \citep{wgangp} and sample $x$ as random interpolation of real and fake samples. We provide more details of this algorithm (MaxGP) in Appendix~\ref{app_maxgp}. 

According to our experiments, MaxGP in practice is usually comparable with GP and LP. However, in some of our synthetic experiments, we find that MaxGP is able to achieve the optimal discriminative function while GP and LP fail, e.g., the problem of solving $\ff$ in Figure~\ref{verifing_f}. Also, in some real data experiments, we find the training with GP or LP diverges and it is able to converge if we switch to MaxGP, e.g., the training with metric $\phi(x)=\varphi(-x)=\exp(x)$. 

\subsection{Benchmark with Unsupervised Image Generation}

\begin{figure*}[t]
\begin{minipage}{.66\textwidth}
\vspace{-15pt}
\centering
\captionof{table}{Quantitative comparisons with unsupervised image generation.}
\label{table2}
\vspace{-5pt}
\hspace{3.2pt}
\resizebox{0.9537\textwidth}{!}{
\begin{tabular}{|c|c | c|c | c|}
\hline 
 \multirow{2}{*}{Objective}             & \multicolumn{2}{c|}{CIFAR-10} & \multicolumn{2}{c|}{Tiny ImageNet} \\ 
                         \cline{2-5}
                         &IS                & FID                & IS               & FID  \\
\hline
\hline 
$x$                      & $7.68\pm0.03$    & $18.35\pm0.12$    & $8.66\pm 0.04$    & $16.47\pm0.04$    \\
\hline 
\hline
$\exp(x)$                & $\bf8.03\pm0.03$ & $\bf15.64\pm0.07$ & $8.67\pm0.04$     & $\bf14.90\pm0.07$    \\
\hline 
$-\log(\sigma(-x))$      & $7.95\pm0.04$    & $16.47\pm0.11$    & $8.70\pm0.04$     & $15.05\pm0.07$ \\
\hline 
$x+\sqrt{x^2+1}$         & $7.97\pm0.03$    & $16.03\pm0.09$    & $\bf8.82\pm0.03$  & $15.11\pm0.06$  \\
\hline 
\hline
$(x+1)^2$              & $7.97\pm0.04$    & $15.90\pm0.09$    & $8.53\pm0.04$     & $15.72\pm0.11$  \\
\hline
$\max(0,x+1)$          & $7.91\pm0.04$    & $16.52\pm0.12$    & $8.63\pm0.04$     & $15.75\pm0.06$  \\
\hline 
\end{tabular}
}
\end{minipage}
\begin{minipage}{.33\textwidth}	
    \centering
    \includegraphics[width=0.99\columnwidth]{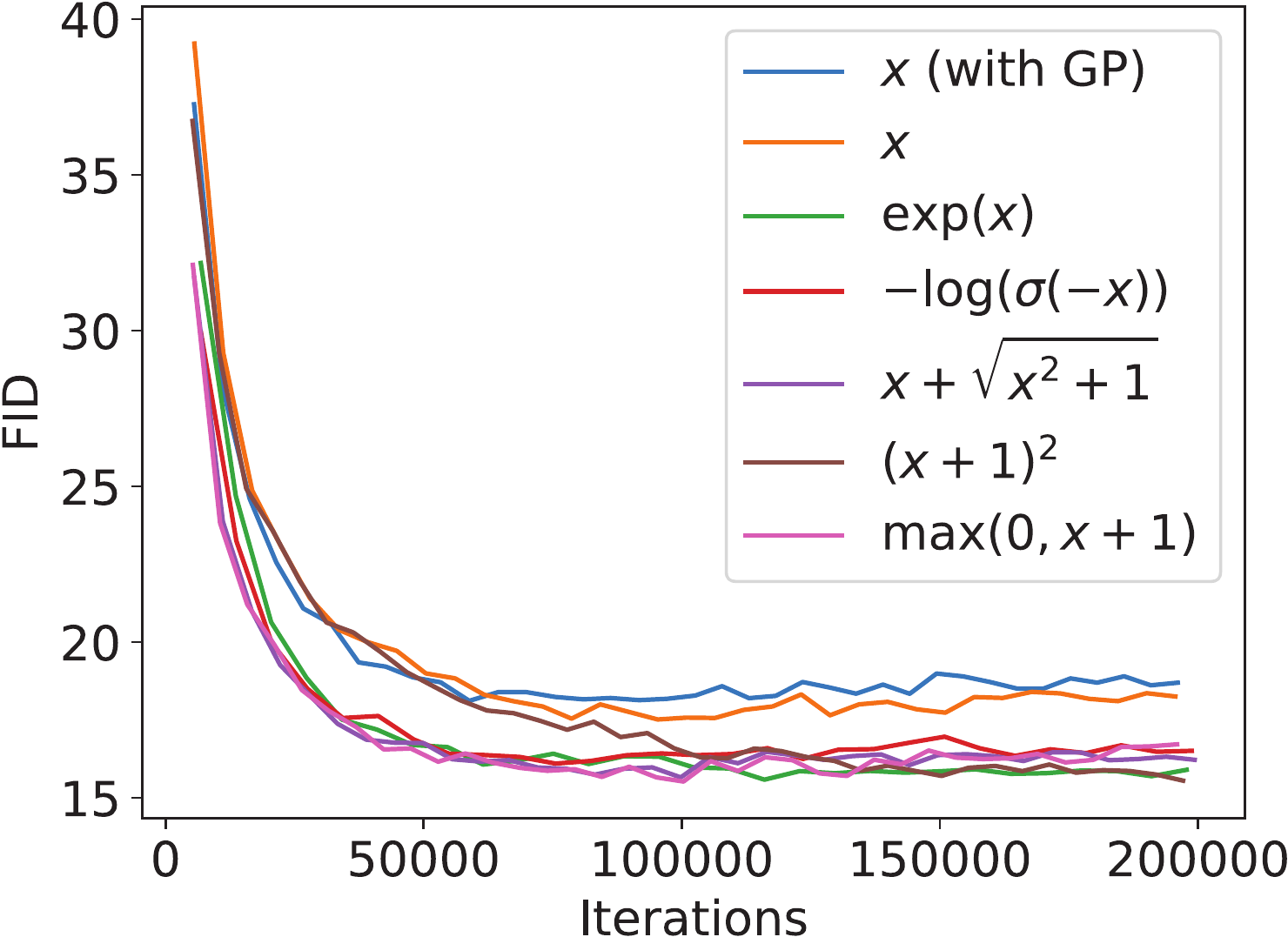}
    \vspace{-19pt}
    \caption{Training curves on CIFAR.}
    \label{training_curve_cifar10}
\end{minipage}
\vspace{-5pt} 
\end{figure*}  

\begin{figure*}[t]
\begin{minipage}{.66\textwidth}
\begin{subfigure}{0.49\linewidth}
    \includegraphics[width=0.99\columnwidth]{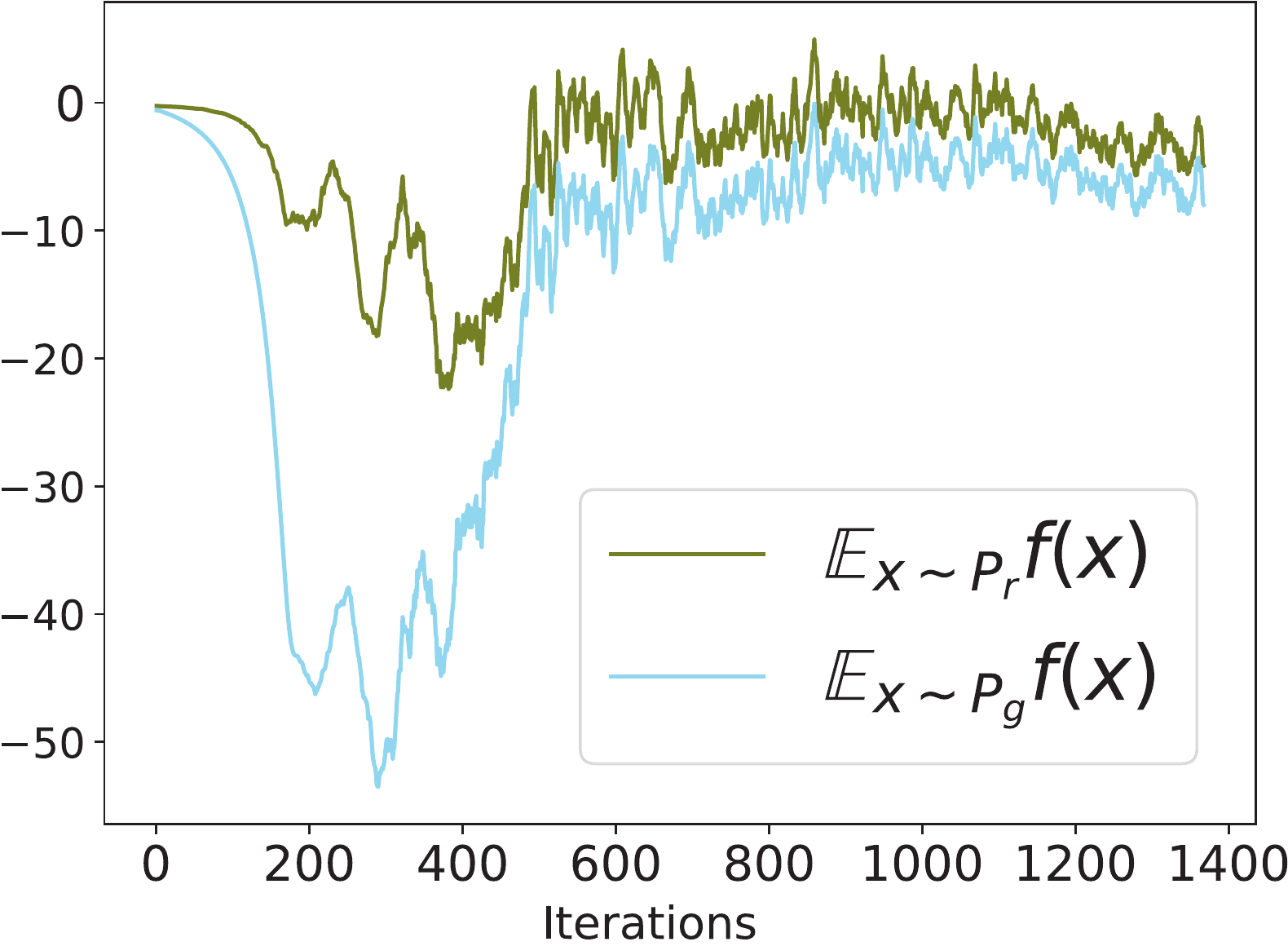}
    \label{fig5_a}
\end{subfigure}	
\begin{subfigure}{0.49\linewidth}
    \includegraphics[width=0.99\columnwidth]{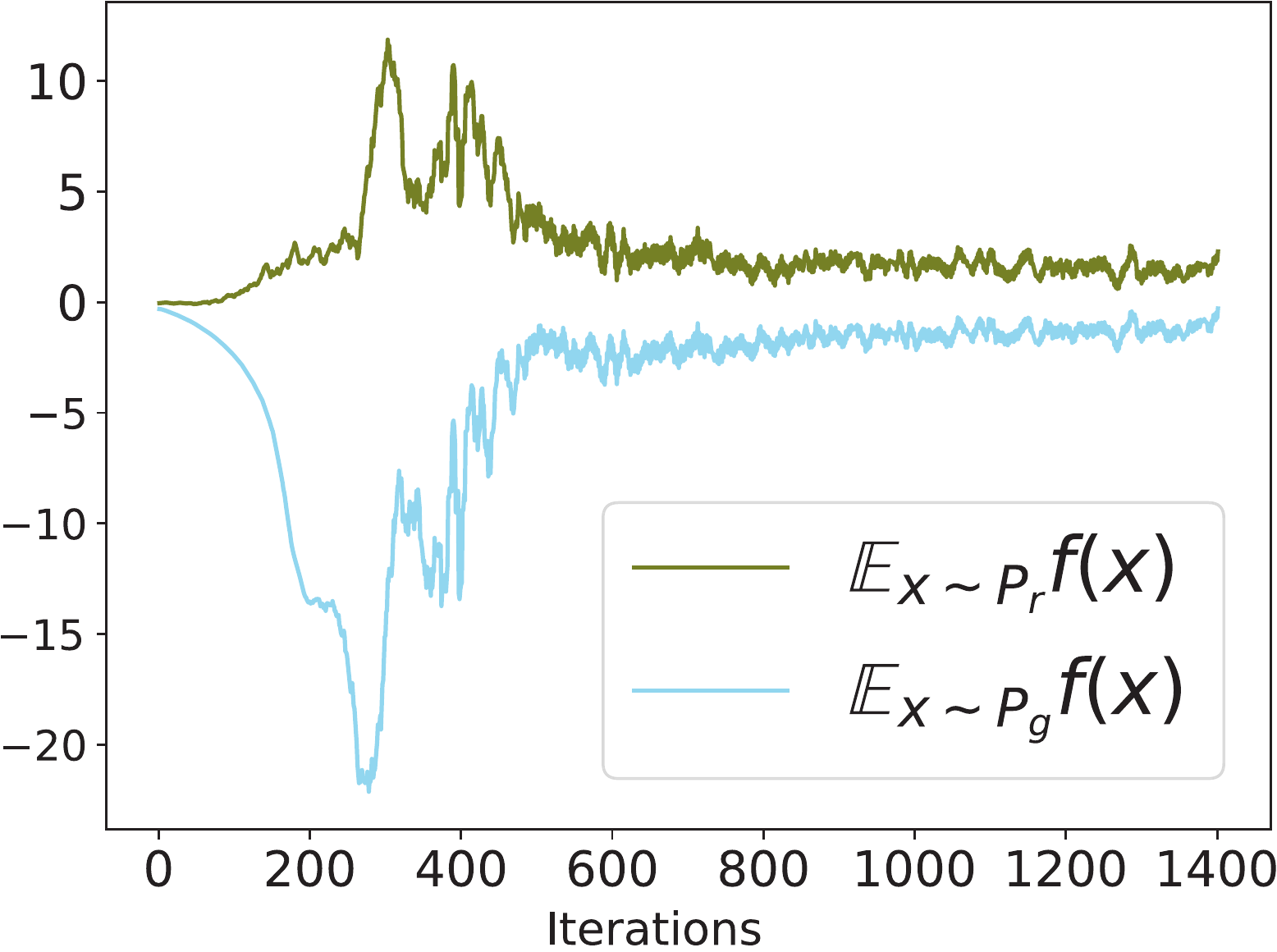}
    \label{fig5_b}
\end{subfigure}	
\vspace{-19pt}
\caption{$\ff(x)$ in LGANs is more stable. Left: WGAN. Right: LGANs.}
\label{f_oscillations}
\end{minipage}
\begin{minipage}{.33\textwidth}	
    \centering    
    \includegraphics[width=0.99\columnwidth]{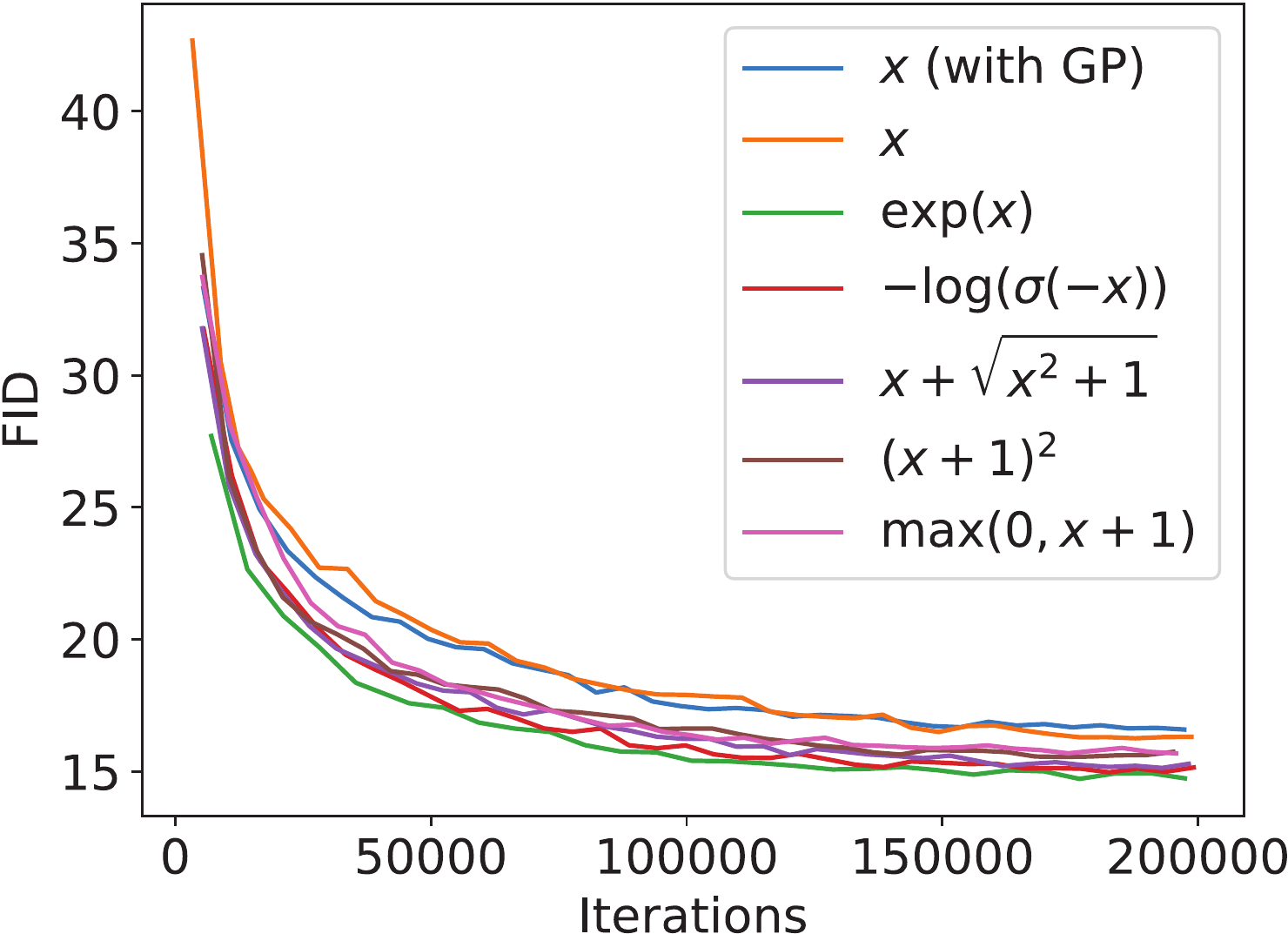}
    \vspace{-19pt}
    \caption{Training curves on Tiny.}
    \label{training_curve_tiny}
\end{minipage}
\vspace{-5pt}
\end{figure*} 

To quantitatively compare the performance of different objectives under Lipschitz constraint, we test them with unsupervised image generation tasks. In this part of  experiments, we also include the hinge loss $\phi(x)=\varphi(-x)=\max(0,x+\alpha)$ and quadratic loss \citep{lsgan}, which do not fit the assumption of strict monotonicity. For the quadratic loss, we set $\phi(x)=\varphi(-x)=(x+\alpha)^2$. 
To make the comparison simple, we fix $\psi(x)$ in the objective of generator as $-x$. We set $\alpha=1.0$ in the experiment. 

The strict monotonicity assumption of $\phi$ and $\varphi$ is critical in Theorem~\ref{theorem_main} to theoretically guarantee the existences of bounding relationships for \emph{arbitrary datas}. But if we further assume $S_r$ and $S_g$ are limited, it is possible that there exists a suitable $\lambda$ such that all real and fake samples lie in a strict monotone region of $\phi$ and $\varphi$: for the hinge loss, it would mean $2\alpha < k(f)\cdot \lVert y-x \rVert$ for all $y\in \CS_r$ and $x \in \CS_g$. 

The results in terms of Inception Score (IS) \citep{improved_gan} and Frechet Inception Distance (FID) \citep{heusel2017gans} are presented in Table~\ref{table2}. For all experiments, we adopt the network structures and hyper-parameter setting from \cite{wgangp}, where WGAN-GP in our implementation achieves IS $7.71\pm0.03$ and FID $18.86\pm0.13$ on CIFAR-10. We use MaxGP for all experiments and search the best $\lambda$ in $[0.01, 0.1, 1.0, 10.0]$. We use $200,000$ iterations for better convergence and use $500k$ samples to evaluate IS and FID for preferable stability. We note that IS is remarkably unstable during training and among different initializations. By contrast, FID is fairly stable. 

From Table~\ref{table2}, we can see that LGANs generally work better than WGAN. Different LGANs have relatively similar final results, while the objectives $\phi(x)=\varphi(-x)=\exp(x)$ and $\phi(x)=\varphi(-x)=x+\sqrt{x^2+1}$ achieve the best performances. 
The hinge loss and quadratic loss with a suitable $\lambda$ turn out to also work pretty good. We plot the training curves in terms of FID in Figures~\ref{training_curve_cifar10} and \ref{training_curve_tiny}. Due to page limitation, we leave more results and details in Appendix~\ref{sec_details}. 


\section{Related Work} \label{sec_related_work}

WGAN \citep{wgan} based on the KR dual does not suffer from the gradient uninformativeness problem. We have shown that the Lipschitz constraint in the KR dual of the Wasserstein distance can be relaxed. With the new dual form, the resulting model suffers from the gradient uninformativeness problem. 

We have shown that Lipschitz constraint is able to ensure the convergence for a family of GAN objectives, which is not limited to the Wasserstein distance.
For example, Lipschitz continuity is also introduced to the vanilla GAN \citep{sngan,kodali2017convergence,fedus2017many}, achieving improvements in the quality of generated samples. As a matter of fact, the vanilla GAN objective $\phi(x)=$ $\varphi(-x)=$ $-\log(\sigma(-x))$ is an special case of our LGANs. Thus our analysis explains why and how it works. \cite{convex_duality} also provide some analysis on how $f$-divergence behaviors when combined with Lipschitz. However, their analysis is limited to the symmetric $f$-divergence. 

\citet{fedus2017many} also argued that divergence is not the primary guide of the training of GANs. However, they thought that the vanilla GAN with a non-saturating generator objective  somehow works. According to our analysis, given the optimal $\ff$, the vanilla GAN has no guarantee on its convergence. 
\citet{coulombgan} provided some arguments on the unreliability of $\nabla_{\!x} \ff(x)$ in traditional GANs, which motivates their proposal of Coulomb GAN. However, the arguments there are not thorough. By contrast, we identify the gradient uninformativeness problem and link it to the restrictions on $\mathcal{F}$. Moreover, we have accordingly proposed a new solution, i.e., the Lipschitz GANs. 

Some work studies the suboptimal convergence of GANs \citep{mescheder2017numerics,mescheder2018training,arora2017generalization,liu2017approximation,convex_duality}, which is another important direction for theoretically understanding GANs. Despite the fact that the behaviors of suboptimal can be different, we think the optimal should well-behave in the first place, e.g., informative gradient and stable Nash equilibrium. 
Researchers found that applying Lipschitz continuity condition to the generator also benefits the quality of generated samples \citep{zhang2018self, odena2018generator}. And \cite{qi2017loss} studied the Lipschitz condition from the perspective of loss-sensitive with a Lipschitz data density assumption. 



\section{Conclusion} \label{sec_conclusion}

In this paper we have studied one fundamental cause of failure in the training of GANs, i.e., the gradient uninformativeness issue. In particular, for generated samples which are not surrounded by real samples, the gradients of the optimal discriminative function $\nabla_{\!x} \ff(x)$ tell nothing about $\CP_r$. That is, in a sense, there is no guarantee that $\CP_g$ will converge to $\CP_r$. Typical case is that $\CP_r$ and $\CP_g$ are disjoint, which is common in practice. The gradient uninformativeness is common for unrestricted GANs and also appears in restricted GANs. 


To address the nonconvergence problem caused by uninformative $\nabla_{\!x} \ff(x)$, we have proposed LGANs and shown that it makes $\nabla_{\!x} \ff(x)$ informative in the way that the gradient for each generated sample points towards some real sample. We have also shown that in LGANs, the optimal discriminative function exists and is unique, and the only Nash equilibrium is achieved when $\CP_r=\CP_g$ where $k(\ff)=0$. Our experiments shown LGANs lead to more stable discriminative functions and achieve higher sample qualities. 




\section*{Acknowledgements}

This work is sponsored by APEX-YITU Joint Research Program. The authors thank the support of National Natural Science Foundation of China (61702327, 61772333, 61632017), Shanghai Sailing Program (17YF1428200) and the helpful discussions with Dachao Lin.
Jiadong Liang and Zhihua Zhang have been supported by Beijing Municipal Commission of Science and Technology under Grant No.
181100008918005, and by Beijing Academy of Artificial Intelligence (BAAI). 

\bibliography{reference}

\begin{thebibliography}{32}
\providecommand{\natexlab}[1]{#1}
\providecommand{\url}[1]{\texttt{#1}}
\expandafter\ifx\csname urlstyle\endcsname\relax
  \providecommand{\doi}[1]{doi: #1}\else
  \providecommand{\doi}{doi: \begingroup \urlstyle{rm}\Url}\fi

\bibitem[Adler \& Lunz(2018)Adler and Lunz]{adler2018banach}
Adler, J. and Lunz, S.
\newblock {Banach} {Wasserstein} {GAN}.
\newblock \emph{arXiv preprint arXiv:1806.06621}, 2018.

\bibitem[Arjovsky \& Bottou(2017)Arjovsky and Bottou]{principled_methods}
Arjovsky, M. and Bottou, L.
\newblock Towards principled methods for training generative adversarial
  networks.
\newblock In \emph{ICLR}, 2017.

\bibitem[Arjovsky et~al.(2017)Arjovsky, Chintala, and Bottou]{wgan}
Arjovsky, M., Chintala, S., and Bottou, L.
\newblock {Wasserstein} {GAN}.
\newblock \emph{arXiv preprint arXiv:1701.07875}, 2017.

\bibitem[Arora et~al.(2017)Arora, Ge, Liang, Ma, and
  Zhang]{arora2017generalization}
Arora, S., Ge, R., Liang, Y., Ma, T., and Zhang, Y.
\newblock Generalization and equilibrium in generative adversarial nets
  ({GANs}).
\newblock \emph{arXiv preprint arXiv:1703.00573}, 2017.

\bibitem[Bellemare et~al.(2017)Bellemare, Danihelka, Dabney, Mohamed,
  Lakshminarayanan, Hoyer, and Munos]{bellemare2017cramer}
Bellemare, M.~G., Danihelka, I., Dabney, W., Mohamed, S., Lakshminarayanan, B.,
  Hoyer, S., and Munos, R.
\newblock The {Cramer} distance as a solution to biased {Wasserstein}
  gradients.
\newblock \emph{arXiv preprint arXiv:1705.10743}, 2017.

\bibitem[Farnia \& Tse(2018)Farnia and Tse]{convex_duality}
Farnia, F. and Tse, D.
\newblock A convex duality framework for {GANs}.
\newblock In \emph{Advances in Neural Information Processing Systems 31}. 2018.

\bibitem[Fedus et~al.(2017)Fedus, Rosca, Lakshminarayanan, Dai, Mohamed, and
  Goodfellow]{fedus2017many}
Fedus, W., Rosca, M., Lakshminarayanan, B., Dai, A.~M., Mohamed, S., and
  Goodfellow, I.
\newblock Many paths to equilibrium: {GANs} do not need to decrease divergence
  at every step.
\newblock \emph{arXiv preprint arXiv:1710.08446}, 2017.

\bibitem[Goodfellow(2016)]{gan_tutorial}
Goodfellow, I.
\newblock Nips 2016 tutorial: Generative adversarial networks.
\newblock \emph{arXiv preprint arXiv:1701.00160}, 2016.

\bibitem[Goodfellow et~al.(2014)Goodfellow, Pouget-Abadie, Mirza, Xu,
  Warde-Farley, Ozair, Courville, and Bengio]{gan}
Goodfellow, I., Pouget-Abadie, J., Mirza, M., Xu, B., Warde-Farley, D., Ozair,
  S., Courville, A., and Bengio, Y.
\newblock Generative adversarial nets.
\newblock In \emph{Advances in neural information processing systems}, pp.\
  2672--2680, 2014.

\bibitem[Gulrajani et~al.(2017)Gulrajani, Ahmed, Arjovsky, Dumoulin, and
  Courville]{wgangp}
Gulrajani, I., Ahmed, F., Arjovsky, M., Dumoulin, V., and Courville, A.
\newblock Improved training of {Wasserstein} {GANs}.
\newblock \emph{arXiv preprint arXiv:1704.00028}, 2017.

\bibitem[Heusel et~al.(2017)Heusel, Ramsauer, Unterthiner, Nessler, and
  Hochreiter]{heusel2017gans}
Heusel, M., Ramsauer, H., Unterthiner, T., Nessler, B., and Hochreiter, S.
\newblock {GANs} trained by a two time-scale update rule converge to a local
  nash equilibrium.
\newblock In \emph{Advances in Neural Information Processing Systems}, pp.\
  6626--6637, 2017.

\bibitem[Karras et~al.(2017)Karras, Aila, Laine, and
  Lehtinen]{progressive_growing_gan}
Karras, T., Aila, T., Laine, S., and Lehtinen, J.
\newblock Progressive growing of {GANs} for improved quality, stability, and
  variation.
\newblock \emph{arXiv preprint arXiv:1710.10196}, 2017.

\bibitem[Kodali et~al.(2017)Kodali, Abernethy, Hays, and
  Kira]{kodali2017convergence}
Kodali, N., Abernethy, J., Hays, J., and Kira, Z.
\newblock On convergence and stability of {GANs}.
\newblock \emph{arXiv preprint arXiv:1705.07215}, 2017.

\bibitem[Lim \& Ye(2017)Lim and Ye]{geometric_gan}
Lim, J.~H. and Ye, J.~C.
\newblock Geometric {GAN}.
\newblock \emph{arXiv preprint arXiv:1705.02894}, 2017.

\bibitem[Liu et~al.(2017)Liu, Bousquet, and Chaudhuri]{liu2017approximation}
Liu, S., Bousquet, O., and Chaudhuri, K.
\newblock Approximation and convergence properties of generative adversarial
  learning.
\newblock In \emph{Advances in Neural Information Processing Systems}, pp.\
  5545--5553, 2017.

\bibitem[Lucic et~al.(2017)Lucic, Kurach, Michalski, Gelly, and
  Bousquet]{lucic2017gans}
Lucic, M., Kurach, K., Michalski, M., Gelly, S., and Bousquet, O.
\newblock Are {GANs} created equal? a large-scale study.
\newblock \emph{arXiv preprint arXiv:1711.10337}, 2017.

\bibitem[Mao et~al.(2016)Mao, Li, Xie, Lau, Wang, and Smolley]{lsgan}
Mao, X., Li, Q., Xie, H., Lau, R.~Y., Wang, Z., and Smolley, S.~P.
\newblock Least squares generative adversarial networks.
\newblock \emph{arXiv preprint ArXiv:1611.04076}, 2016.

\bibitem[Mescheder et~al.(2017)Mescheder, Nowozin, and
  Geiger]{mescheder2017numerics}
Mescheder, L., Nowozin, S., and Geiger, A.
\newblock The numerics of {GANs}.
\newblock In \emph{Advances in Neural Information Processing Systems}, pp.\
  1825--1835, 2017.

\bibitem[Mescheder et~al.(2018)Mescheder, Geiger, and
  Nowozin]{mescheder2018training}
Mescheder, L., Geiger, A., and Nowozin, S.
\newblock Which training methods for {GANs} do actually converge?
\newblock In \emph{International Conference on Machine Learning}, pp.\
  3478--3487, 2018.

\bibitem[Miyato et~al.(2018)Miyato, Kataoka, Koyama, and Yoshida]{sngan}
Miyato, T., Kataoka, T., Koyama, M., and Yoshida, Y.
\newblock Spectral normalization for generative adversarial networks.
\newblock \emph{arXiv preprint arXiv:1802.05957}, 2018.

\bibitem[Mroueh \& Sercu(2017)Mroueh and Sercu]{fishergan}
Mroueh, Y. and Sercu, T.
\newblock Fisher {GAN}.
\newblock In \emph{Advances in Neural Information Processing Systems}, pp.\
  2510--2520, 2017.

\bibitem[Mroueh et~al.(2017)Mroueh, Li, Sercu, Raj, and Cheng]{sobolevgan}
Mroueh, Y., Li, C., Sercu, T., Raj, A., and Cheng, Y.
\newblock Sobolev {GAN}.
\newblock \emph{arXiv preprint arXiv:1711.04894}, 2017.

\bibitem[Nowozin et~al.(2016)Nowozin, Cseke, and Tomioka]{fgan}
Nowozin, S., Cseke, B., and Tomioka, R.
\newblock f-{GAN}: Training generative neural samplers using variational
  divergence minimization.
\newblock In \emph{Advances in Neural Information Processing Systems}, pp.\
  271--279, 2016.

\bibitem[Odena et~al.(2018)Odena, Buckman, Olsson, Brown, Olah, Raffel, and
  Goodfellow]{odena2018generator}
Odena, A., Buckman, J., Olsson, C., Brown, T.~B., Olah, C., Raffel, C., and
  Goodfellow, I.
\newblock Is generator conditioning causally related to {GAN} performance?
\newblock \emph{arXiv preprint arXiv:1802.08768}, 2018.

\bibitem[Petzka et~al.(2017)Petzka, Fischer, and Lukovnicov]{wganlp}
Petzka, H., Fischer, A., and Lukovnicov, D.
\newblock On the regularization of {Wasserstein} {GANs}.
\newblock \emph{arXiv preprint arXiv:1709.08894}, 2017.

\bibitem[Qi(2017)]{qi2017loss}
Qi, G.-J.
\newblock Loss-sensitive generative adversarial networks on lipschitz
  densities.
\newblock \emph{arXiv preprint arXiv:1701.06264}, 2017.

\bibitem[Salimans et~al.(2016)Salimans, Goodfellow, Zaremba, Cheung, Radford,
  and Chen]{improved_gan}
Salimans, T., Goodfellow, I., Zaremba, W., Cheung, V., Radford, A., and Chen,
  X.
\newblock Improved techniques for training {GANs}.
\newblock In \emph{Advances in Neural Information Processing Systems}, pp.\
  2226--2234, 2016.

\bibitem[Unterthiner et~al.(2017)Unterthiner, Nessler, Klambauer, Heusel,
  Ramsauer, and Hochreiter]{coulombgan}
Unterthiner, T., Nessler, B., Klambauer, G., Heusel, M., Ramsauer, H., and
  Hochreiter, S.
\newblock Coulomb {GANs}: Provably optimal nash equilibria via potential
  fields.
\newblock \emph{arXiv preprint arXiv:1708.08819}, 2017.

\bibitem[Villani(2008)]{oldandnew}
Villani, C.
\newblock \emph{Optimal Transport: Old and New}, volume 338.
\newblock Springer Science \& Business Media, 2008.

\bibitem[Yadav et~al.(2017)Yadav, Shah, Xu, Jacobs, and
  Goldstein]{yadav2017stabilizing}
Yadav, A., Shah, S., Xu, Z., Jacobs, D., and Goldstein, T.
\newblock Stabilizing adversarial nets with prediction methods.
\newblock \emph{arXiv preprint arXiv:1705.07364}, 2017.

\bibitem[Zhang et~al.(2018)Zhang, Goodfellow, Metaxas, and
  Odena]{zhang2018self}
Zhang, H., Goodfellow, I., Metaxas, D., and Odena, A.
\newblock Self-attention generative adversarial networks.
\newblock \emph{arXiv preprint arXiv:1805.08318}, 2018.

\bibitem[Zhao et~al.(2016)Zhao, Mathieu, and LeCun]{energy_based_gan}
Zhao, J., Mathieu, M., and LeCun, Y.
\newblock Energy-based generative adversarial network.
\newblock \emph{arXiv preprint arXiv:1609.03126}, 2016.

\end{thebibliography}
\bibliographystyle{icml2019}

\newpage
\onecolumn
\appendix

\section{Proofs}

\subsection{Proof of Theorem~\ref{thm:esistence}}

Let $X,Y$ be two random vectors such that $X\sim \CP_g,Y\sim \CP_r$. Assume $\E_{X\sim \CP_g}\lVert X \rVert<\infty$ and $\E_{Y\sim \CP_r}\lVert Y\rVert<\infty$. Let $\mathfrak{G}(f) = \E_{X\sim \CP_g}\phi(f(X))+\E_{Y\sim \CP_r}\varphi(f(Y))$. Let $\lVert f\rVert_{Lip}$ denote the Lipschitz constant of $f$. Let $\CS_r$ and $\CS_g$ denote the supports of $\CP_r$ and $\CP_g$, respectively. Let $W_1(\CP_r,\CP_g)$ denote the $1$-st Wasserstein distance between $\CP_r$ and $\CP_g$. 

\begin{lemma} \label{lemma_lower_bound}
Let $\phi$ and $\varphi$ be two convex functions, whose domains are both $\BR$. Assume $f$ is subject to $\lVert f\rVert_{Lip}\le k$. If there is $a_0 \in \BR$ such that $\phi'(a_0) + \varphi'(a_0)=0$, then we have a lower bound for $\mathfrak{G}(f)$. 
\end{lemma}

\begin{proof}
Given that $\phi, \varphi$ are convex functions, we have
\begin{equation}
\begin{aligned} 
	\mathfrak{G}(f)&=\E_{X\sim \CP_g}\phi(f(X))+\E_{Y\sim \CP_r}\varphi(f(Y))\\
               &\ge \E_{X\sim \CP_g}(\phi'(a_0)(f(x)-a_0) + \phi(a_0))+\E_{Y\sim \CP_r}(\varphi'(a_0)(f(x)-a_0)+\varphi(a_0))\\
               &= \phi'(a_0)\E_{X\sim \CP_g}f(x)+\varphi'(a_0)\E_{Y\sim \CP_r}f(Y)+C\\
               &=(\phi'(a_0)+\varphi'(a_0))\E_{X\sim \CP_g}f(X)+\varphi'(a_0)(\E_{Y\sim \CP_r}f(Y)-\E_{X\sim \CP_g}f(X))+C\\
               &= k\varphi'(a_0)(\E_{Y\sim \CP_r}\frac{1}{k}f(Y)-\E_{X\sim \CP_g}\frac{1}{k}f(X))+C\\
               &\ge -k\varphi'(a_0)W_1(\CP_r,\CP_g)+C.
\end{aligned}
\end{equation}
Therefore, we get the lower bound.
\end{proof}

\begin{lemma} \label{lemma_infty}
Let $\phi$ and $\varphi$ be two convex functions, whose domains are both $\BR$. Assume $f$ is subject to $\lVert f\rVert_{Lip}\le k$. 
\begin{itemize}
    \item If there exists $a_1 \in \BR$ such that $\phi'(a_1) + \varphi'(a_1)>0$, then we have: if $f(0)\to +\infty$, then $\mathfrak{G}(f) \to +\infty$;
    \item If there exists $a_2 \in \BR$ such that $\phi'(a_2) + \varphi'(a_2)<0$, then we have: if $f(0)\to -\infty$, then $\mathfrak{G}(f) \to +\infty$. 
\end{itemize} 
\end{lemma}
\begin{proof}
	Since $\phi, \varphi$ are convex functions, we have 
	\begin{equation}
	\begin{aligned} 
	\mathfrak{G}(f)&=\E_{X\sim \CP_g}\phi(f(X))+\E_{Y\sim \CP_r}\varphi(f(Y))\\
	&\ge \E_{X\sim \CP_g}(\phi'(a_1)(f(x)-a_1) + \phi(a_1))+\E_{Y\sim \CP_r}(\varphi'(a_1)(f(x)-a_1)+\varphi(a_1))\\
	&= \phi'(a_1)\E_{X\sim \CP_g}f(x)+\varphi'(a_1)\E_{Y\sim \CP_r}f(Y)+C_1\\
	&=(\phi'(a_1)+\varphi'(a_1))\E_{X\sim \CP_g}f(X)+\varphi'(a_1)(\E_{Y\sim \CP_r}f(Y)-\E_{X\sim \CP_g}f(X))+C_1\\
	&=(\phi'(a_1)+\varphi'(a_1))\E_{X\sim \CP_g}f(X) +k\varphi'(a_1)(\E_{Y\sim \CP_r}\frac{1}{k}f(Y)-\E_{X\sim \CP_g}\frac{1}{k}f(X))+C_1\\
	&\ge(\phi'(a_1)+\varphi'(a_1))\E_{X\sim \CP_g}f(X)  -k\varphi'(a_1)W_1(\CP_r,\CP_g)+C_1\\
	&\ge (\phi'(a_1)+\varphi'(a_1))f(0) -k(\phi'(a_1)+\varphi'(a_1))\E_{X\sim \CP_g}\lVert X \rVert -k\varphi'W_1(\CP_r,\CP_g)+C_1     .   \end{aligned}
	\end{equation}
	Thus, if $f(0)\to +\infty$, then $\mathfrak{G}(f)\to +\infty$. And we can prove the other case symmetrically. 
\end{proof}

\begin{lemma} \label{lemma_aaa}
Let $\phi$ and $\varphi$ be two convex functions, whose domains are both $\BR$. If $\phi$ and $\varphi$ satisfy the following properties: 
\begin{itemize} 
	\item $\phi' \geq 0,\varphi'\leq 0$;
	\item There exist $a_0,a_1,a_2\in \BR$ such that $\phi'(a_0) + \varphi'(a_0)=0,\phi'(a_1)+\varphi'(a_1)>0,\phi'(a_2)+\varphi'(a_2)<0$.
\end{itemize} 
Then we have $\mathfrak{G}(f) = \E_{X\sim \CP_r}\phi(f(X))+\E_{Y\sim \CP_g}\varphi(f(Y))$, where $f$ is subject to $\lVert f\rVert_{Lip}\le k$ , has global minima. 

That is, $\exists \ff, \,\, s.t.$ 
\begin{itemize}
	\item $\lVert \ff\rVert_{Lip}\le k$; 
	\item $\forall f \,\, s.t. \,\, \lVert f\rVert_{Lip}\le k$, we have $\mathfrak{G}(\ff)\le \mathfrak{G}(f)$. 
\end{itemize}
\end{lemma}

\begin{proof}
According to Lemma~\ref{lemma_lower_bound}, $\mathfrak{G}(f)$ has a lower bound, which means $inf(\mathfrak{G}(f)) > -\infty$. Thus we can get a series of functions $\left\{f _n\right\}_{n=1}^\infty$ such that $\lim_{n\to \infty}\mathfrak{G}(f_n)=inf(\mathfrak{G}(f))$. Suppose that $\left\{r_i\right\}_{i=1}^\infty$ is the sequence of all rational points in $dom(f)$. Due to Lemma~\ref{lemma_infty}, for any $x\in \BR$, $\left\{f_n(x)|n\in \BR\right\}$ is bounded. By Bolzano-Weierstrass theorem, there is a subsequence $\left\{f_{1n}\right\}\subseteq \left\{f_n\right\}$ such that $\left\{f_{1n}(r_1)\right\}_{n=1}^\infty$ converges. And there is a subsequence $\left\{f_{2n}\right\}\subseteq \left\{f_{1n}\right\}$ such that $\left\{f_{2n}(r_2)\right\}_{n=1}^\infty$ converges. As for $r_i$, there  is a subsequence $\left\{f_{in}\right\}\subseteq \left\{f_{i-1n}\right\}$ such that   $\left\{f_{in}(r_i)\right\}_{n=1}^\infty$ converges. Then the sequence $\left\{f_{nn}\right\}_{n=1}^\infty$ will converge at  $r_i$. 

Furthermore, for all $x\in dom(f)$, we claim that $\left\{f_{nn}\right\}_{n=1}^\infty $ converges at $x$. Actually, $\forall \epsilon > 0 $, find $r\in \left\{r_i\right\}$ such that $\|x-r\|\le \frac{\epsilon}{10k}$, we have
\begin{equation} 
\begin{aligned}
    \lim_{m,l\to \infty}|f_{mm}(x)-f_{ll}(x)|
    &\le \lim_{m,l \to \infty}(|f_{mm}(x)-f_{mm}(r)|+|f_{mm}(r)-f_{ll}(r)|+|f_{ll}(r)-f_{ll}(x)|)\\
    &\le \lim_{m,l\to \infty}(\frac{\epsilon}{10}+\frac{\epsilon}{10}+|f_{mm}(r)-f_{ll}(r)|)=\frac{\epsilon}{5}
\end{aligned}
\end{equation} 	
Let $\epsilon\to 0$, then we get $\lim _{m,l\to\infty}|f_{mm}(x)-f_{ll}(x)|=0 $. 

We denote $\left\{f_{nn}\right\}_{n=1}^ \infty$ as $\left\{g_n\right\}_{n=1}^\infty$ and $\left\{g_n\right\}_{n=1}^\infty $ converges to $g$. Due to Lemma~\ref{lemma_infty}, we know that $\exists C'$ such that $|g_n(0)|\le C', \,\,\forall n \in \BN$. Because $\phi '\geq 0 ,\varphi'\leq 0$, we have
\begin{equation} 
\begin{aligned}
\phi(g_n(x))&\ge \phi(g_n(0)-k\|x\|)\ge \phi(-C'-k\|x\|)\ge\phi'(a_0)(-C'-k\|x\|-a_0)+\phi(a_0)=-k\phi'(a_0)\|x\|+C^{''}
\end{aligned}
\end{equation}
That is, $ \phi(g_n(x))+k\phi'(a_0)\|x\|-C^{''}\ge 0$.

By Fatou's Lemma, 
\begin{equation} 
\begin{aligned}
	\E_{X\sim \CP_g}(\phi(g(X))+k\phi'(a_0)\|X\|-C^{''})
	    &=\E_{X\sim \CP_g}\varliminf_{n\to \infty}(\phi(g_n(X))+k\phi'(a_0)\|X\|-C^{''})\\
	    &\le\varliminf_{n\to \infty}\E_{X\sim \CP_g}(\phi(g_n(X))+k\phi'(a_0)\|X\|-C^{''})\\
	    &=\varliminf_{n\to\infty}\E_{X\sim \CP_g}\phi(g_n(X)) + \E_{X\sim \CP_g}(k\phi'(a_0)\|X\|-C^{''})
\end{aligned}
\end{equation}
It means $\E_{X\sim \CP_g}\phi(g(X))\le \varliminf_{n\to \infty}\E_{X\sim \CP_g}\phi(g_n(X))$. Similarly, we have $\E_{Y\sim \CP_r}\varphi(g(Y))\le \varliminf_{n\to \infty}\E_{Y\sim \CP_r}\varphi(g_n(Y))$. Combining the two inequalities, we have
\begin{equation}
\begin{aligned}
\mathfrak{G}(g)
&=\E_{X\sim \CP_g}\phi(g(X))+\E_{Y\sim \CP_r}\varphi(g(Y))\le \varliminf_{n\to \infty}\E_{X\sim \CP_g}\phi(g_n(X))+\varliminf_{n\to \infty}\E_{Y\sim \CP_r}\varphi(g_n(Y))\\
&\le\varliminf_{n\to \infty}(\E_{X\sim \CP_g}\phi(g_n(X))+\E_{Y\sim \CP_r}\varphi(g_n(Y)))=\inf_{\|f\|_{Lip}\le k}\mathfrak{G}(f)
\end{aligned}
\end{equation}

Note that for any $x,y\in dom(g)$, $|g(x)-g(y)|\le\lim_{n\to\infty}(|g(x)-g_n(x)|+|g_n(x)-g_n(y)|+|g_n(y)-g(y)|)\le k\|x-y\|$. That is, $\|g\|_{Lip}\le k$, $\mathfrak{G}(g)=\inf_{\|f\|_{Lip}\le k}\mathfrak{G}(f)$. 
\end{proof}	

\begin{lemma}[Wasserstein distance]
$\mathfrak{T}(f)=\E_{X\sim \CP_g}f(X) - \E_{Y\sim \CP_r}f(Y)$, where $f$ is subject to $\|f\|_{Lip}\le k$, has global minima. 
\end{lemma}

\begin{proof}
It is easy to find that for any $C\in \BR$, $\mathfrak{T}(f+C)=\mathfrak{T}(f)$. Similar to the previous lemma, we can get a series of functions $\left\{f _n\right\}_{n=1}^\infty$ such that $\lim_{n\to \infty}\mathfrak{T}(f_n)=inf(\mathfrak{T}(f))$. Without loss of generality, we assume that $f_n(0)=0, \forall n\in \BN^+$. Because $\|f_n\|_{Lip}\le k$, we can claim that for any $x\in \BR$, $\left\{f_n(x)|n\in \BR\right\}$ is bounded. Then we can imitate the method used in Lemma~\ref{lemma_aaa} and find the optimal function $\ff$ such that $\mathfrak{T}(\ff) =\inf\limits_{\|f\|_{Lip}\le k}\mathfrak{T}(f)$. 
\end{proof} 

\begin{lemma} \label{lemma_k_exist}
Let $\phi$ and $\varphi$ be two convex functions, whose domains are both $\BR$. If we further suppose that the support sets $\CS_r$ and $\CS_g$ are bounded. Then if $\phi$ and $\varphi$ satisfy the following properties: 
\begin{itemize}
    \item ${\phi}^\prime \ge0,{\varphi}^\prime\le0$;
    \item There is $a_0\in \BR$ such that $\phi^\prime(a_0)+\varphi^\prime(a_0)=0$.
\end{itemize} 
We have $\mathfrak{G}(f) = \E_{X\sim \CP_g}\phi(f(X))+\E_{Y\sim \CP_r}\varphi(f(Y))$, where $f$ is subject to $\begin{Vmatrix}f\end{Vmatrix}_{Lip}\le k$, has global minima. 

That is, $\exists \ff,  \,\,  s.t.$ 
\begin{itemize}
\item $\begin{Vmatrix}\ff\end{Vmatrix}_{Lip}\le  k$
\item $\forall f \,\, s.t.\,\, \begin{Vmatrix}f\end{Vmatrix}_{Lip}\le k $ , we have $\mathfrak{G}(\ff)\le \mathfrak{G}(f)$.
\end{itemize} 
\end{lemma}

\begin{proof}
We have proved most conditions in previous lemmas. And we only have to consider the condition that for any $x\in \BR$, $\phi^\prime(x)+\varphi^\prime(x)\ge 0$ (or $\phi^\prime(x)+\varphi^\prime(x)\le0$) and there exists $a_1$ such that $\phi^\prime(a_1)+\varphi^\prime(a_1)>0$ (or $\phi^\prime(a_1)+\varphi^\prime(a_1)<0$). 

Without loss of generality, we assume that $\phi^\prime(x)+\varphi^\prime(x)\ge0$ for all $x$ and there exists $a_1$ such that $\phi^\prime(a_1)+\varphi^\prime(a_1)>0$. Then we know $\forall x\le a_0 ,\,\, \phi^\prime(x)+\varphi^\prime(x)=0$, which leads to $\forall x\le a_0 ,\,\, \phi^\prime(x)=-\varphi^\prime(x)$. Thus, for any $x\le a_0$, $0\le \phi^{\prime\prime}(x)=-\varphi^{\prime\prime}(x)\le 0$, which means $\forall x\le a_0,\,\, \phi(x)=-\varphi(x)=tx ,\,\, t\ge 0$. 
Similar to the previous lemmas, we can get a series of functions $\left\{f _n\right\}_{n=1}^\infty$ such that $\lim_{n\to \infty}\mathfrak{G}(f_n)=inf(\mathfrak{G}(f))$. Actually we can assume that for all $n\in \BN^+$, there is $f_n(0)\in [-C,C]$, where $C$ is a constant. In fact, it is not difficult to find $f_n(0)\le C$ with Lemma~\ref{lemma_infty}. On the other hand, when $C > k\cdot diam(\CS_r\cup \CS_g)+a_0$, then: if $f(0)<-C$, we have $f(X)<a_0$ for all $X\in \CS_r\cup \CS_g$. In this case, $\mathfrak{G}(f)=\mathfrak{G}(f-f(0)-C)$. This is the reason we can assume $f_n(0)\in [-C,C]$.  Because $\|f_n\|_{Lip}\le k$, we can assert that for any $x\in \BR$, $\left\{f_n(x)|n\in \BR\right\}$ is bounded. So we can imitate the method used in Lemma~\ref{lemma_aaa} and find the optimal function $\ff$ such that $\mathfrak{G}(\ff) =\inf\limits_{\|f\|_{Lip}\le k}\mathfrak{G}(f)$. 
\end{proof}

\begin{lemma}[Theorem~\ref{thm:esistence} Part I]
	Under the same assumption of Lemma~\ref{lemma_k_exist}, we have $\mathfrak{F}(f) = \E_{X\sim \CP_g}\phi(f(X))+\E_{Y\sim \CP_r}\varphi(f(Y))+\lambda\|f\|^\alpha_{Lip}$ with $\lambda>0$ and $\alpha>1$ has global minima.
\end{lemma}
\begin{proof}
	When $\|f\|_{Lip}=\infty$, it is trivial that $\mathfrak{F}(f)=\infty$. And when $\|f\|_{Lip}<\infty$,  combining  Lemma~\ref{lemma_lower_bound}, we have $\mathfrak{F}(f)=\mathfrak{G}(f)+\lambda\|f\|^\alpha_{Lip}\ge -\|f\|_{Lip}\varphi'(a_0)W_1(\CP_r,\CP_g) + \lambda\|f\|^\alpha_{Lip}$. When $\lambda>0$ and $\alpha>1$, the right term is a convex function about $\|f\|_{Lip}$, it has a lower bound. So we can find a sequence $\left\{f_n\right\}_{n=1}^\infty$ such that $\lim_{n\to \infty} \mathfrak{F}(f_n)=\inf_{f\in dom}\mathfrak{F}(f)$. It is no doubt that there exists a constant $C$  such that $\|f_n\|_{Lip}\le C$ for all $f_n$. Then it is not difficult to show for any point $x$, $ \left\{f_n(x) \right\}$ is bounded. So we can imitate the method used in main theorem to find the sequence $\left\{g_n\right\}$ such that $\left\{g_n\right\}\subseteq \left\{f_n\right\}$ and $\left\{g_n \right\}_{n=1}^\infty$ converge at every point $x$. Suppose $\lim_{n\to \infty}g_n=g$, then by Fatou's Lemma, we have $\mathfrak{G}(g)\le \varliminf_{n\to \infty}\mathfrak{G}(g_n)$. 
	
	Next, We prove that $\|g\|_{Lip}\le \varliminf_{n\to \infty}\|g_n\|_{Lip}$. If  the claim holds, then $\mathfrak{F}(g)=\mathfrak{G}(g)+\lambda\|g\|^\alpha_{Lip}
	\le \varliminf_{n\to \infty}\mathfrak{G}(g_n)+\varliminf_{n\to \infty}\lambda\|g_n\|^\alpha_{Lip}
	\le \varliminf_{n\to \infty}(\mathfrak{G}(g_n)+\lambda\|g_n\|^\alpha_{Lip})=\inf \mathfrak{F}(f)$. Thus, the global minima exists. 		
	In fact, if $\|g\|_{Lip}> \varliminf_{n\to \infty}\|g_n\|_{Lip}$, then there exist $x,y$ such that $\frac{|g(x)-g(y)|}{\|x-y\|}\ge \varliminf_{n\to \infty}\|g_n\|_{Lip} +\epsilon\ge \varliminf_{n\to \infty}\frac{|g_n(x)-g_n(y)|}{\|x-y\|}+\epsilon$. i.e. $|g(x)-g(y)|\ge \varliminf_{n\to\infty}|g_n(x)-g_n(y)|+\epsilon\|x-y\|=|g(x)-g(y)|+\epsilon\|x-y\|>|g(x)-g(y)|$. The contradiction tells us that $\|g\|_{Lip}\le\varliminf_{n\to\infty}\|g_n\|_{Lip}$. 
\end{proof}

\begin{lemma}[Theorem~\ref{thm:esistence} Part II]
Let $\phi$ and $\varphi$ be two convex functions, whose domains are both $\BR$. If $\phi$ or $\varphi$ is strictly convex, then the minimizer of $\mathfrak{F}(f) = \E_{X\sim \CP_g}\phi(f(X))+\E_{Y\sim \CP_r}\varphi(f(Y))+\lambda\|f\|^\alpha_{Lip}$ with $\lambda>0$ and $\alpha>1$ is unique (in the support of $\CS_r\cup\CS_g$). 
\end{lemma} 

\begin{proof}
Without loss of generality, we assume that $\phi$ is strictly convex. By the strict convexity of $\phi$, we have $\forall x,y \in \BR,\,\, \phi(\frac{x+y}{2})<\frac{1}{2}(\phi(x)+\phi(y))$. Assume $f_1$ and $f_2$ are two different minimizers of $\mathfrak{F}(f)$. 

First, we have
\begin{equation} \label{eq_k_lip_ff}
\begin{aligned}
    \Big\lVert \frac{f_1+f_2}{2} \Big\rVert_{Lip} 
    &= \sup_{x,y} \frac{\frac{f_1(x)+f_2(x)}{2}-\frac{f_1(y)+f_2(y)}{2}}{\lVert x - y \rVert}  \\ 
    &\leq \sup_{x,y} \frac{1}{2}\frac{ |f_1(x)-f_1(y)| + |f_2(x)-f_2(y)|}{\lVert x - y \rVert}  \\ 
    &\leq \frac{1}{2} \Big (\sup_{x,y} \frac{ |f_1(x)-f_1(y)|}{\lVert x - y \rVert} + \sup_{x,y} \frac{|f_2(x)-f_2(y)|}{\lVert x - y \rVert} \Big ) \\ 
    & = \frac{1}{2} (\lVert f_1 \rVert_{Lip} + \lVert f_2 \rVert_{Lip}) .
\end{aligned}
\end{equation}
And given $\lambda>0$ and $\alpha>1$, we further have
\begin{equation}
\begin{aligned}
    \lambda \Big\lVert \frac{f_1+f_2}{2} \Big\rVert^\alpha_{Lip} 
    &\leq  \lambda \Big(\frac{1}{2} (\lVert f_1 \rVert_{Lip} + \lVert f_2 \rVert_{Lip})\Big)^\alpha  \\
    & \leq \lambda \frac{1}{2} (\lVert f_1 \rVert^\alpha_{Lip} + \lVert f_2 \rVert^\alpha_{Lip}). 
\end{aligned}
\end{equation}

Let $\mathfrak{F}(f_1) = \mathfrak{F}(f_2) =\inf\mathfrak{F}(f)$. Then we have 
\begin{equation}\begin{aligned}\mathfrak{G}\Big(\frac{f_1+f_2}{2}\Big)
&=\E_{X\sim \CP_g}\phi\Big(\frac{f_1+f_2}{2}\Big)+\E_{Y\sim \CP_r}\varphi\Big(\frac{f_1+f_2}{2}\Big) + \lambda \Big\lVert \frac{f_1+f_2}{2} \Big\rVert^\alpha_{Lip} \\
&< \E_{X\sim \CP_g}\Big(\frac{\phi(f_1)+\phi(f_2)}{2}\Big)+\E_{Y\sim \CP_r}\varphi\Big(\frac{f_1+f_2}{2}\Big) + \lambda \Big\lVert \frac{f_1+f_2}{2} \Big\rVert^\alpha_{Lip} \\
&\le \E_{X\sim \CP_g}\Big(\frac{\phi(f_1)+\phi(f_2)}{2}\Big)+\E_{Y\sim \CP_r}\Big(\frac{\varphi(f_1)+\varphi(f_2)}{2}\Big) + \lambda \Big\lVert \frac{f_1+f_2}{2} \Big\rVert^\alpha_{Lip}\\
&\le \E_{X\sim \CP_g}\Big(\frac{\phi(f_1)+\phi(f_2)}{2}\Big)+\E_{Y\sim \CP_r}\Big(\frac{\varphi(f_1)+\varphi(f_2)}{2}\Big) + \lambda \frac{1}{2} (\lVert f_1 \rVert^\alpha_{Lip} + \lVert f_2 \rVert^\alpha_{Lip}) \\
&=\frac{1}{2}(\mathfrak{G}(f_1)+\mathfrak{G}(f_2)) = \inf\mathfrak{G}(f) \end{aligned}
\end{equation}

We get a contradiction $\mathfrak{G}(\frac{f_1+f_2}{2}) < \inf \mathfrak{G}(f) $, which implies that the minimizer of $\mathfrak{G}(f)$ is unique. 
\end{proof}

\subsection{Proof of Theorem~\ref{theorem_main}}
\label{app_proof}

Let $J_D = \mE_{x\sim \CP_g} [\phi(f(x))]+\mE_{x\sim \CP_r}[\varphi(f(x))]$. Let $\mathring{J}_D(x) = \CP_g(x) \phi(f(x)) + \CP_r(x) \varphi(f(x)).$ Clearly, $J_D= \int_{\BR^n}{\mathring{J}_D(x) d x}$. Let $J_D^*(k)=\min_{f\in \mathcal{F_\text{\textbf{k}-Lip}}} J_D=\min_{f\in \mathcal{F_\text{1-Lip}}, b}\mE_{x\sim \CP_g} [\phi(k\cdot f(x)+b)]+\mE_{x\sim \CP_r}[\varphi(k\cdot f(x) + b)]$. 

Let $k(f)$ denote the Lipschitz constant of $f$. Define $J = J_D + \lambda \cdot k(f)^2$ and $\ff = \argmin_{f} [J_D + \lambda \cdot k(f)^2]$. 

\begin{lemma} \label{lem_k_neq_zero}
	It holds $\pdv{\mathring{J}_D(x)}{\ff(x)}=0$ for all $x$, if and only if, $k(\ff)=0$. 
\end{lemma}

\begin{proof} $ $ \newline \newline	
	(\rnum{1}) If $\pdv{\mathring{J}_D(x)}{\ff(x)}=0$ holds for all $x$, then $k(\ff)=0$. 
 
	For the optimal $\ff$, it holds that $\pdv{J}{k(\ff)} = \pdv{J_D^*}{k(\ff)} + 2\lambda\cdot k(\ff) = 0$. 
	
	$\pdv{\mathring{J}_D(x)}{\ff(x)}=0$ for all $x$ implies $\pdv{J_D^*}{k(\ff)}=0$. Thus we conclude that $k(\ff) = 0$.

	(\rnum{2}) If $k(\ff)=0$, then $\pdv{\mathring{J}_D(x)}{\ff(x)}=0$ holds for all $x$. 
	
	For the optimal $\ff$, it holds that $\pdv{J}{k(\ff)} = \pdv{J_D^*}{k(\ff)} + 2\lambda\cdot k(\ff) = 0$. 
	
	$k(\ff)=0$ implies $\pdv{J_D^*}{k(\ff)}=0$. $k(\ff)=0$ also implies $\,\forall x, y, \ff(x)=\ff(y)$. 
	
	Given $\,\forall x, y, \ff(x)=\ff(y)$, if there exists some point $x$ such that $\pdv{\mathring{J}_D(x)}{\ff(x)}\neq0$, then it is obvious that $\pdv{J_D^*}{k(\ff)}\neq0$. 
	
	It is contradictory to $\pdv{J_D^*}{k(\ff)}=0$. Thus we have $\,\forall x, \pdv{\mathring{J}_D(x)}{\ff(x)}=0$.
\end{proof}

\begin{lemma} \label{lemma_pair_f}
If $\,\,\forall x, y, \ff(x)=\ff(y)$, then $\CP_r=\CP_g$.
\end{lemma}

\begin{proof} 
$\,\,\forall x, y, \ff(x)=\ff(y)$ implies $k(\ff)=0$. According to Lemma \ref{lem_k_neq_zero}, for all $x$ it holds $\pdv{\mathring{J}_D(x)}{\ff(x)}  = 0$, i.e., $\CP_g(x)\pdv{\phi(\ff(x))}{\ff(x)} $ $ +\CP_r(x)\pdv{\varphi(\ff(x))}{\ff(x)} = 0$. Thus, $\frac{\CP_g(x)}{\CP_r(x)}=-\frac{\pdv{\varphi(\ff(x))}{\ff(x)}}{\pdv{\phi(\ff(x))}{\ff(x)}}$. That is, $\frac{\CP_g(x)}{\CP_r(x)}$ has a constant value, which straightforwardly implies $\CP_r=\CP_g$.
\end{proof}

\begin{proof}[\textbf{Proof of Theorem \ref{theorem_main}}] $ $ \newline \newline
(a): Let $k$ be the Lipschitz constant of $\ff$. Consider $x$ with $\pdv{\mathring{J}_D(x)}{\ff(x)} \neq 0$. Define $k(x) = \sup_y \frac{|f(y)-f(x)|}{\lVert y-x \rVert}$. 

$\quad$ (i) If $\forall \delta$ s.t. $\forall \epsilon$ there exist $z, w \in B(x, \epsilon)$ such that $\frac{|\ff(z)-\ff(w)|}{\lVert z-w\rVert}\geq k - \delta$, which means there exists $t$ such that $f'(t)\geq k-\delta$, because $\frac{|\ff(z)-\ff(w)|}{\lVert z-w\rVert} = \frac{\int^z_{w}{\ff'(t)d t}}{\lVert z-w\rVert}$. Let $\epsilon\to 0$, we have $t\to x$. Then $|\ff'(t)| \to |\ff'(x)|$. Let $\delta \to 0$, we have $(k-\delta) \to k$. Assume $\ff$ is smooth, we have that $|f'(x)|=k$, which means there exists a $y$ such that $|\ff(y)-\ff(x)|=k\lVert y-x \rVert$. 

$\quad$ (ii) Assume that $\exists \delta$ s.t. $\exists \epsilon$ and for all $z, w \in B(x, \epsilon)$,  $\frac{|\ff(z)-\ff(w)|}{\lVert z-w\rVert} < k - \delta$.  Consider the following condition, for all $\delta_2$ and $\epsilon_2\in(0, \epsilon/2)$, $\exists y\in B(x, \epsilon_2)$, such that $k(y) > k-\delta_2$. Then there exists a sequence of $\{y_n\}^\infty_{n=1}$ s.t. $\lim_{n\to \infty} \frac{|f(y)-f(y_n)}{\lVert y-y_n \rVert} = k(y)$. Then there exists a $y'$ such that $\frac{|f(y)-f(y')}{\lVert y-y' \rVert} \geq k - \delta_2$. According to the assumption, we have $\lVert y-y'\rVert\geq \frac{\epsilon}{2}$. Then $k(x) \geq \frac{|\ff(x)-\ff(y)|}{\lVert x-y\rVert}\geq \frac{|\ff(y)-\ff(y')|-|\ff(x)-\ff(y)|}{\rVert x-y\lVert + \lVert y-y'\rVert} \geq \frac{|\ff(y)-\ff(y')|-k\rVert x-y\lVert}{\rVert x-y\lVert + \lVert y-y'\rVert} \geq (k - \delta_2) \frac{\lVert y-y'\rVert}{\rVert x-y\lVert + \lVert y-y'\rVert} - k\frac{\rVert x-y\lVert}{\rVert x-y\lVert + \lVert y-y'\rVert}\geq (1-\frac{\epsilon_2}{\epsilon_2+\lVert y-y'\rVert})(k - \delta_2)-k\frac{\epsilon_2}{\lVert y-y'\rVert}\geq (1-\frac{\epsilon_2}{\epsilon_2+\lVert y-y'\rVert})(k-\delta_2)-k\frac{\epsilon_2}{\lVert y-y'\rVert}.$  Let $\epsilon_2 \to 0$ and $\delta_2 \to 0$. We get $k(x)=k$, which means there exists a $y$ such that $|\ff(y)-\ff(x)|=k\lVert y-x \rVert$. 

$\quad$ (iii) Now we can assume $\exists \delta_2$ s.t. $\exists \epsilon_2$ and for all $y\in B(x, \epsilon_2)$, such that $k(y) \leq k-\delta_2$. If $\pdv{\mathring{J}_D(x)}{\ff(x)} \neq 0$, without loss of generality, we can assume $\pdv{\mathring{J}_D(x)}{\ff(x)} >0$. Then, for all $y\in B(x, \epsilon_2)$, we have $\pdv{\mathring{J}_D(y)}{\ff(y)} >0$, as long as $\epsilon_2$ is small enough. Now we change the value of $\ff(y)$ for $y \in B(x,\epsilon_2)$. Let $g(y)=\begin{cases}
\ff(y)-\frac{\epsilon_2}{N}(1-\frac{\lVert x-y \rVert}{\epsilon_2}), \quad y \in B(x,\epsilon_2); \\
\ff(y) \quad\quad  \text{otherwise}.
\end{cases}$. 
Because $\pdv{\mathring{J}_D(y)}{\ff(y)} >0$, $\forall y \in B(x,\epsilon_2)$, when $N$ is sufficiently large, it is not difficult to show $J_D(g) < J_D(\ff)$. We next verify that $\lVert g \rVert_{Lip} \leq k$. For any $y,z$, if $y,z \notin B(x, \epsilon_2)$, then $\frac{|g(y)-g(z)|}{\lVert y-z\rVert} = \frac{|\ff(y)-\ff(z)|}{\lVert y-z\rVert} < k$. If $y \in B(x, \epsilon_2)$, $z \notin B(x, \epsilon_2)$, then $\frac{|g(y)-g(z)|}{\lVert y-z \rVert} \leq  \frac{|(\ff(y)-\ff(z)| + \frac{\epsilon_2}{N}(1-\frac{\lVert x-y \rVert}{\epsilon_2}))}{{\lVert y-z\rVert}} \leq \frac{|\ff(y)-\ff(z)|}{\lVert y-z\rVert} + \frac{\frac{\epsilon_2}{N}(1-\frac{\lVert x-y \rVert}{\epsilon_2})}{\epsilon_2 - \lVert x-y \rVert} = \frac{|(\ff(y)-\ff(z)|}{{\lVert y-z\rVert}} + \frac{1}{N} \leq k(y) + \frac{1}{N} \leq k -\delta_2 + \frac{1}{N} < k$ (when $N \gg \frac{1}{\delta_2}$). If $y, z\in B(x, \epsilon)$, then $\frac{|g(y)-g(z)|}{\lVert y-z \rVert} \leq \frac{|\ff(y)-\ff(z)|+|\frac{\epsilon_2}{N}(1-\frac{\lVert x-y \rVert}{\epsilon_2})-\frac{\epsilon_2}{N}(1-\frac{\lVert x-z \rVert}{\epsilon_2})|}{\lVert y-z \rVert} = \frac{|\ff(y)-\ff(z)|}{\lVert y-z \rVert} + \frac{\frac{\epsilon_2}{N}(\frac{\lVert x-y \rVert-\lVert x-z \rVert}{\epsilon_2})|}{\lVert y-z \rVert} \leq \frac{|\ff(y)-\ff(z)|}{\lVert y-z \rVert} + \frac{1}{N}\frac{\lVert y-z \rVert}{\lVert y-z \rVert} = \frac{|\ff(y)-\ff(z)|}{\lVert y-z \rVert} + \frac{1}{N} \leq k - \delta_2 + \frac{1}{N} < k$ (when $N \gg \frac{1}{\delta_2}$). So, we have $\lVert g \rVert_{Lip} \leq k$. But we have $J_D(g) < J_D(\ff)$. The contradiction tells us that there must exists a $y$ such that $|\ff(y)-\ff(x)| = k \lVert y-x\rVert$. 


(b): For $x \in \CS_r\cup\CS_g - \CS_r\cap\CS_g $, assuming $\CP_g(x)\neq0$ and $\CP_r(x)=0$, we have $\pdv{\mathring{J}_D(x)}{\ff(x)} = \CP_g(x)\pdv{\phi(\ff(x))}{\ff(x)} +\CP_r(x)\pdv{\varphi(\ff(x))}{\ff(x)}= \CP_g(x)\pdv{\phi(\ff(x))}{\ff(x)}>0$, because $\CP_g(x) >0$ and $\pdv{\phi(\ff(x))}{\ff(x)} >0$. Then according to (a), there must exist a $y$ such that $|\ff(y)-\ff(x)| =k(\ff) \cdot \lVert y - x \rVert$. The other situation can be proved in the same way. 

(c): According to Lemma \ref{lemma_pair_f}, in the situation that $\CP_r \neq \CP_g$, for the optimal $\ff$, there must exist at least one pair of points $x$ and $y$ such that $y\neq x$ and $\ff(x)\neq \ff(y)$. It also implies that $k(\ff)>0$. Then according to Lemma~\ref{lem_k_neq_zero}, there exists a point $x$ such that $\pdv{\mathring{J}_D(x)}{\ff(x)}\neq 0$. According to (a), there exists $y$ with $y\neq x$ satisfying that $|\ff(y)-\ff(x)| =k(\ff) \cdot \lVert y - x \rVert$. 


(d): In Nash equilibrium state, it holds that, for any $x \in \CS_r \cup \CS_g$, $\pdv{J}{k(f)} = \pdv{J_D^*}{k(f)} + 2\lambda\cdot k(f) = 0$ and $\pdv{\mathring{J}_D(x)}{f(x)}\pdv{f(x)}{x}=0$. We claim that in the Nash equilibrium state, the Lipschitz constant $k(f)$ must be 0.
If $k(f) \neq 0$, 
according to Lemma \ref{lem_k_neq_zero}, there must exist a point $\hat{x}$ such that $\pdv{{\mathring{J_D}(\hat{x}})}{f(\hat{x})}\neq0$. And according to (a), it must hold that $\exists\hat{y}$ fitting $|f(\hat{y})-f(\hat{x})| =k(f)\cdot \lVert \hat{x} - \hat{y} \rVert$. According to Theorem \ref{theorem_gradient}, we have $\big\Vert\pdv{f(\hat{x})}{\hat{x}}\big\Vert = k(f) \neq 0$. This is contradictory to that $\pdv{\mathring{J}_D(\hat{x})}{f(\hat{x})}\pdv{f(\hat{x})}{\hat{x}}=0$.
Thus $k(f)=0$. That is, $\forall x \in \CS_r \cup \CS_g$, $\pdv{f(x)}{x}=0$, which means $\,\forall x, y, f(x)=f(y)$. According to Lemma \ref{lemma_pair_f}, $\,\forall x, y, f(x)=f(y)$ implies $\CP_r=\CP_g$. Thus $\CP_r=\CP_g$ is the only Nash equilibrium in our system. 
\end{proof}

\vspace{-5pt}
\begin{remark}
For the Wasserstein distance, $\nabla_{\!\ff(x)} \mathring{J_D}(x)=0$ if and only if $\CP_r(x)=\CP_g(x)$.
For the Wasserstein distance, penalizing the Lipschitz constant also benefits: at the convergence state, it will hold $\pdv{\ff(x)}{x}=0$ for all $x$. 
\end{remark}

\subsection{Proof of Theorem \ref{theorem_chain}}

\begin{lemma}\label{lemma_chain}
Let $k$ be the Lipschitz constant of $f$. If $f(a)-f(b) = k \Vert a-b \Vert$ and $f(b)-f(c) = k \Vert b-c \Vert$, then $f(a)-f(c) = k\Vert a - c\Vert$ and $(a, f(a)), (b, f(b)), (c, f(c))$ lies in the same line.
\end{lemma}

\begin{proof}
$f(a)-f(c)=f(a)-f(b)+f(b)-f(c)=k \Vert a-b \Vert + k \Vert b-c \Vert \geq  k \Vert a-c \Vert$. Because  the Lipschitz constant of $f$ is $k$, we have $f(a)-f(c)\leq  k \Vert a-c \Vert$. Thus $f(a)-f(c)=k \Vert a-c \Vert$. Because the triangle equality holds, we have $a, b, c$ is in the same line. Furthermore, because $f(a)-f(b) = k \Vert a-b \Vert$, $f(b)-f(c) = k \Vert b-c \Vert$ and $f(a)-f(c) = k\Vert a - c\Vert$, we have $(a, f(a)), (b, f(b)), (c, f(c))$ lies in the same line. 
\end{proof}

\begin{lemma}\label{lemma_gl}
For any $x$ with $\pdv{{\mathring{J_D}(x})}{\ff(x)} > 0$, there exists a $y$ with $\pdv{{\mathring{J_D}(x})}{\ff(x)} < 0$ such that $\ff(y)-\ff(x) = k(\ff) \Vert y- x \Vert$. 

\hspace{50pt} For any $y$ with $\pdv{{\mathring{J_D}(y})}{\ff(y)} < 0$, there exists a $x$ with $\pdv{{\mathring{J_D}(x})}{\ff(x)} > 0$ such that $\ff(y)-\ff(x) = k(\ff) \Vert y- x \Vert$. 
\end{lemma}

\begin{proof}
Consider $x$ with $\pdv{{\mathring{J_D}(x})}{\ff(x)} > 0$. According to Theorem~\ref{theorem_main}, there exists $y$ such that $|\ff(y)-\ff(x)| = k(\ff) \Vert y- x \Vert$. Assume that for every $y$ that holds $|\ff(y)-\ff(x)| = k(\ff) \Vert y- x \Vert$, it has $\pdv{{\mathring{J_D}(y})}{\ff(y)} \geq 0$. Consider the set $S(x)=\{y \mid \ff(y)-\ff(x) = k(\ff) \Vert y- x \Vert \}$. Note that, according to Lemma~\ref{lemma_chain}, any $z$ that holds $\ff(z)-\ff(y) = k(\ff) \Vert z- y \Vert$ for any $y\in S(x)$ will also be in $S(x)$. Similar as the proof of (a) in Theorem~\ref{theorem_main}, we can decrease the value of $\ff(y)$ for all $y \in S(x)$ to construct a better $f$. By contradiction, we have that there must exist a $y$ with $\pdv{{\mathring{J_D}(x})}{\ff(x)} < 0$ such that $|\ff(y)-\ff(x)| = k(\ff) \Vert y- x \Vert$. Given the fact $\pdv{{\mathring{J_D}(x})}{\ff(x)} > 0$ and $\pdv{{\mathring{J_D}(x})}{\ff(x)} < 0$, we can conclude that $\ff(y)>\ff(x)$ and $\ff(y)-\ff(x) = k(\ff) \Vert y- x \Vert$. Otherwise, if $\ff(x)-\ff(y) = k(\ff) \Vert y- x \Vert$, then we can construct a better $f$ by decreasing $\ff(x)$ and increasing $\ff(y)$ which does not break the $k$-Lipschitz constraint. The other case can be proved similarly. 
\end{proof}

\begin{lemma}\label{lemma_xy}
For any $x$, if $\pdv{{\mathring{J_D}(x})}{f(x)} > 0$, then $\CP_g(x)>0$. For any $y$, if $\pdv{{\mathring{J_D}(y})}{f(y)} < 0$, then $\CP_r(y)>0$. 
\end{lemma}

\begin{proof}
$\pdv{{\mathring{J_D}(x})}{f(x)}=\CP_g(x)\frac{\partial \phi(f(x))}{\partial f(x)} + \CP_r(x)\frac{\partial \varphi(f(x))}{\partial f(x)}$. And we know $\phi'(x)>0$ and $\varphi'(x)<0$. Naturally, $\pdv{{\mathring{J_D}(x})}{f(x)}>0$ implies $\CP_g(x)>0$. Similarly, $\pdv{{\mathring{J_D}(y})}{f(y)}<0$ implies $\CP_r(y)>0$. 
\end{proof}

\begin{proof}[\textbf{Proof of Theorem \ref{theorem_chain}}] $ $ \newline \newline
For any $x\in \CS_g$, if $\pdv{{\mathring{J_D}(x})}{\ff(x)} > 0$, according to Lemma~\ref{lemma_gl}, there exists a $y$ with $\pdv{{\mathring{J_D}(x})}{\ff(x)} < 0$ such that $\ff(y)-\ff(x) = k(\ff) \Vert y- x \Vert$. According to Lemma~\ref{lemma_xy}, we have $\CP_r(y)>0$. That is, there is a $y \in \CS_r$ such that $\ff(y)-\ff(x) = k(\ff) \Vert y- x \Vert$. We can prove the other case symmetrically. 
\end{proof}

\begin{remark}
$\pdv{{\mathring{J_D}(x})}{\ff(x)} < 0$ for some $x\in \CS_g$ means $x$ is at the overlapping region of $\CS_r$ and $\CS_g$. It can be regarded as a $y \in \CS_r$, and one can apply the other rule which guarantees that there exists a $x' \in \CS_g$ that bounds this point. 
\end{remark}

\subsection{Proof of Theorem \ref{theorem_gradient}}
\label{lip_direction}
In this section, we will prove Theorem \ref{theorem_gradient}, i.e., Lipschitz continuity with $l_2$-norm (Euclidean Distance) can guarantee that the gradient is directly pointing towards some sample. 

Let $(x, y)$ be such that $y\neq x $, and we define $x_t = x+ t\cdot (y-x)$ with $t \in [0,1]$. 

\begin{lemma}
If $f(x)$ is $k$-Lipschitz with respect to $\Vert . \Vert_p$ and $f(y)-f(x) = k\Vert y - x \Vert_p$, then $f(x_t) = f(x)+t\cdot k\Vert y - x \Vert_p$
\end{lemma} 

\begin{proof}
As we know $f(x)$ is $k$-Lipschitz, with the property of norms, we have
\begin{align}\label{eq:linear_interopation}
f(y)-f(x) &= f(y)-f(x_t)+ f(x_t)-f(x) \nonumber \\
& \leq f(y)-f(x_t)+k\Vert x_t-x\Vert_p  = f(y)-f(x_t)+t\cdot k\Vert y - x\Vert_p \nonumber \\
& \leq k\Vert y-x_t\Vert_p+t\cdot k\Vert y - x\Vert_p  = k \cdot (1-t)\Vert y - x\Vert_p+t \cdot k\Vert y - x\Vert_p \nonumber \\
& = k \Vert y - x\Vert_p. 
\end{align}
$f(y)-f(x) = k\Vert y - x \Vert_p$ implies all the inequalities is equalities. Therefore, $f(x_t) = f(x)+t\cdot k\Vert y - x \Vert_p$. \qedhere
\end{proof}

\begin{lemma}
Let $v$ be the unit vector $\frac{y-x}{\Vert y -x \Vert_2}$. If $f(x_t) = f(x)+t\cdot k\Vert y - x \Vert_2$, then $\tpdv{f(x_t)}{v}$ equals to $k$. 
\end{lemma}

\begin{proof}
\begin{equation}
\begin{aligned}
\vspace{-20pt}
\tpdv{f(x_t)}{v} 
&= \lim\limits_{h\rightarrow0} \frac{f(x_t+hv)-f(x_t)}{h} =\lim\limits_{h\rightarrow0} \frac{f(x_t+h\frac{y-x}{\Vert y-x \Vert_2})-f(x_t)}{h} \\ 
& =\lim\limits_{h\rightarrow0}\frac{f(x_{t+\frac{h}{\Vert y-x \Vert_2}})-f(x_t)}{h}  =\lim\limits_{h\rightarrow0}\frac{\frac{h}{\Vert y-x \Vert_2}\cdot k\Vert y-x \Vert_2}{h}=k. \nonumber \qedhere
\end{aligned}
\end{equation}
\end{proof}

\begin{proof}[\textbf{Proof of Theorem \ref{theorem_gradient}}] $ $ 
Assume $p=2$. According to \cite{adler2018banach}, if $f(x)$ is $k$-Lipschitz with respect to $\Vert.\Vert_2$ and $f(x)$ is differentiable at $x_t$, then $\Vert \nabla f(x_t) \Vert_2 \leq k$.  Let $v$ be the unit vector $\frac{y-x}{\Vert y -x \Vert_2}$. We have 
\begin{align}
k^2 =k\tpdv{f(x_t)}{v} &=k\left<v,\nabla f(x_t)\right>= \left<kv, \nabla f(x_t) \right> \leq \Vert kv \Vert_2\Vert \nabla f(x_t) \Vert_2 = k^2.
\end{align}
Because the equality holds only when $\nabla f(x_t) = kv = k\frac{y-x}{\Vert y -x \Vert_2}$, we have that $\nabla f(x_t) = k\frac{y-x}{\Vert y -x \Vert_2}$.
\end{proof}

\subsection{Proof of the New Dual Form of Wasserstein Distance} \label{app_dual_form}

We here provide a proof for our new dual form of Wasserstein distance, i.e., Eq.~\eqref{eq_w_dual_form_1}. 

The Wasserstein distance is given as follows
\begin{equation}\label{eq_w_primal_app}
W_1(\CP_r,\CP_g) =  \inf_{\pi \in \Pi(\CP_r,\CP_g)} \, \E_{(x,y) \sim \pi} \, [d(x, y)],
\end{equation}
where $\Pi(\CP_r, \CP_g)$ denotes the set of all probability measures with marginals $\CP_r$ and $\CP_g$ on the first and second factors, respectively. The Kantorovich-Rubinstein (KR) dual \cite{oldandnew} is written as
\begin{equation}
\begin{aligned}
W_{KR}(\CP_r,\CP_g) &= {\sup}_{f} \,\, \E_{x \sim \CP_r} \, [f(x)] - \E_{x \sim \CP_g} \, [f(x)],  \, \\
&\emph{s.t.} \, f(x) - f(y) \leq d(x, y), \,\, \forall x, \forall y.
\end{aligned}
\label{eq_w_dual_form_app1}
\end{equation}
We will prove that Wasserstein distance in its dual form can also be written as
\begin{equation}
\begin{aligned}
W_{LL}(\CP_r,\CP_g) &= {\sup}_{f} \,\, \E_{x \sim \CP_r} \, [f(x)] - \E_{x \sim \CP_g} \, [f(x)],  \, \\
&\emph{s.t.} \, f(x) - f(y) \leq d(x, y), \,\, \forall x \in \CS_r, \forall y \in \CS_g,
\end{aligned}
\label{eq_w_dual_form_app2}
\end{equation}
which relaxes the constraint in the KR dual form of Wasserstein distance. 

\begin{theorem}
Given $W_{KR}(\CP_r,\CP_g)=W_1(\CP_r,\CP_g)$, we have $W_{KR}(\CP_r,\CP_g)=W_{LL}(\CP_r,\CP_g)=W_1(\CP_r,\CP_g)$. 
\end{theorem} 

\begin{proof} $ $\newline$ $\newline 
    (i) For any $f$ that satisfies ``$f(x) - f(y) \leq d(x, y), \,\, \forall x, \forall y$'', it must satisfy ``$f(x) - f(y) \leq d(x, y), \,\, \forall x \in \CS_r, \forall y \in \CS_g$''. 
    
    \hspace{13pt} Thus, $W_{KR}(\CP_r,\CP_g)\leq W_{LL}(\CP_r,\CP_g)$.  
    
    \vspace{3pt}
    (ii) Let $F_{LL}=\{f| \, f(x) - f(y) \leq d(x, y), \,\, \forall x \in \CS_r, \forall y \in \CS_g \}$.
    
    \hspace{13pt} Let $A=\{(x,y) \,|\, x \in \CS_r, y \in \CS_g \}$ and $I_A=\begin{cases}
    1, \quad (x,y) \in A; \\
    0, \quad otherwise
    \end{cases}$.
    
    \hspace{13pt} Let $A^c$ denote the complementary set of $A$ and define $I_{A^c}$ accordingly.
        
    \vspace{4pt}
    \hspace{13pt} $\forall \pi \in \Pi(\CP_r, \CP_g)$, we have the following:
    \vspace{2pt}
    \begin{align}
    W_{LL}(\CP_r,\CP_g) &= \nonumber {\sup}_{f\in F_{LL}} \,\, \E_{x \sim \CP_r} \, [f(x)] - \E_{x \sim \CP_g} \, [f(x)]  \\ \nonumber
    & = {\sup}_{f\in F_{LL}} \,\, \E_{(x,y) \sim \pi} [f(x)-f(y)]  \\ \nonumber
    & = {\sup}_{f\in F_{LL}} \,\, \E_{(x,y) \sim \pi} [(f(x)-f(y)) I_A] + \E_{(x,y) \sim \pi} [(f(x)-f(y)) I_{A^c}] \\\nonumber
    & = {\sup}_{f\in F_{LL}} \,\, \E_{(x,y) \sim \pi} [(f(x)-f(y)) I_A] \\\nonumber
    & \leq \E_{(x,y) \sim \pi} [\lVert y - x \rVert I_A]  \\\nonumber
    & \leq \E_{(x,y) \sim \pi} [d(x,y)].
    \end{align}
    
    \hspace{15pt}$W_{LL}(\CP_r,\CP_g)\leq \E_{(x,y) \sim \pi} [d(x,y)], \forall \pi \in \Pi(\CP_r, \CP_g)$ 
    
    \hspace{15pt}$\Rightarrow W_{LL}(\CP_r,\CP_g)\leq \inf_{\pi \in \Pi(\CP_r,\CP_g)} \, \E_{(x,y) \sim \pi} \, [d(x, y)] = W_1(\CP_r,\CP_g)$.
    
    \vspace{3pt}
    (iii) Combining (i) and (ii), we have $W_{KR}(\CP_r,\CP_g)\leq W_{LL}(\CP_r,\CP_g)\leq W_1(\CP_r,\CP_g)$.
    
    \hspace{15pt} Given $I(\CP_r,\CP_g)=W_1(\CP_r,\CP_g)$, we have $I(\CP_r,\CP_g)=W_{LL}(\CP_r,\CP_g)=W_1(\CP_r,\CP_g)$.        
\end{proof}

\section{The Practical Behaviors of Gradient Uninformativeness}
\label{hyper_para}

To study the practical behaviors of gradient uninformativeness, we conducted a set of experiments with various hyper-parameter settings. We use the Least-Squares GAN in this experiments as an representative of traditional GANs. The value surface and the gradient of generated samples under various situations are plotted as follows. 

\begin{figure}[!h]
	\begin{subfigure}{0.33\linewidth}
		\centering
		\includegraphics[width=0.99\columnwidth]{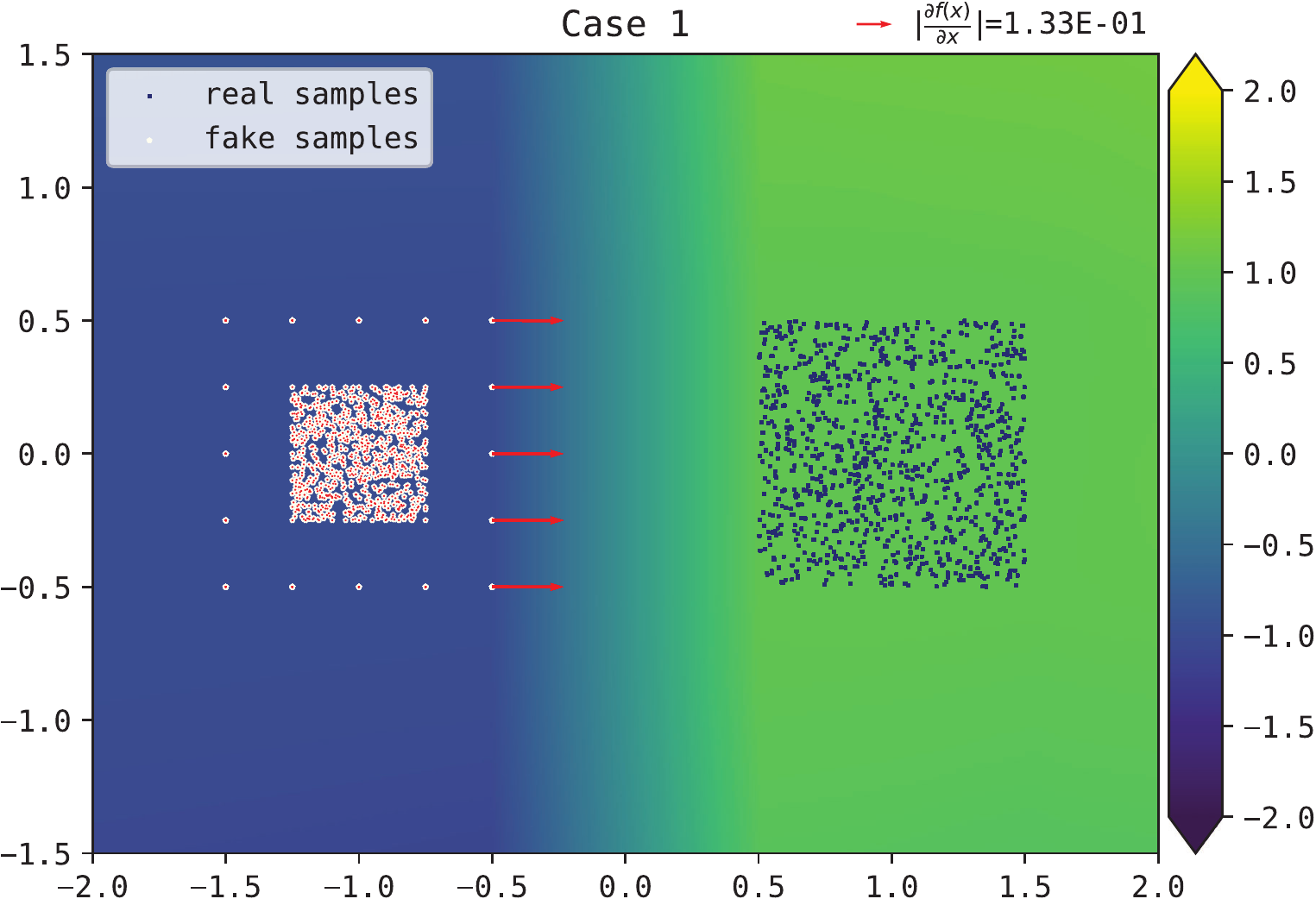}
		\label{fig_case1_lsgan_adam_1e-2_relu_1024*1_toy}
	\end{subfigure}
	\begin{subfigure}{0.33\linewidth}
		\centering
		\includegraphics[width=0.99\columnwidth]{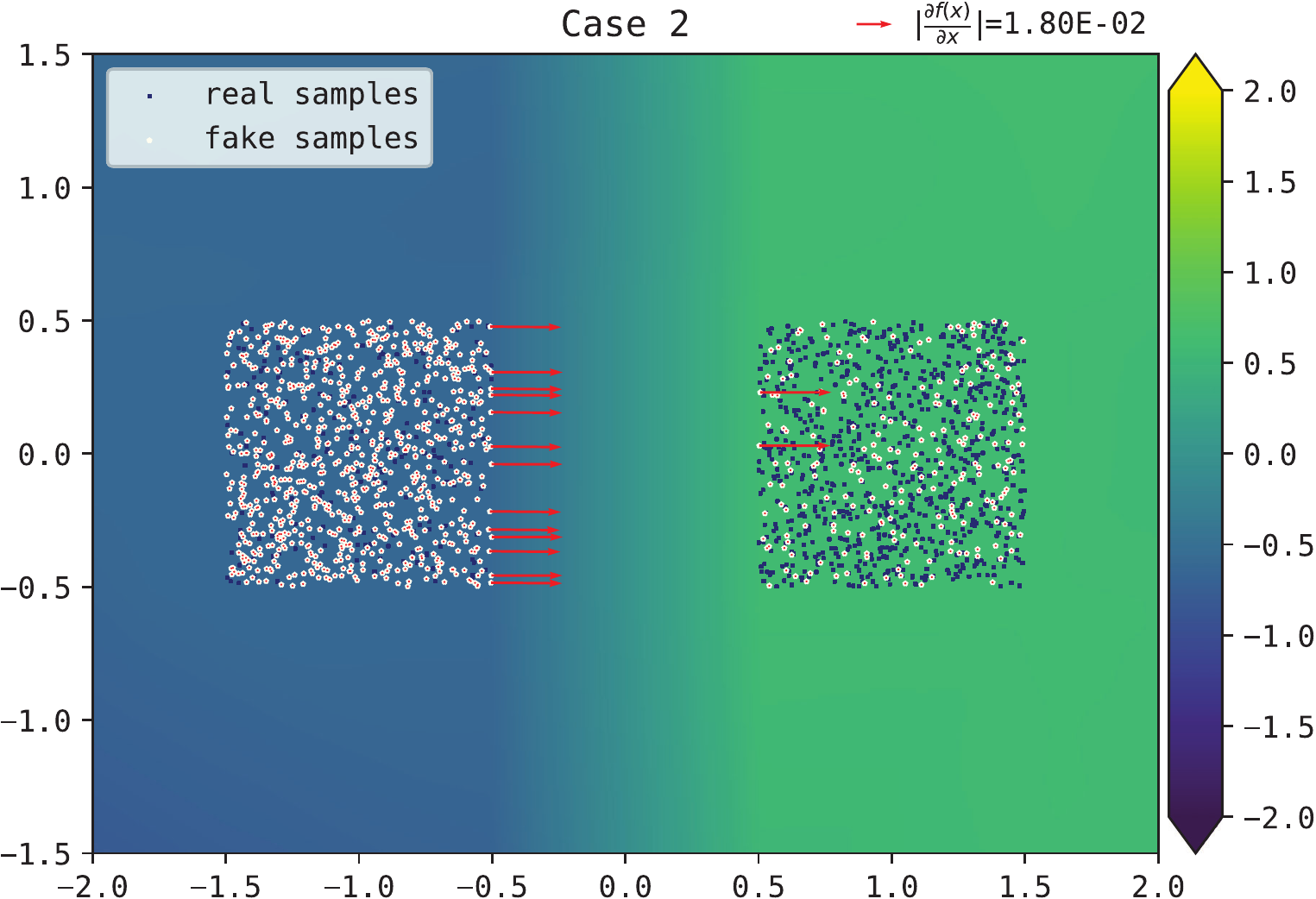}
		\label{fig_case2_lsgan_adam_1e-2_relu_1024*1_toy}
	\end{subfigure}
	\begin{subfigure}{0.33\linewidth}
		\centering
		\includegraphics[width=0.99\columnwidth]{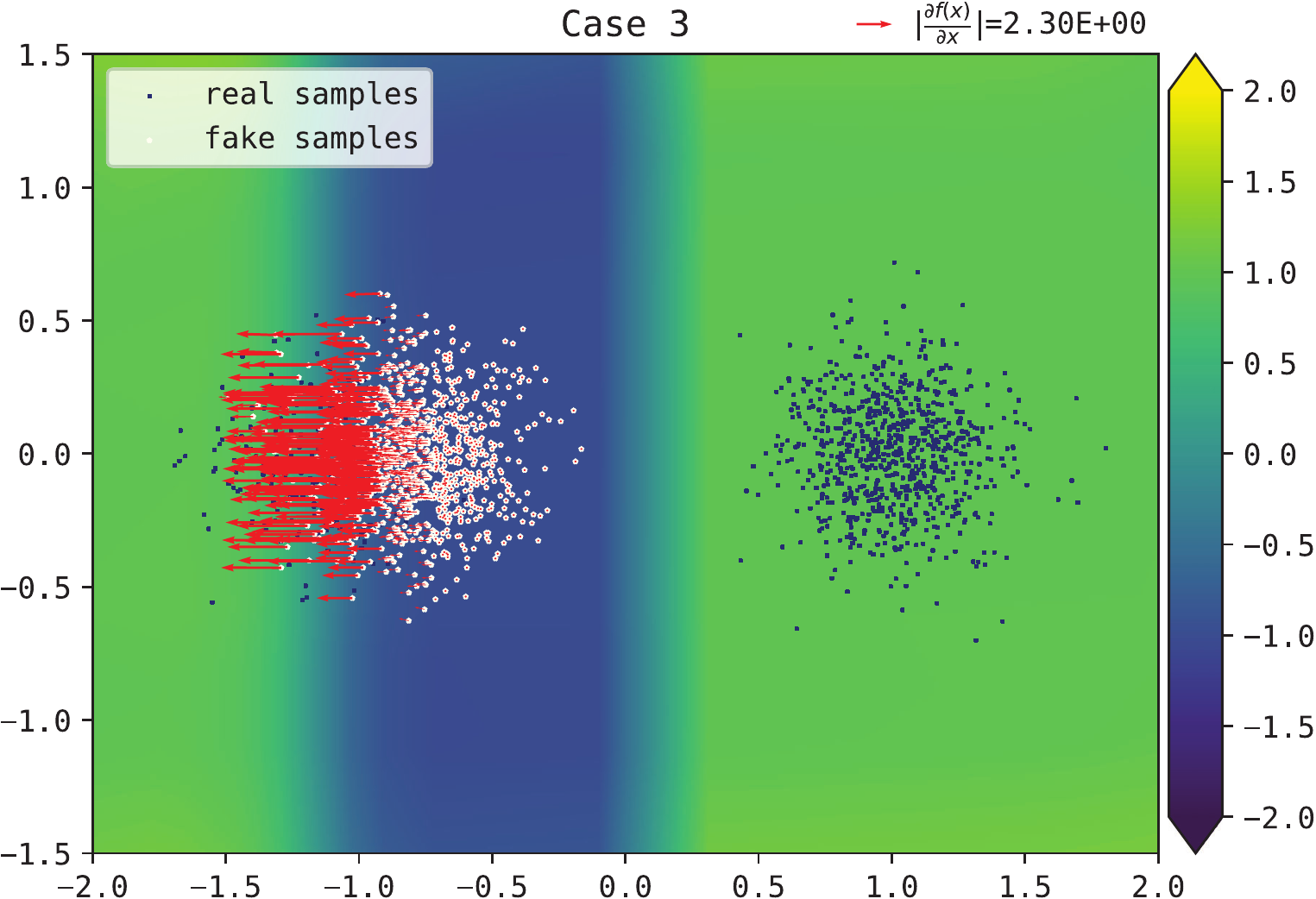}
		\label{fig_case3_lsgan_adam_1e-2_relu_1024*1_toy}
	\end{subfigure}
	\vspace{-10pt}
	\caption{ADAM with lr=1e-2, beta1=0.0, beta2=0.9. MLP with RELU activations, \#hidden units=1024, \#layers=1.}
	\label{hyper_test_1}
\end{figure}

\begin{figure}[!h]
	\begin{subfigure}{0.33\linewidth}
		\centering
		\includegraphics[width=0.99\columnwidth]{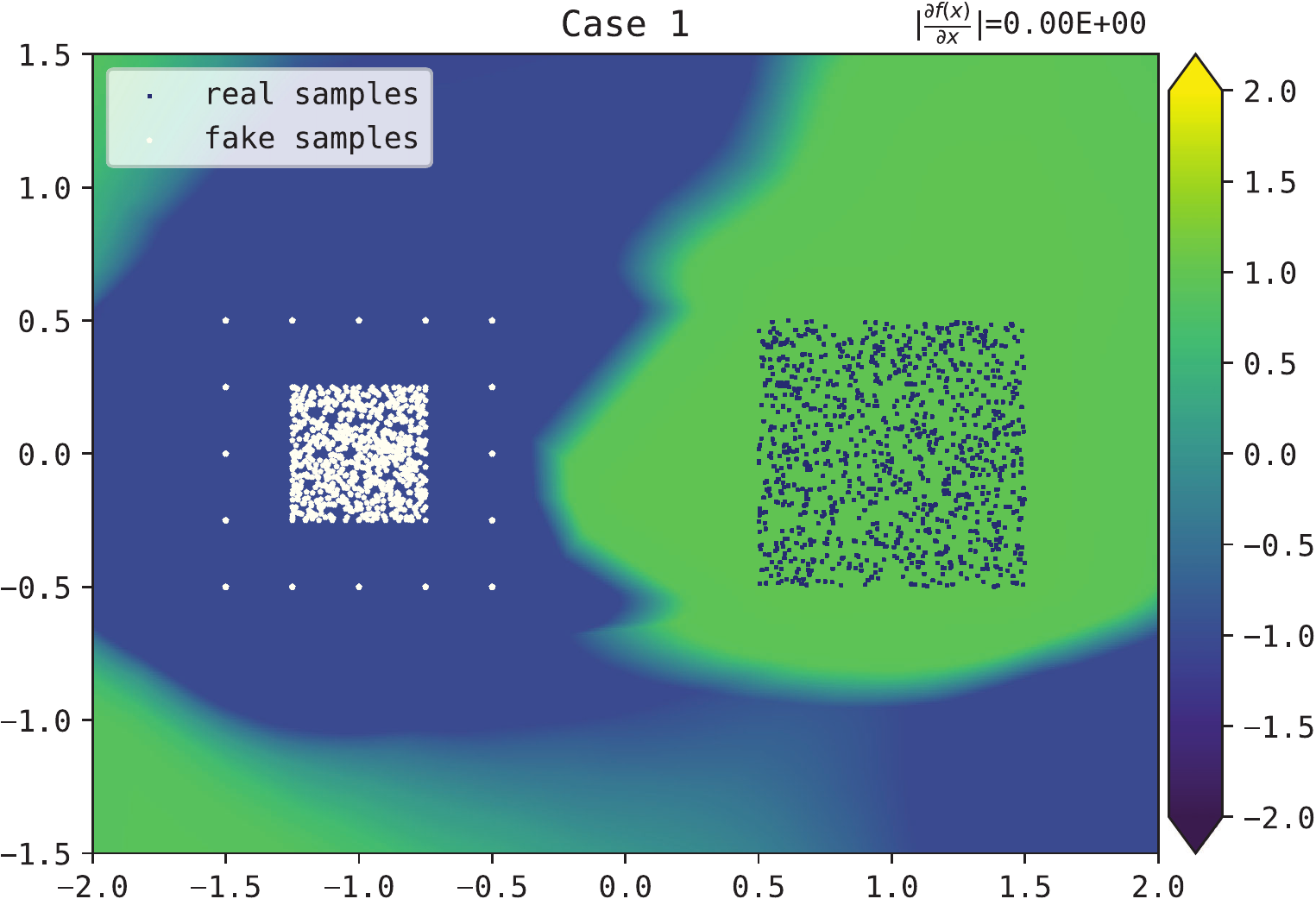}
		\label{fig_case1_lsgan_adam_1e-2_relu_1024*4_toy}
	\end{subfigure}
	\begin{subfigure}{0.33\linewidth}
		\centering
		\includegraphics[width=0.99\columnwidth]{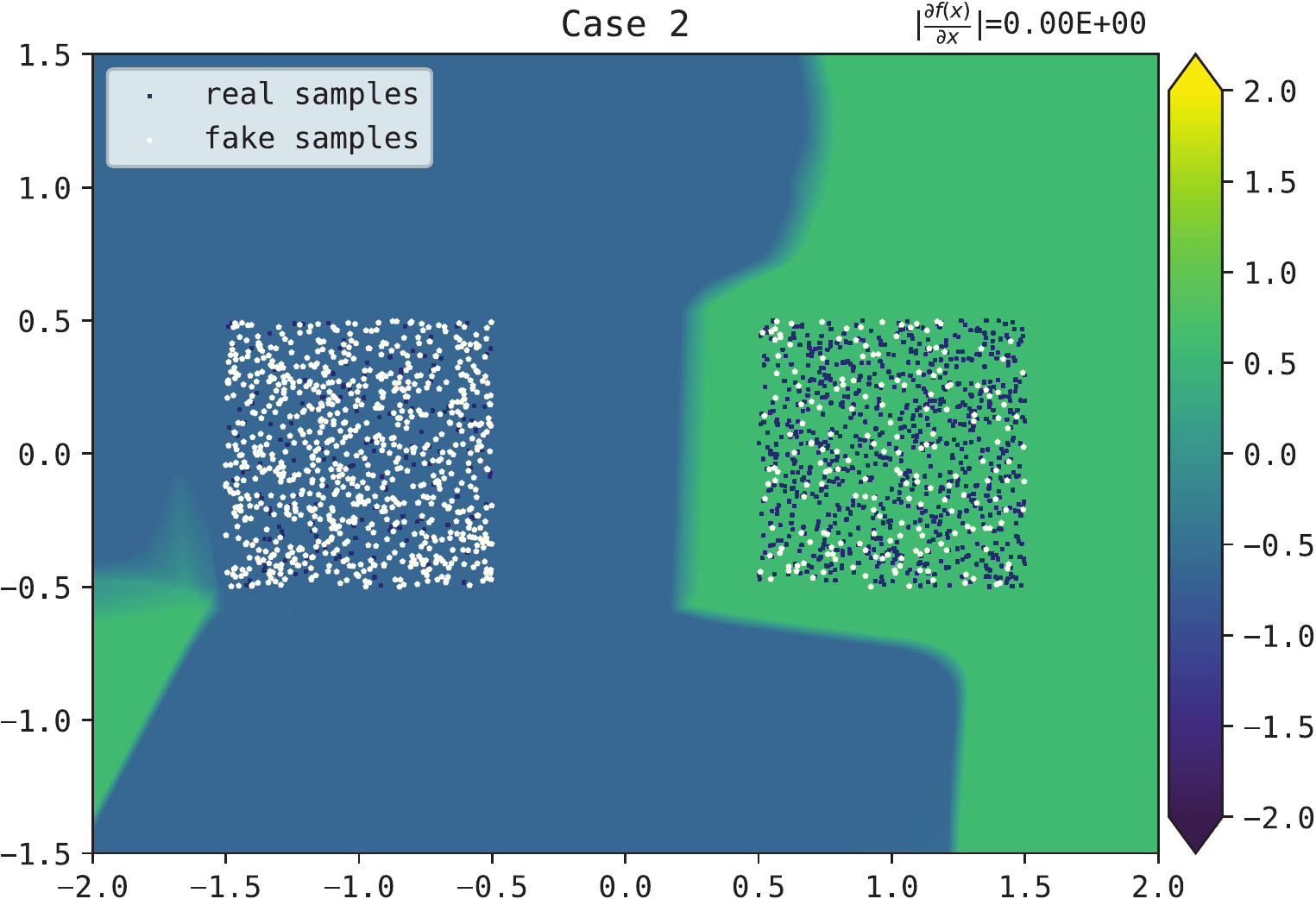}
		\label{fig_case2_lsgan_adam_1e-2_relu_1024*4_toy}
	\end{subfigure}
	\begin{subfigure}{0.33\linewidth}
		\centering
		\includegraphics[width=0.99\columnwidth]{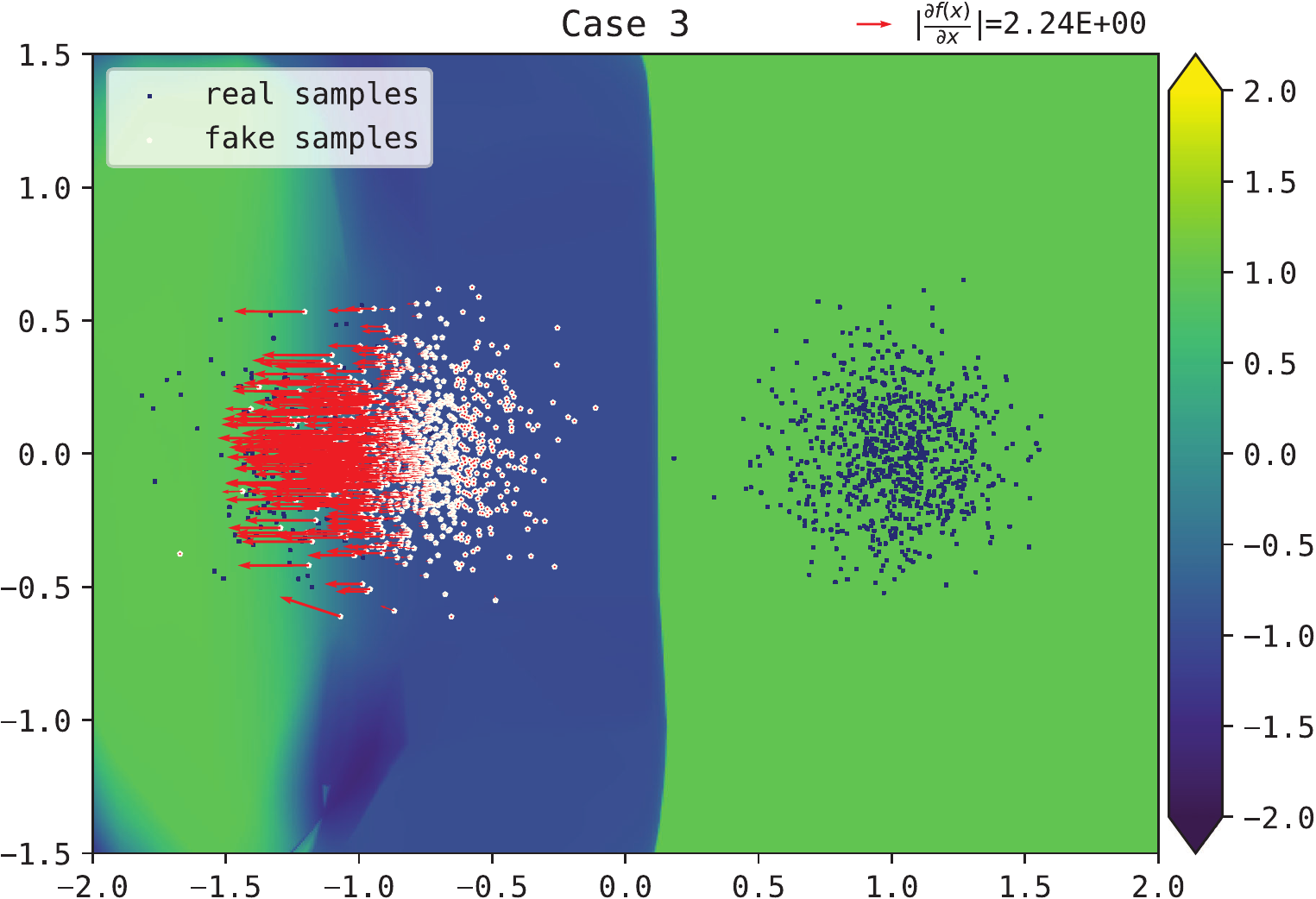}
		\label{fig_case3_lsgan_adam_1e-2_relu_1024*4_toy}
	\end{subfigure}
	\vspace{-10pt}
	\caption{ADAM with lr=1e-2, beta1=0.0, beta2=0.9. MLP with RELU activations, \#hidden units=1024, \#layers=4.}
\end{figure}

\begin{figure}[!h]
	\begin{subfigure}{0.33\linewidth}
		\centering
		\includegraphics[width=0.99\columnwidth]{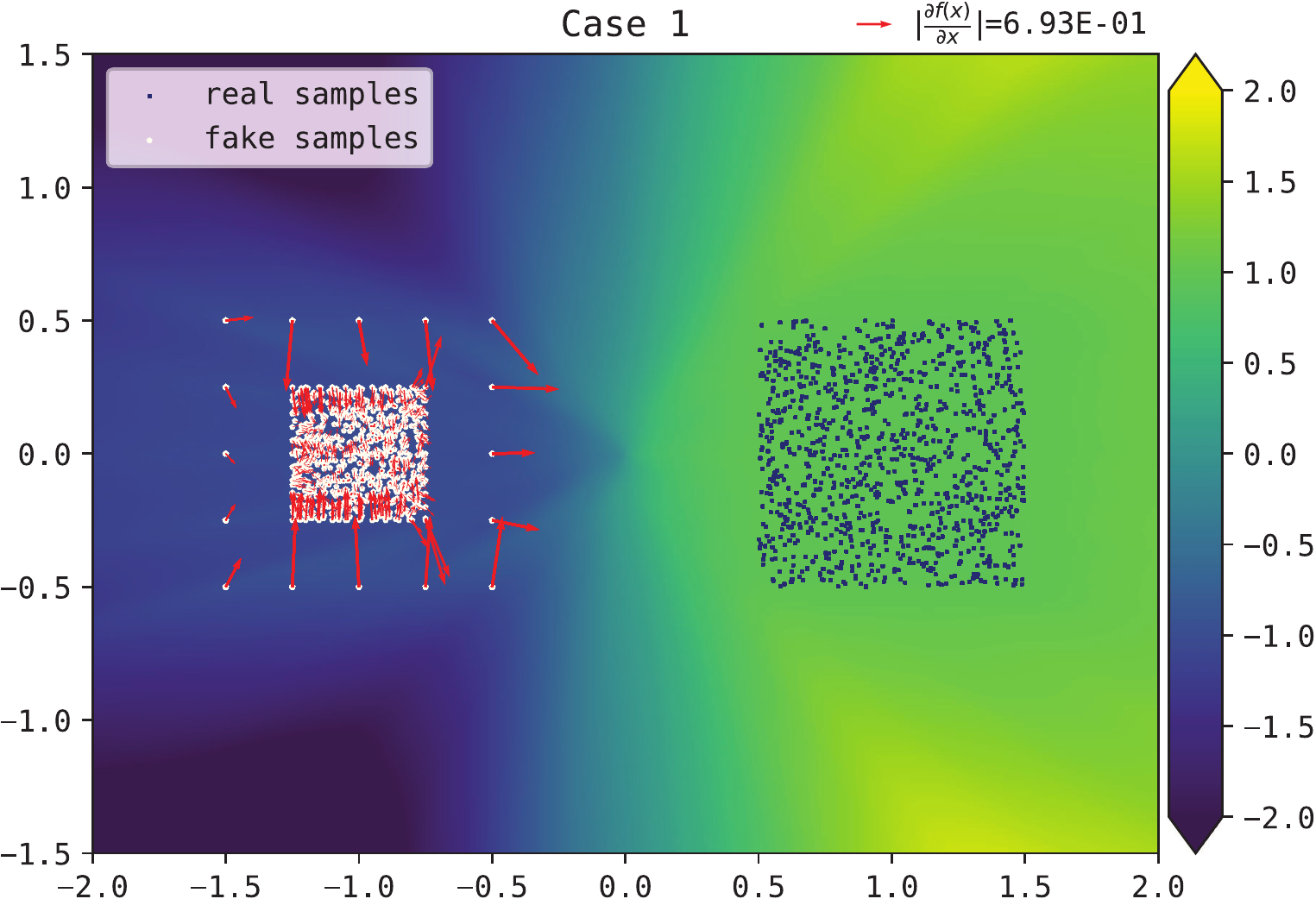}
		\label{fig_case1_lsgan_adam_1e-5_relu_1024*4_toy}
	\end{subfigure}
	\begin{subfigure}{0.33\linewidth}
		\centering
		\includegraphics[width=0.99\columnwidth]{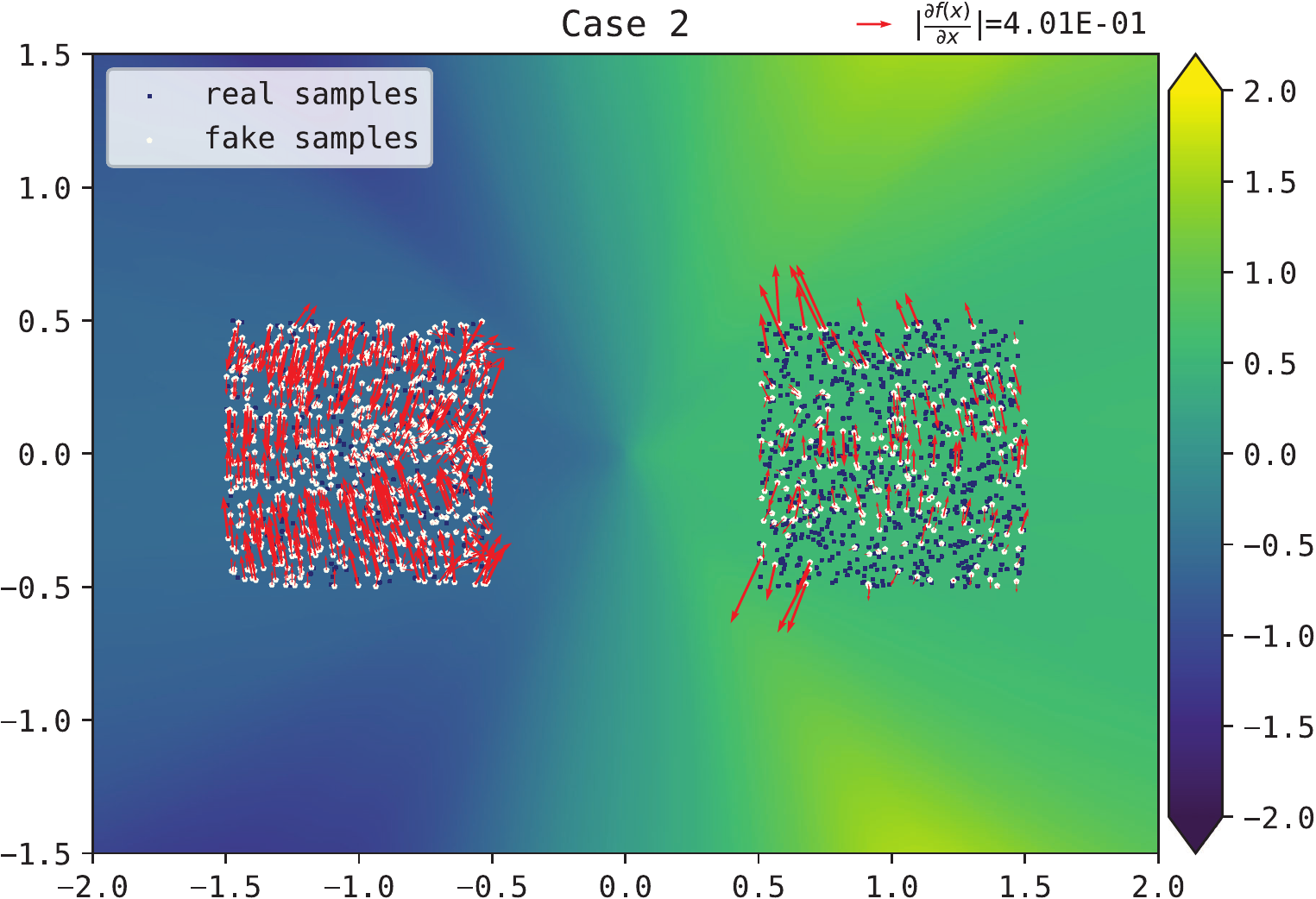}
		\label{fig_case2_lsgan_adam_1e-5_relu_1024*4_toy}
	\end{subfigure}
	\begin{subfigure}{0.33\linewidth}
		\centering
		\includegraphics[width=0.99\columnwidth]{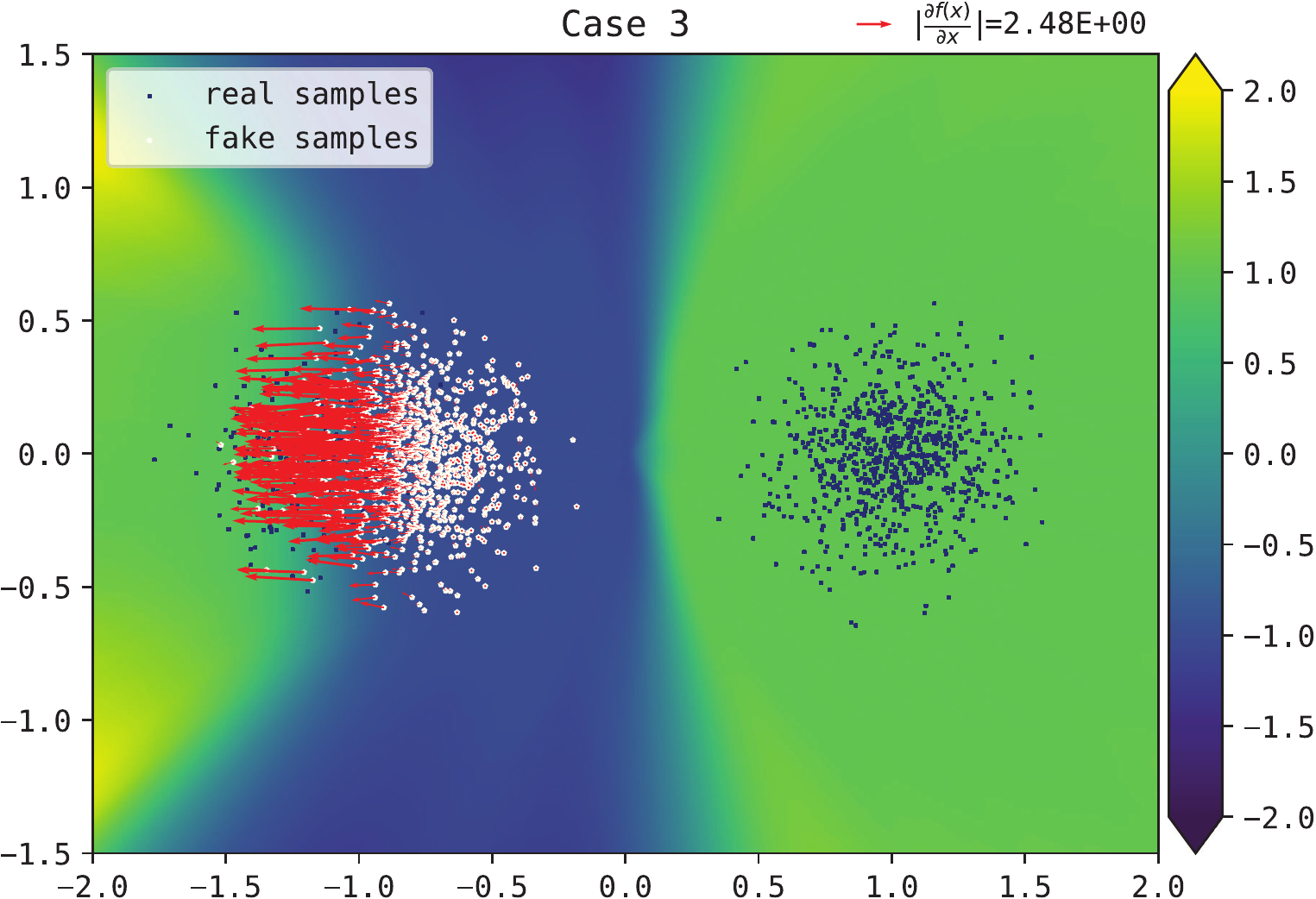}
		\label{fig_case3_lsgan_adam_1e-5_relu_1024*4_toy}
	\end{subfigure}
	\vspace{-10pt}
	\caption{ADAM with lr=1e-5, beta1=0.0, beta2=0.9. MLP with RELU activations, \#hidden units=1024, \#layers=4.}
\end{figure}

\begin{figure}[!h]
	\begin{subfigure}{0.33\linewidth}
		\centering
		\includegraphics[width=0.99\columnwidth]{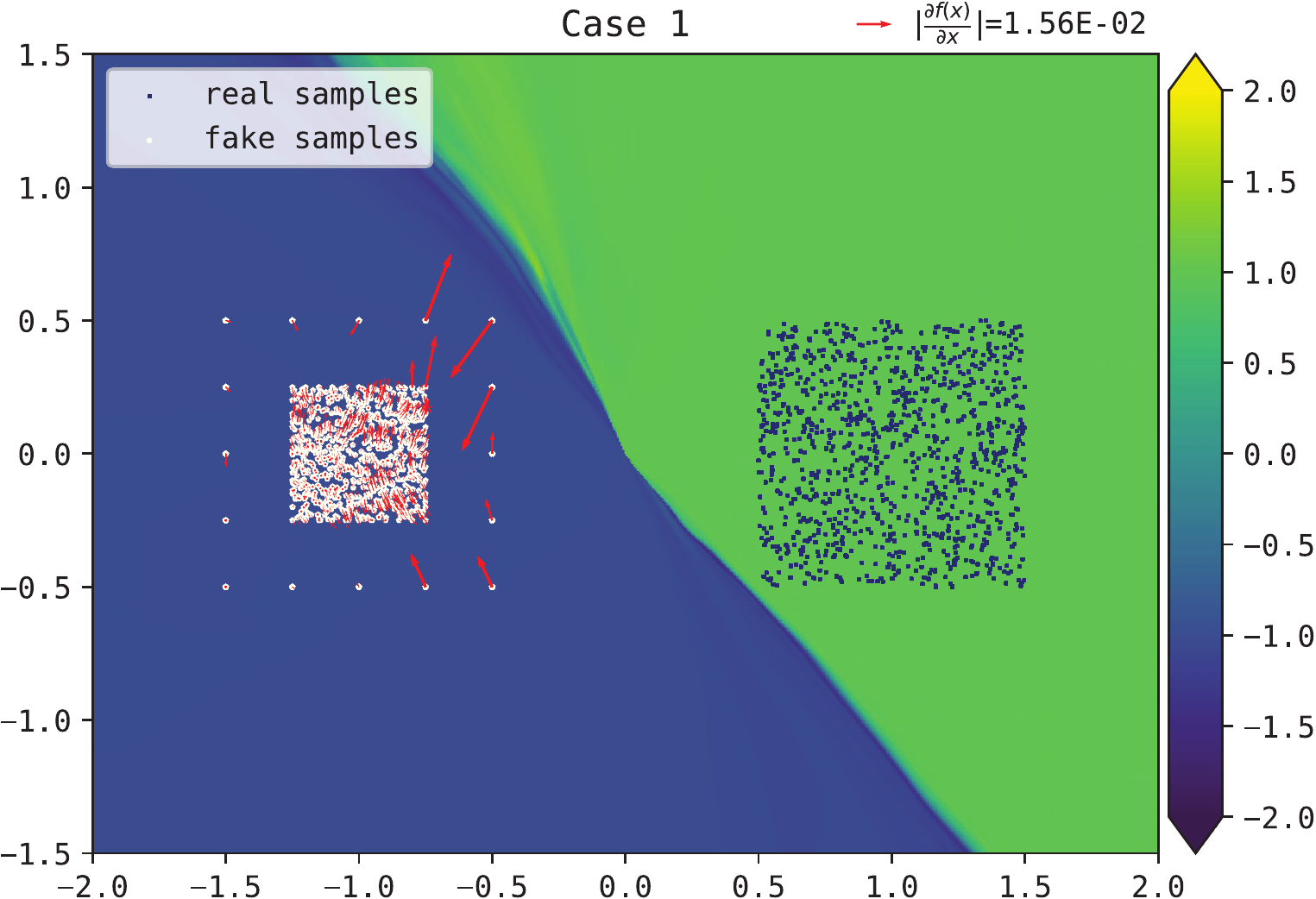}
		\label{fig_case1_lsgan_sgd_1e-3_relu_128*64_toy}
	\end{subfigure}
	\begin{subfigure}{0.33\linewidth}
		\centering
		\includegraphics[width=0.99\columnwidth]{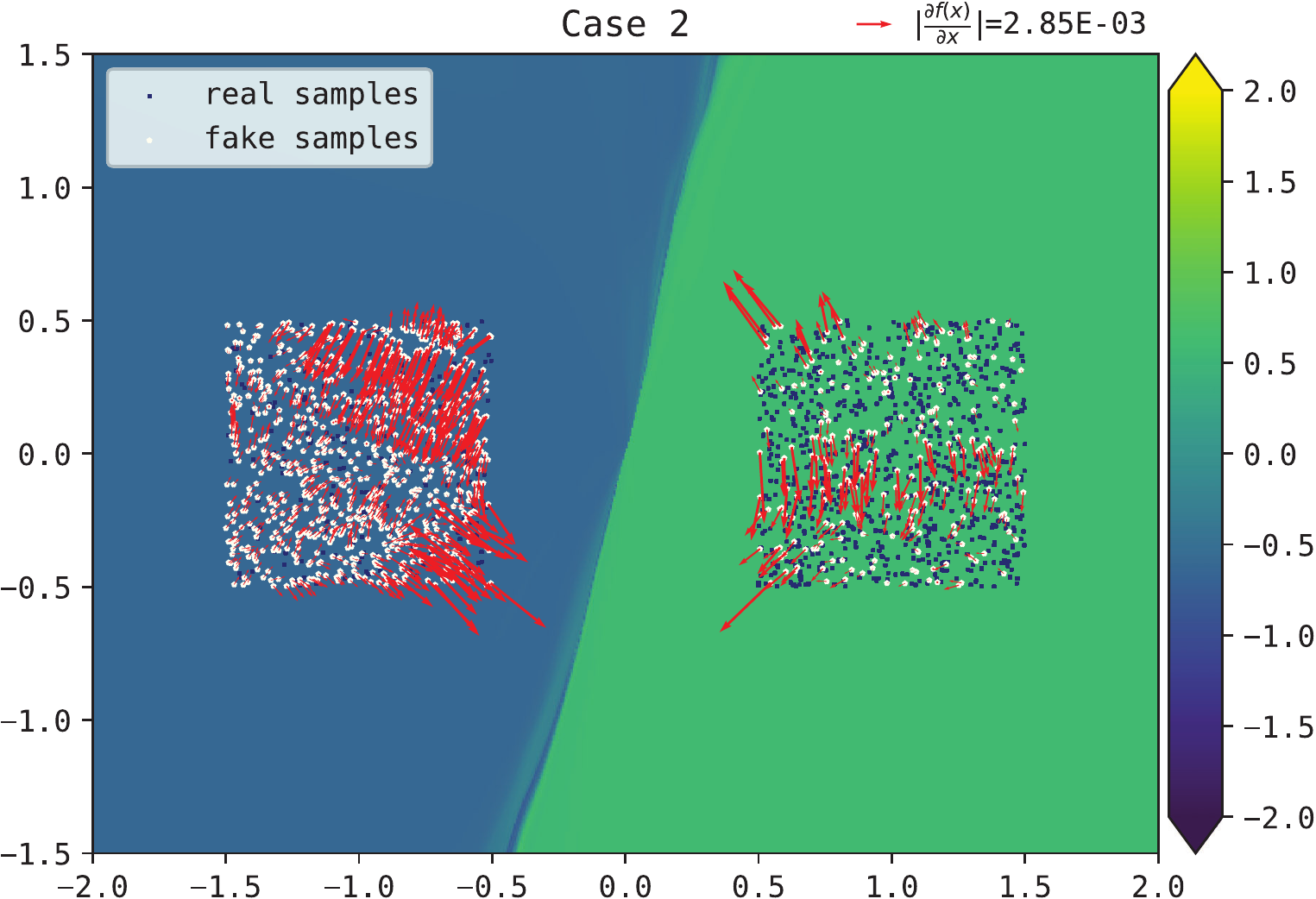}
		\label{fig_case2_lsgan_sgd_1e-3_relu_128*64_toy}
	\end{subfigure}
	\begin{subfigure}{0.33\linewidth}
		\centering
		\includegraphics[width=0.99\columnwidth]{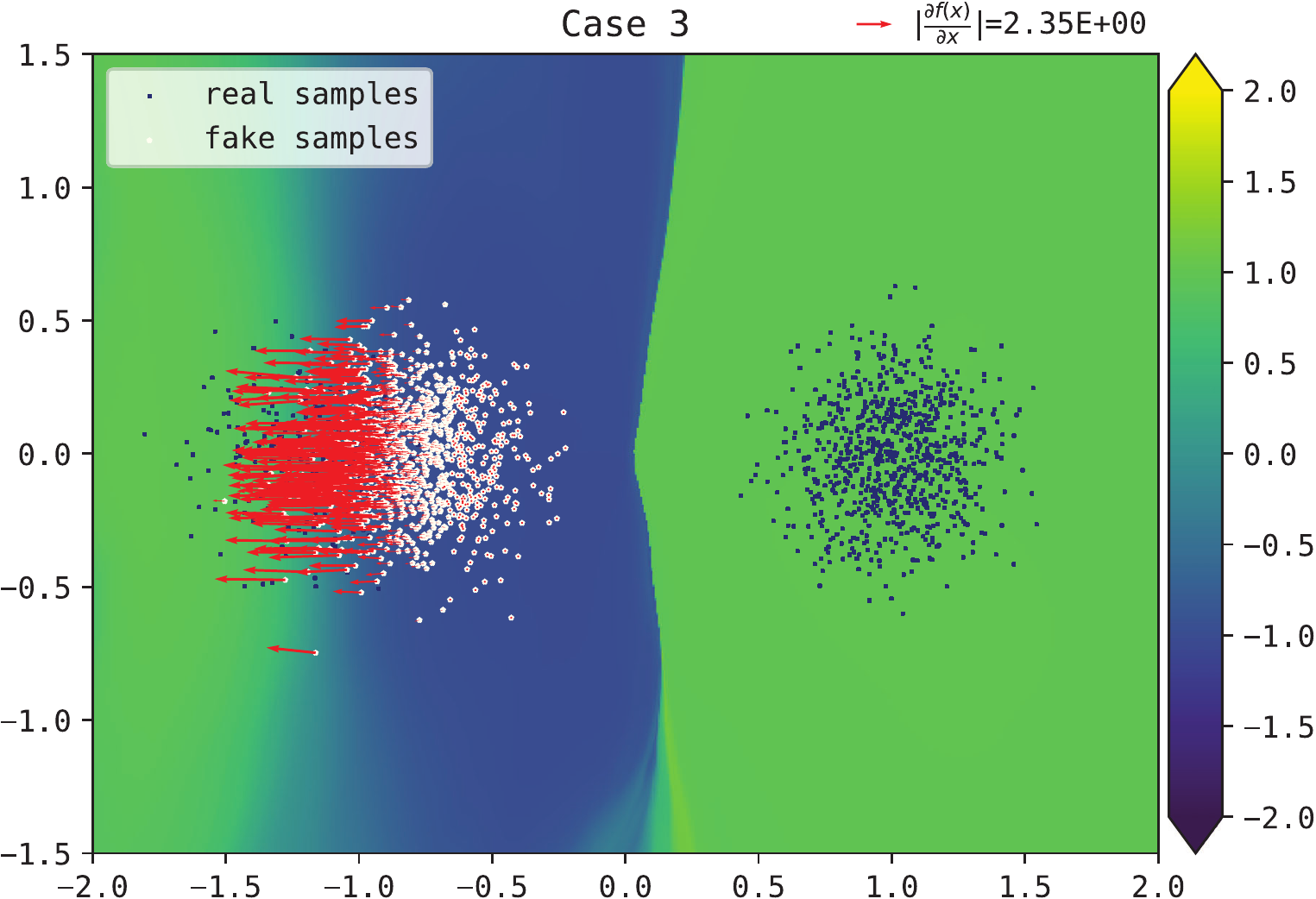}
		\label{fig_case3_lsgan_sgd_1e-3_relu_128*64_toy}
	\end{subfigure}
	\vspace{-10pt}
	\caption{SGD with lr=1e-3. MLP with SELU activations, \#hidden units=128, \#layers=64.}
\end{figure}

\begin{figure}[!h]
	\begin{subfigure}{0.33\linewidth}
		\centering
		\includegraphics[width=0.99\columnwidth]{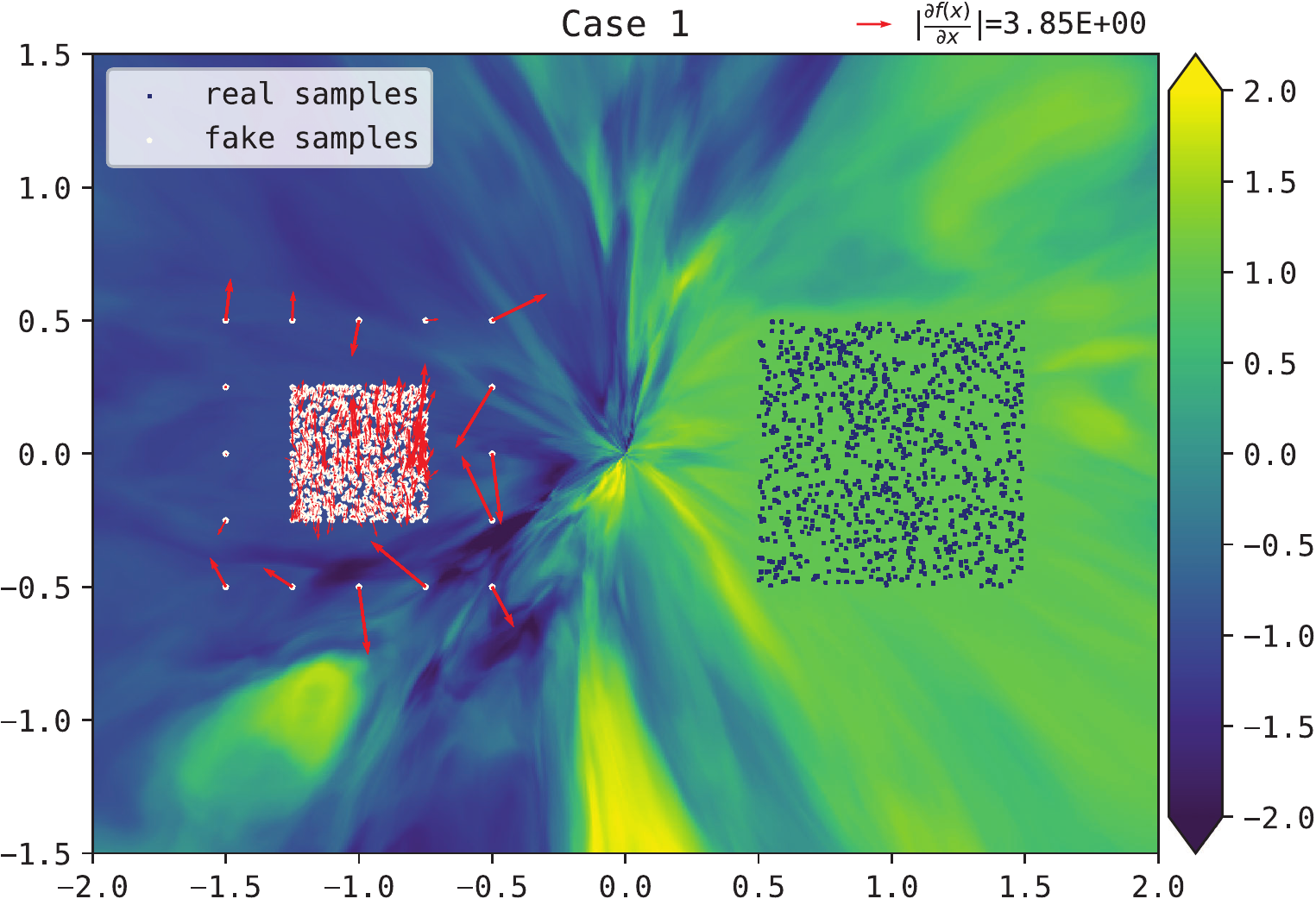}
		\label{fig_case1_lsgan_sgd_1e-4_relu_128*64_toy}
	\end{subfigure}
	\begin{subfigure}{0.33\linewidth}
		\centering
		\includegraphics[width=0.99\columnwidth]{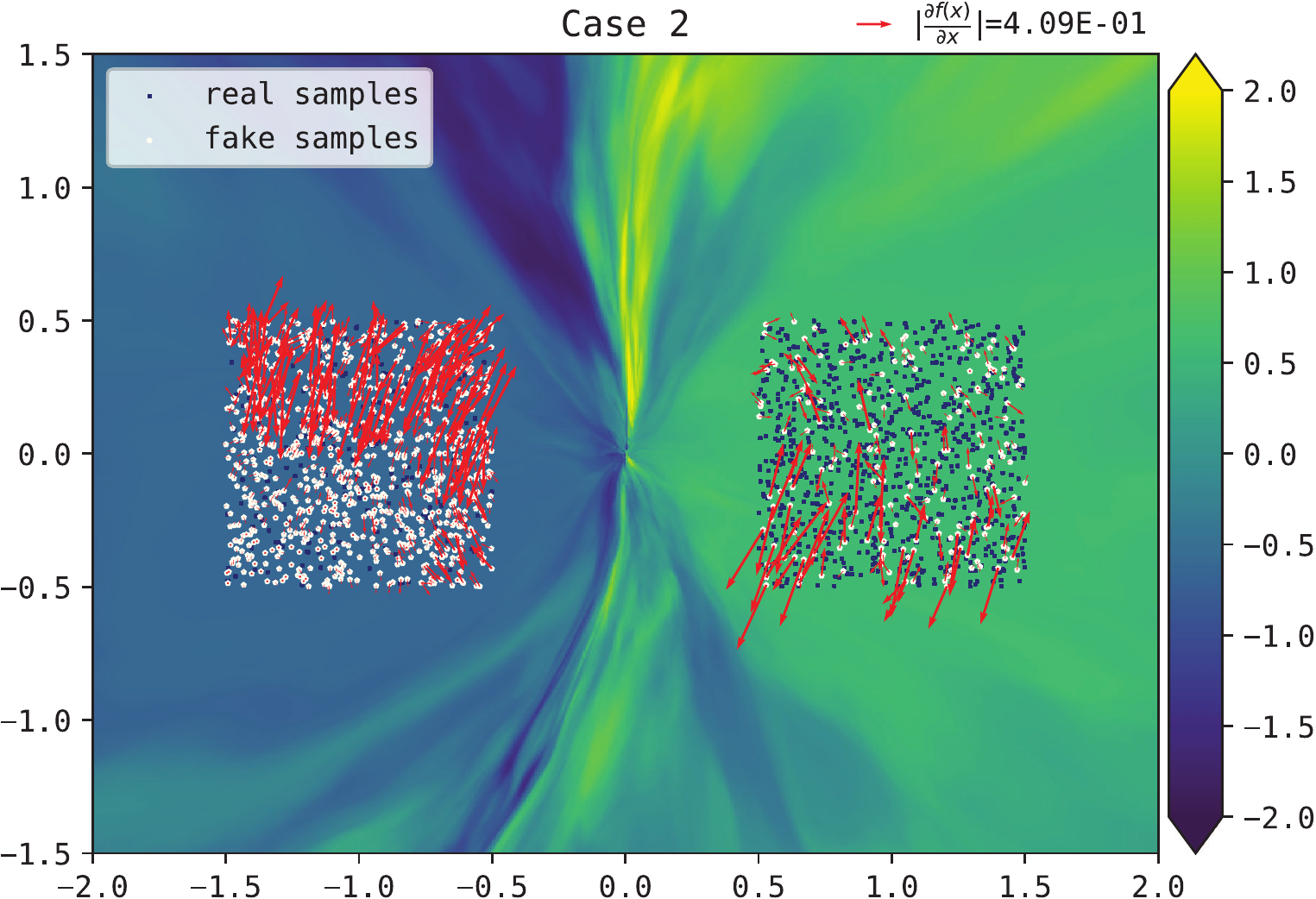}
		\label{fig_case2_lsgan_sgd_1e-4_relu_128*64_toy}
	\end{subfigure}
	\begin{subfigure}{0.33\linewidth}
		\vspace{-0pt}
		\centering
		\includegraphics[width=0.99\columnwidth]{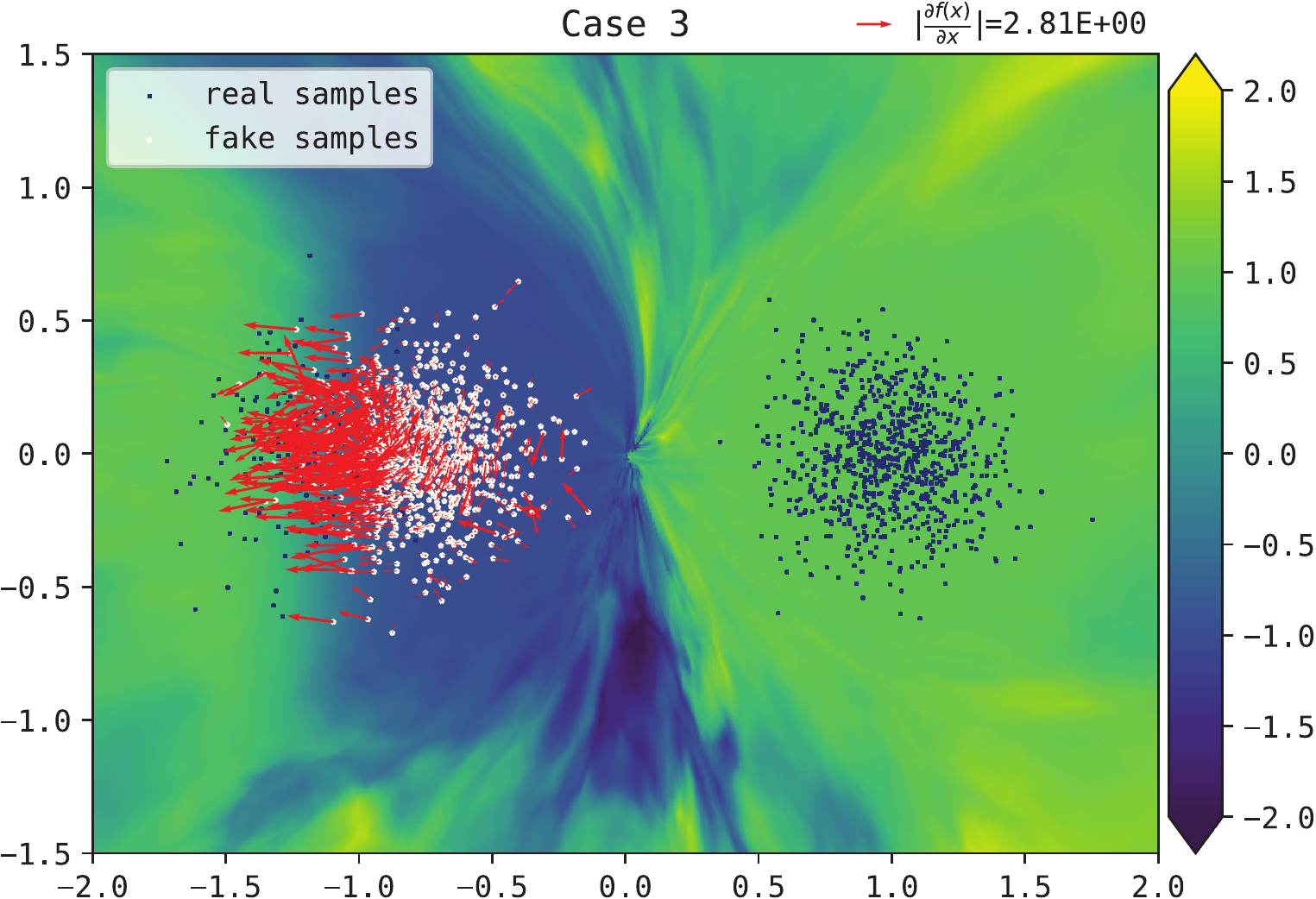}
		\label{fig_case3_lsgan_sgd_1e-4_relu_128*64_toy}
	\end{subfigure}
	\vspace{-10pt}
	\caption{SGD with lr=1e-4. MLP with SELU activations, \#hidden units=128, \#layers=64.}
	\label{hyper_test_2}
	\vspace{-10pt}
\end{figure}

These experiments shown that the practical $f$ highly depend on the hyper-parameter setting. Given limited capacity, the neural network try to learn the best $f$. When the neural network is capable of learning approximately the optimal $\ff$, how the actual $f$ approaches $\ff$ and how the points whose gradients are theoretically undefined behave highly depends the optimization details and the characteristics of the network. 

\section{On the Implementation of Lipschitz continuity for GANs} \label{app_maxgp}

Typical techniques for enforcing $k$-Lipschitz includes: spectral normalization \citep{sngan}, gradient penalty \citep{wgangp}, and Lipschitz penalty \citep{wganlp}. Before moving into the detailed discussion of these methods, we would like to provide several important notes in the first place.

Firstly, enforcing $k$-Lipschitz in the blending-region of $\CP_r$ and $\CP_g$ is actually sufficient.

Define $B(\CS_r,\CS_g) = \{\hat{x}=x \cdot t + y \cdot (1-t) \mid x\in\CS_r$ and $y\in \CS_g$ and $t \in [0,1] \}$. It is clear that $f$ is 1-Lipschitz in $B(\CS_r,\CS_g)$ implies $f(x)-f(y) \leq d(x, y), \forall x \in \CS_r, \forall y \in \CS_g$. Thus, it is a sufficient constraint for Wasserstein distance in Eq.~\eqref{eq_w_dual_form_1}. In fact, $f(x)$ is $k$-Lipschitz in $B(\CP_r,\CP_g)$ is also a sufficient condition for all properties described in Lipschitz GANs. 

Secondly, enforcing $k$-Lipschitz with regularization would provide a dynamic Lipschitz constant $k$. 

\vspace{2pt}
\begin{lemma}\label{lem_lipschitz_constant}
With Wasserstein GAN objective, we have $\min_{f \in \mathcal{F}_\text{k-Lip}} J_D(f) = k \cdot \min_{f \in \mathcal{F}_\text{1-Lip}} J_D(f)$. 
\end{lemma} 
\vspace{-5pt}

Assuming we can directly control the Lipschitz constant $k(f)$ of $f$, the total loss of the discriminator becomes $J(k) \triangleq \min_{f \in \mathcal{F}_\text{k-Lip}} J_D(f) + \lambda \cdot (k-k_0)^2$. 
With Lemma \ref{lem_lipschitz_constant}, let $\alpha=-\min_{f \in \mathcal{F}_\text{1-Lip}} J_D(f)$, then $J(k)=-k\cdot\alpha+\lambda\cdot (k-k_0)^2$, and $J(k)$ achieves its minimum when $k=\frac{\alpha}{2\lambda}+k_0$. When $\alpha$ goes to zero, i.e., $\CP_g$ converges to $\CP_r$, the optimal $k$ decreases. And when $\CP_r=\CP_g$, we have $\alpha=0$ and the optimal $k=k_0$. The similar analysis applies to Lipschitz GANs. 

\subsection{Existing Methods}

For practical methods, though spectral normalization \citep{sngan} recently demonstrates their excellent results in training GANs, spectral normalization is an absolute constraint for Lipschitz over the entire space, i.e., constricting the maximum gradient of the entire space, which is unnecessary. On the other side, we also notice both penalty methods proposed in \citep{wgangp} and \citep{wganlp} are not exact implementation of the Lipschitz continuity condition, because it does not directly penalty the maximum gradient, but penalties all gradients towards the given target Lipschitz constant or penalties all these greater than one towards the given target. 

We also empirically found that the existing methods including spectral normalization \citep{sngan}, gradient penalty \citep{wgangp}, and Lipschitz penalty \citep{wganlp} all fail to converge to the optimal $\ff(x)$ in some of our synthetic experiments. 

\begin{figure*}[t]
	\centering
	\begin{subfigure}{0.45\linewidth}
		\centering
		\includegraphics[width=0.95\columnwidth]{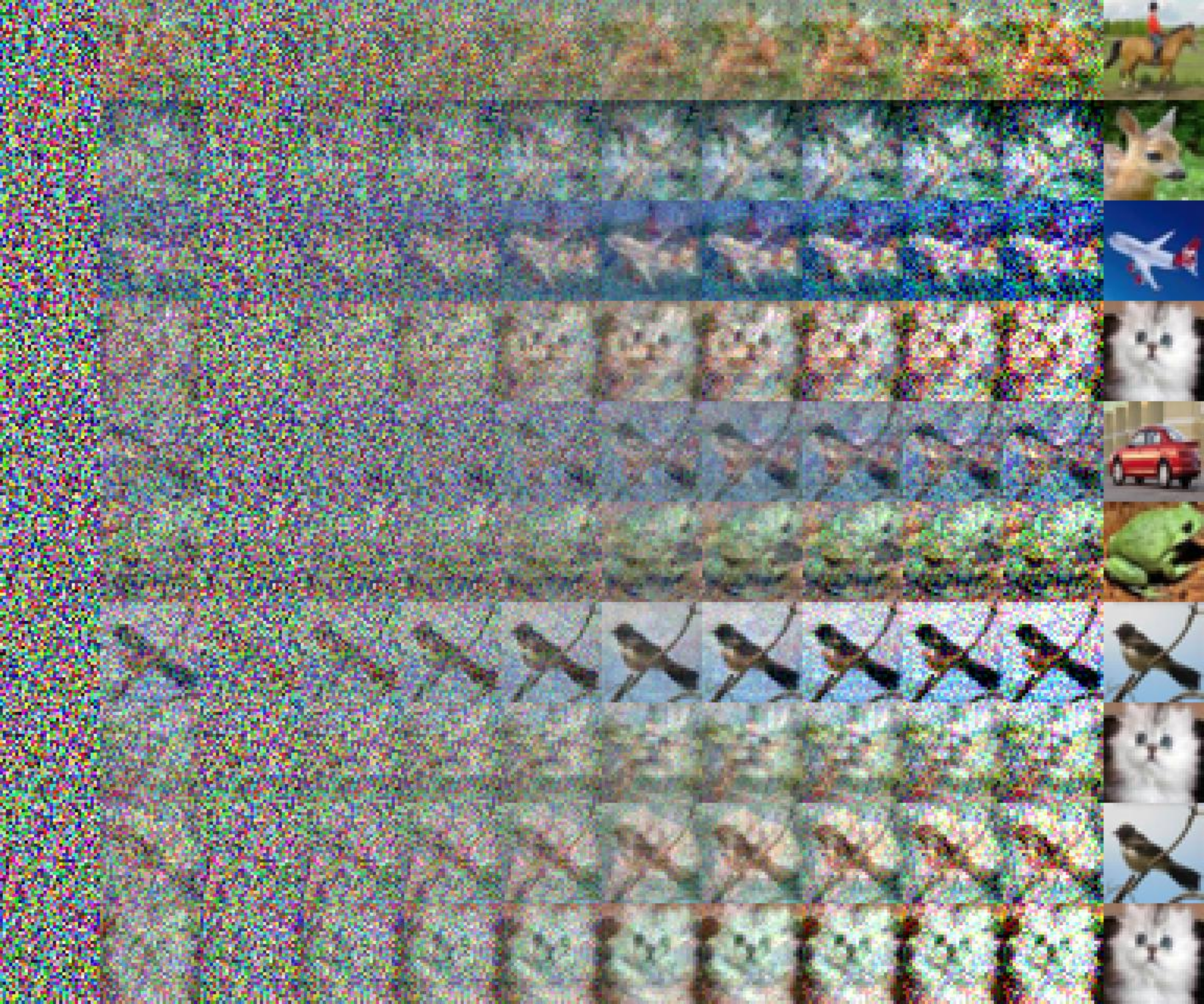}
		\caption{Gradient Penalty}
		\label{fig_gp}
	\end{subfigure}
	\begin{subfigure}{0.45\linewidth}
		\centering	
		\includegraphics[width=0.95\columnwidth]{figures/Case_5-10++_wgans_gp1_0_adam_1e-4_mlpdense256x16_cifar10.pdf}
		\caption{Maximum Gradient Penalty}
		\label{fig_maxgp}
	\end{subfigure}	
	\vspace{-5pt}
	\caption{Comparison between gradient penalty and maximum gradient penalty, with $\CP_r$ and $\CP_g$ consist of ten real and noise images, respectively. The leftmost in each row is a $x \in \CS_g$ and the second is its gradient $\nabla_{\!x} \ff(x)$. The interiors are $x+\epsilon\cdot\nabla_{\!x} \ff(x)$ with increasing $\epsilon$, which will pass through a real sample, and the rightmost is the nearest $y \in \CS_r$.}
	\label{fig_gradient_direction_maxgp_compare}
\end{figure*}

\subsection{The New Method}
Note that this practical method of imposing Lipschitz continuity is not the key contribution of this work. We leave the more rigorous study on this topic as our further work. We introduce it for the necessity for understanding our paper and reproducing of experiments. 

Combining the idea of spectral normalization and gradient penalty, we developed a new way of implementing the regularization of Lipschitz continuity in our experiments. Spectral normalization is actually constraining the maximum gradient over the entire space. And as we argued previously, enforcing Lipschitz continuity in the blending region is sufficient. Therefore, we propose to restricting the maximum gradient over the blending region:
\begin{align}
J_{\text{maxgp}} = \lambda \max_{x\sim B(\CS_r,\CS_g)} [ \big\lVert \nabla_{\!x} f(x) \big\rVert^2] 
\end{align}
In practice, we sample $x$ from $B(\CS_r,\CS_g)$ as in \citep{wgangp,wganlp} using training batches of real and fake samples. 

We compare the practical result of (centralized) gradient penalty $\mE_{x\sim B} [ \big\lVert \nabla_{\!x} f(x) \big\rVert^2]$ and the proposed maximum gradient penalty in Figure \ref{fig_gradient_direction_maxgp_compare}. Before switching to maximum gradient penalty, we struggled for a long time and cannot achieve a high quality result as shown in Figure \ref{fig_maxgp}. The other forms of gradient penalty \citep{wgangp,wganlp} perform similar as $\mE_{x\sim B} [ \big\lVert \nabla_{\!x} f(x) \big\rVert^2]$.

To improve the stability and reduce the bias introduced via batch sampling, one can further keep track $x$ with the maximum $\big\lVert \nabla_{\!x} f(x) \big\rVert$. A practical and light weight method is to maintain a list $S_\text{max}$ that has the currently highest (top-k) $\big\lVert \nabla_{\!x} f(x) \big\rVert_2$ (initialized with random samples), use the $S_\text{max}$ as part of the batch that estimates $J_{\text{maxgp}}$, and update the $S_\text{max}$ after each batch updating of the discriminator. According to our experiments, it is usually does not improve the training significantly. 

\section{Extended Discussions and More Details} \label{sec_details}

\subsection{Various $\phi$ and $\varphi$ That Satisfies Eq.~(\ref{eq_solvable})}

For Lipschitz GANs, $\phi$ and $\varphi$ are required to satisfy Eq.~\eqref{eq_solvable}. Eq.~(\ref{eq_solvable}) is actually quite general and there exists many other instances, e.g., $\phi(x)=\varphi(-x)=x$, $\phi(x)=\varphi(-x)=-\log(\sigma(-x))$, $\phi(x)=\varphi(-x)=x+\sqrt{x^2+\alpha}$ with $\alpha>0$, $\phi(x)=\varphi(-x)=\exp(x)$, etc. We plot these instances of $\phi$ and $\varphi$ in Figure \ref{fig_function_curve}. 

To devise a loss satisfies Eq.~(\ref{eq_solvable}), it is practical to let $\phi$ be an increasing function with non-decreasing derivative and set $\phi(x)=\varphi(-x)$.  Note that rescaling and offsetting along the axes are trivial operation to found more $\phi$ and $\varphi$ within a function class, and linear combination of two or more $\phi$ or $\varphi$ from different function classes also keep satisfying Eq.~\eqref{eq_solvable}. 

\subsection{Experiment Details}

In our experiments with real datas (CIFAR-10, Tiny Imagenet and Oxford 102), we follow the network architecture and hyper-parameters in \citep{wgangp}. The network architectures are detailed in Table \ref{tab:my_label}. We use Adam optimizer with beta1=0.0, beta2=0.9, and the learning rate is 0.0002 which linear decays to zero in 200, 000 iterations. We use 5 discriminator updates per generator update. We use MaxGP for all our experiments of LGANs and search the best penalty weight $\lambda$ in $[0.01, 0.1, 1.0, 10.0]$. Please check more details in our codes. For all experiments in Table~\ref{table2}, we only change $\phi$ and $\varphi$ and the dataset, and all other components are fixed. 

We plot the IS training curve of LGANs in Figure~\ref{training_curve_cifar_icp} and \ref{training_curve_tiny_icp}. We provide the visual results of LGANs in Figure~\ref{fig_cifar10}, Figure~\ref{fig_tiny} for CIFAR-10 and Tiny Imagenet, respectively. As an extra experiment, we also provide the visual results of LGANs on Oxford 102 in Figure~\ref{fig_flowers}. 

\begin{figure*}[h]
	\centering
	\includegraphics[width=0.99\columnwidth]{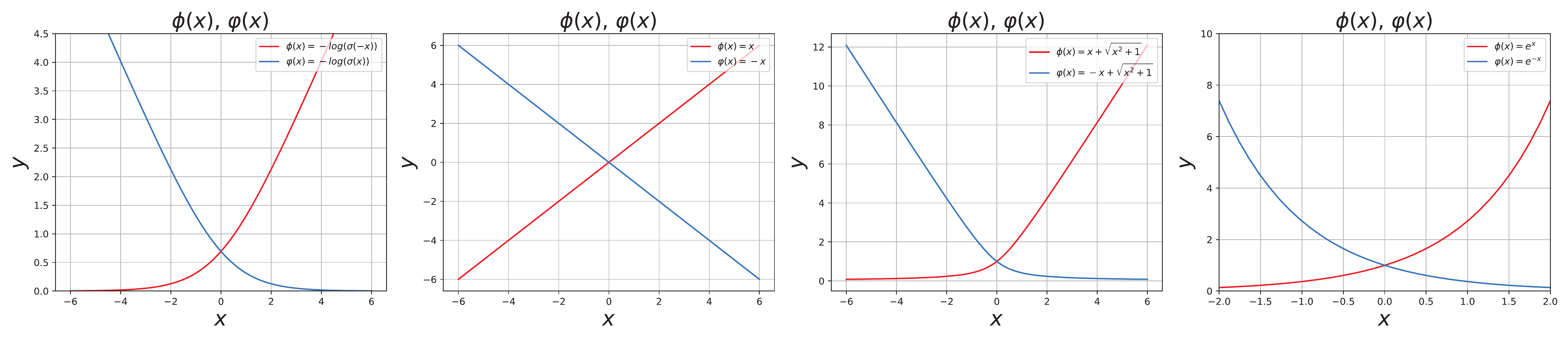}
	\vspace{-7pt}
	\caption{Various $\phi$ and $\varphi$ that satisfies Eq.~\eqref{eq_solvable}.}
	\label{fig_function_curve} 
\end{figure*}

\begin{figure*}[h]
\begin{minipage}{.5\textwidth}	
    \centering    
    \includegraphics[width=0.99\columnwidth]{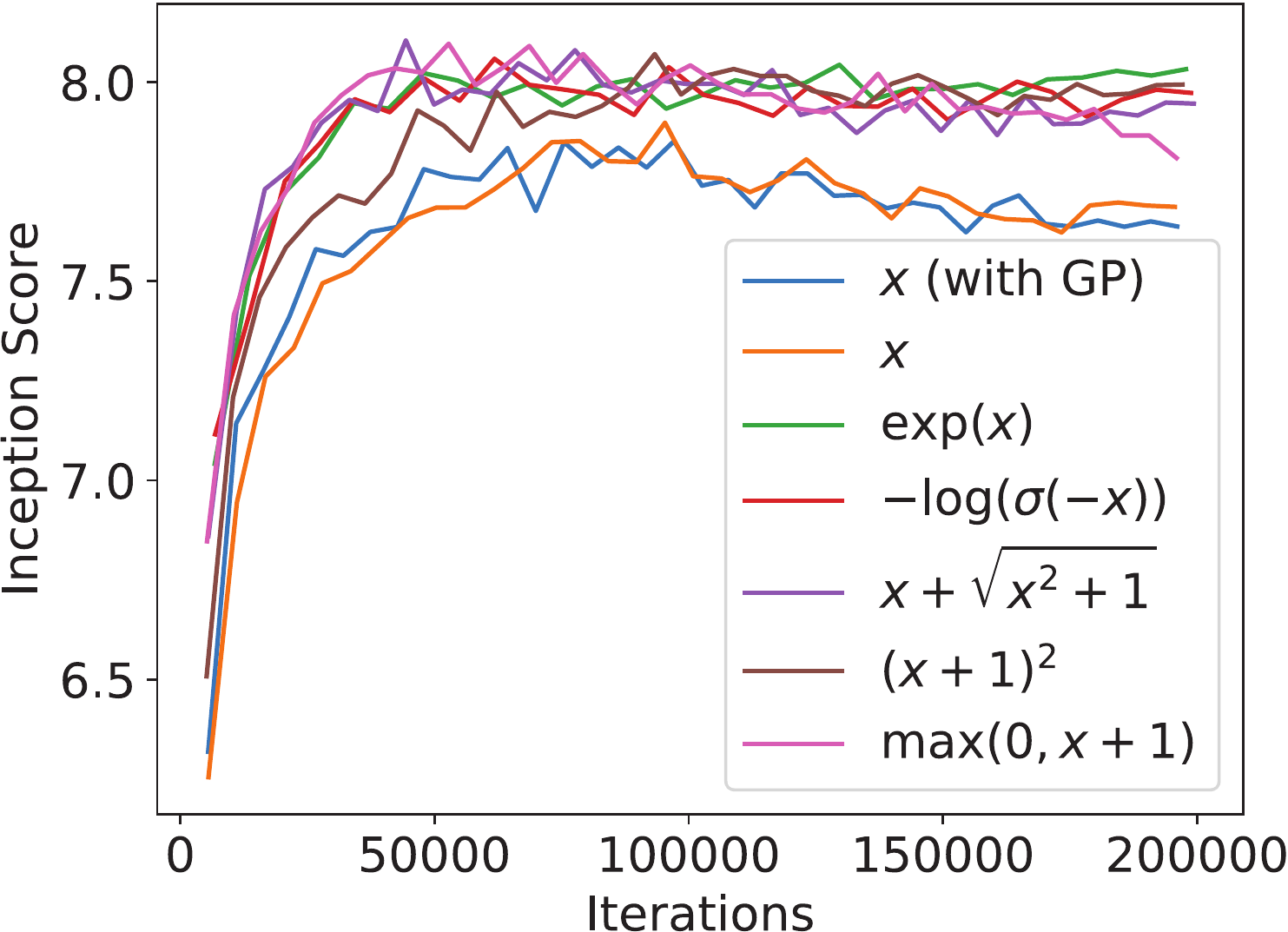}
    \caption{IS training curves on CIFAR-10.}
    \label{training_curve_cifar_icp}
\end{minipage}
\begin{minipage}{.5\textwidth}	
    \centering    
    \includegraphics[width=0.99\columnwidth]{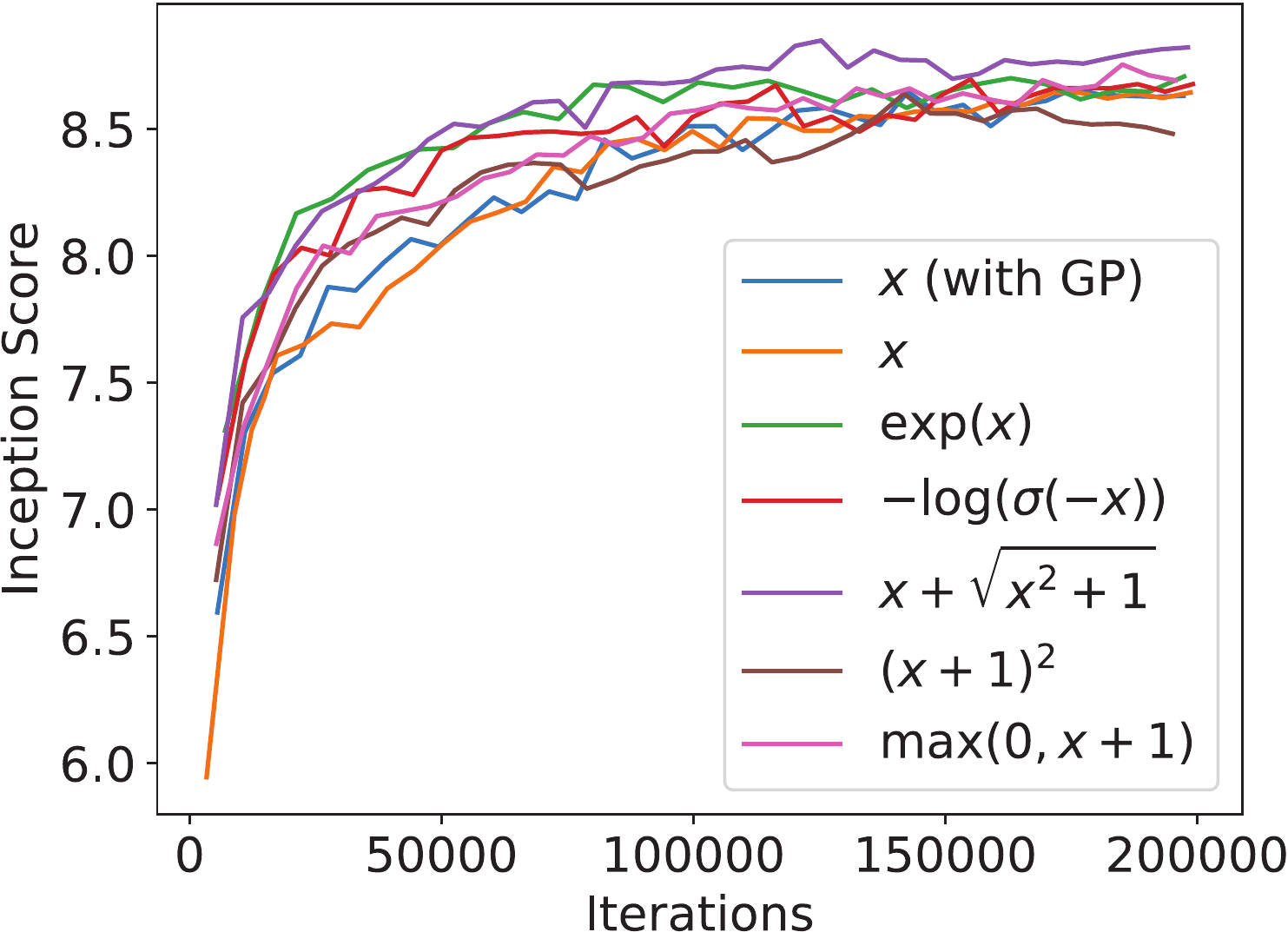}
    \caption{IS training curves on Tiny ImageNet.}
    \label{training_curve_tiny_icp}
\end{minipage}
\vspace{-5pt}
\end{figure*} 

\begin{table}[!h]

\begin{minipage}{.50\textwidth}	
    \centering
    \begin{tabular}{c|c|c|c}
    \multicolumn{4}{l}{Generator:} \\
    \hline \Tstrut 
    Operation     & Kernel & Resample & Output Dims \\[5pt]
    \hline \Tstrut
    Noise         &   N/A     &    N/A     & 128\\[5pt]
    Linear        &   N/A  &   N/A   &  128$\times$4$\times$4   \\[5pt]
    Residual block & 3$\times$3 &UP& 128$\times$8$\times$8 \\[5pt]
    Residual block & 3$\times$3 & UP & 128$\times$16$\times$16 \\[5pt]
    Residual block & 3$\times$3 & UP & 128$\times$32$\times$32 \\[5pt]
    Conv \& Tanh & 3$\times$3 & N/A & 3$\times$32$\times$32 \\[5pt]
    \hline
    \end{tabular}
\end{minipage}
\begin{minipage}{.50\textwidth}	
    \centering
    \begin{tabular}{c|c|c|c}
    \multicolumn{4}{l}{Discriminator:} \\
    \hline \Tstrut
    Operation     & Kernel & Resample & Output Dims \\[5pt]
    \hline \Tstrut
    Residual Block & 3$\times$3$\times$2 & Down & 128$\times$16$\times$16\\[5pt]
    Residual Block & 3$\times$3$\times$2 & Down & 128$\times$8$\times$8\\[5pt]
    Residual Block &3$\times$3$\times$2 & N/A & 128$\times$8$\times$8\\[5pt]
    Residual Block &3$\times$3$\times$2 & N/A & 128$\times$8$\times$8\\[5pt]
    ReLU,mean pool & N/A & N/A & 128 \\[5pt]
    Linear & N/A & N/A & 1 \\[5pt]
    \hline
    \end{tabular}
    \end{minipage}
\caption{The network architectures.}
\label{tab:my_label}
\end{table}

\begin{figure}[!htbp]
\vspace{5pt}
	\centering
	\begin{subfigure}{0.427\linewidth}
	    \centering
	    \includegraphics[width=0.95\columnwidth]{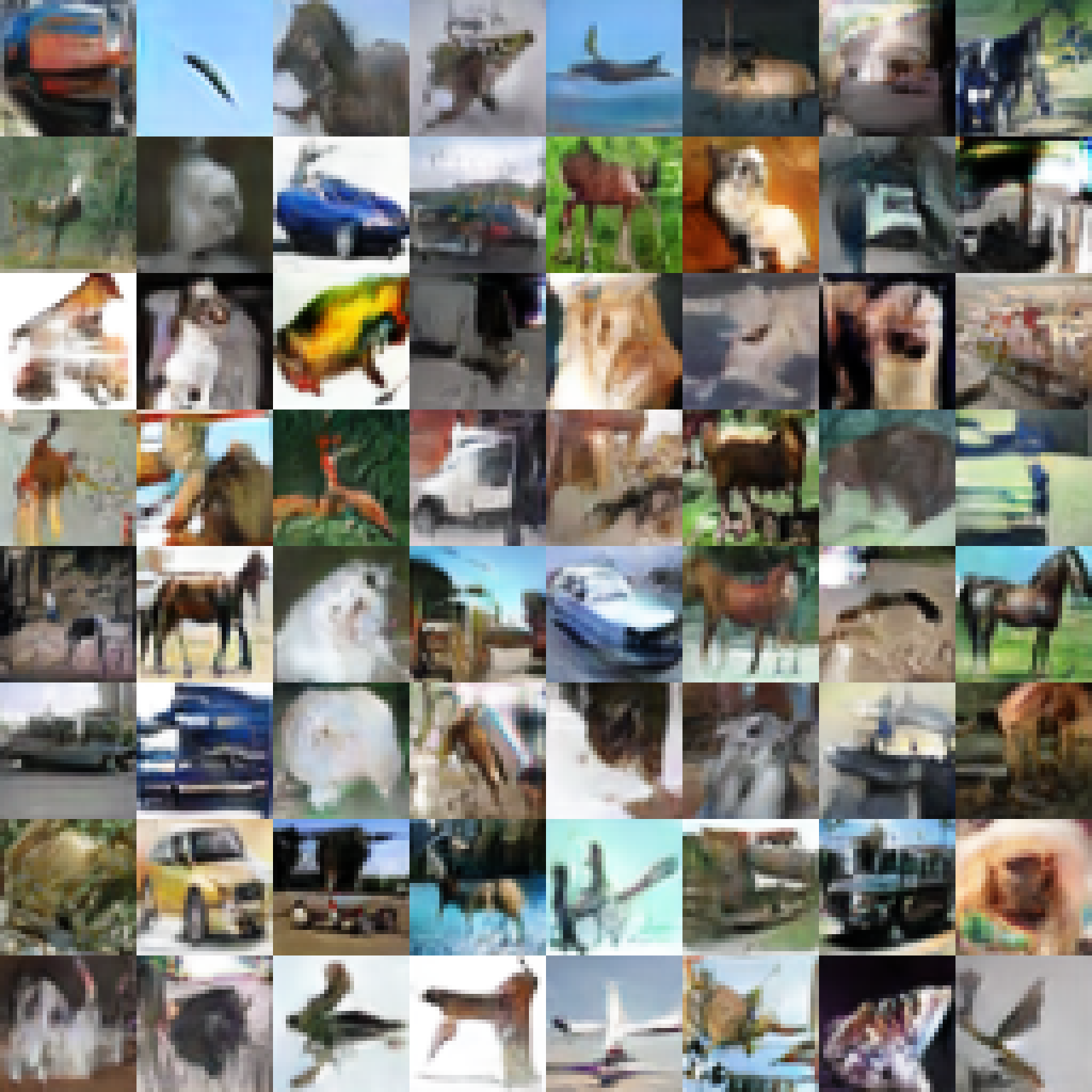}
	    \vspace{-3pt}
	    \caption{$x$}
	\end{subfigure}
	\begin{subfigure}{0.427\linewidth}
    	\centering
	    \includegraphics[width=0.95\columnwidth]{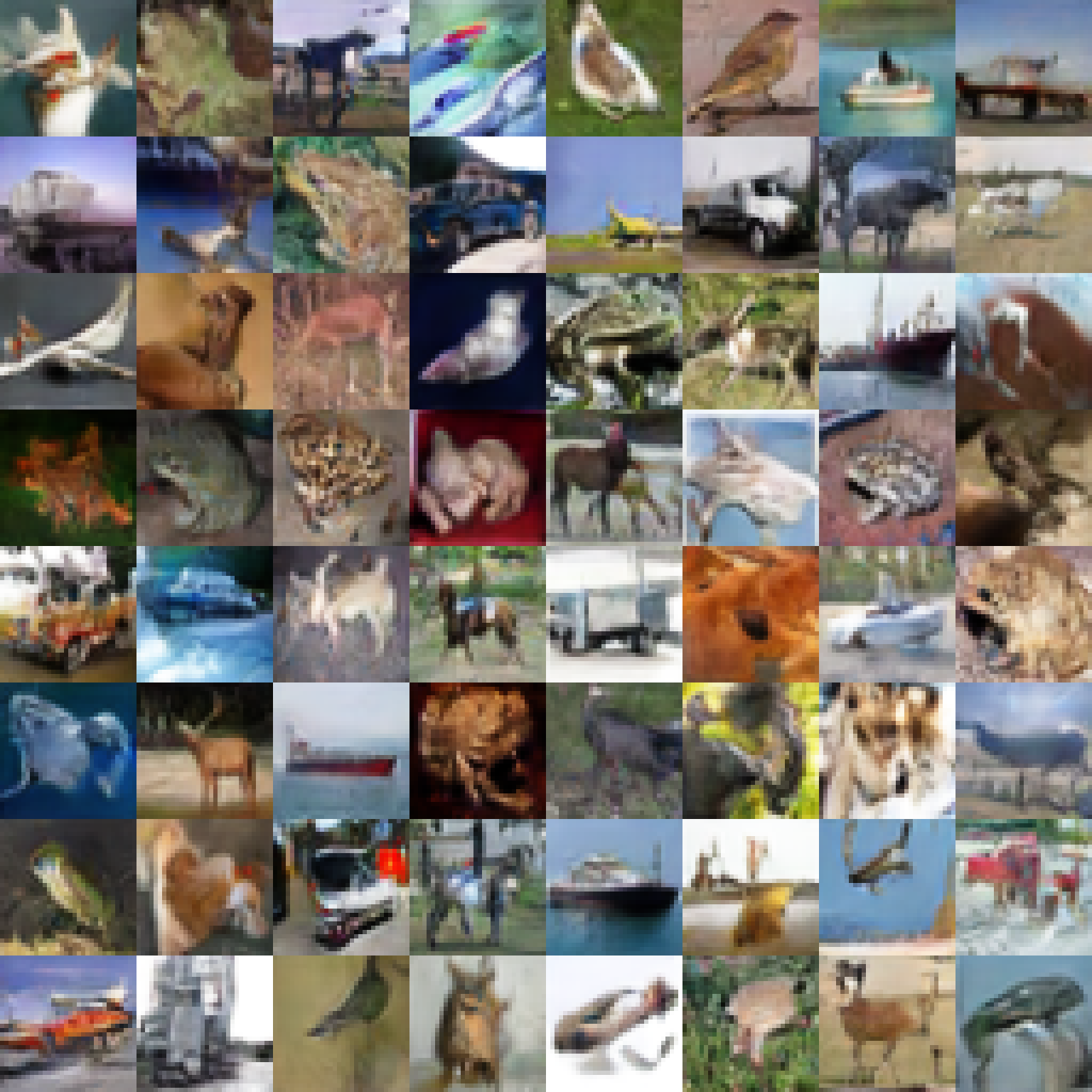}
	    \vspace{-3pt}
	    \caption{$\exp(x)$}
	\end{subfigure}
	\begin{subfigure}{0.427\linewidth}
		\centering
		\includegraphics[width=0.95\columnwidth]{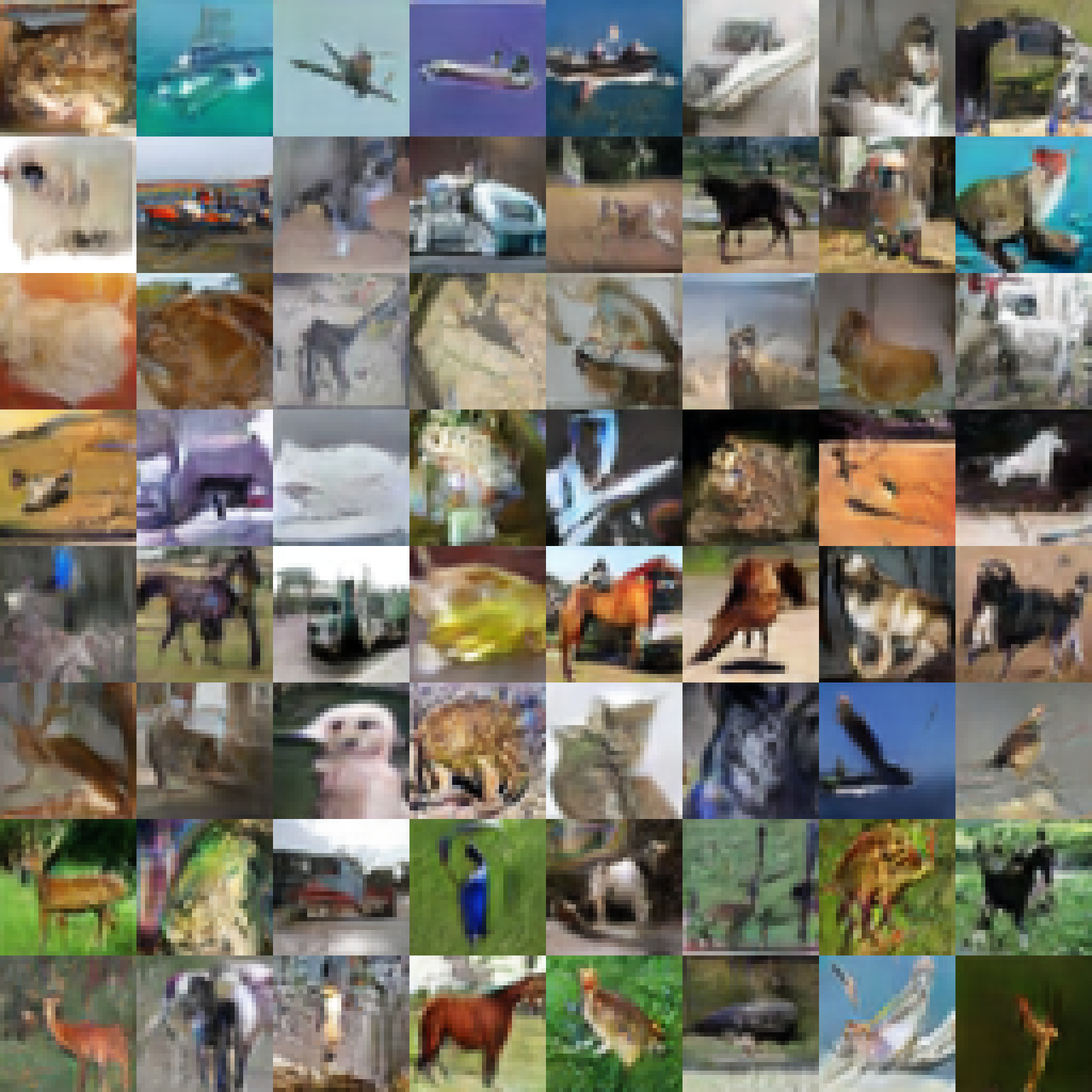}
		\vspace{-3pt}
		\caption{$-\log(\sigma(-x))$}
	\end{subfigure}
	\begin{subfigure}{0.427\linewidth}
		\centering	
		\includegraphics[width=0.95\columnwidth]{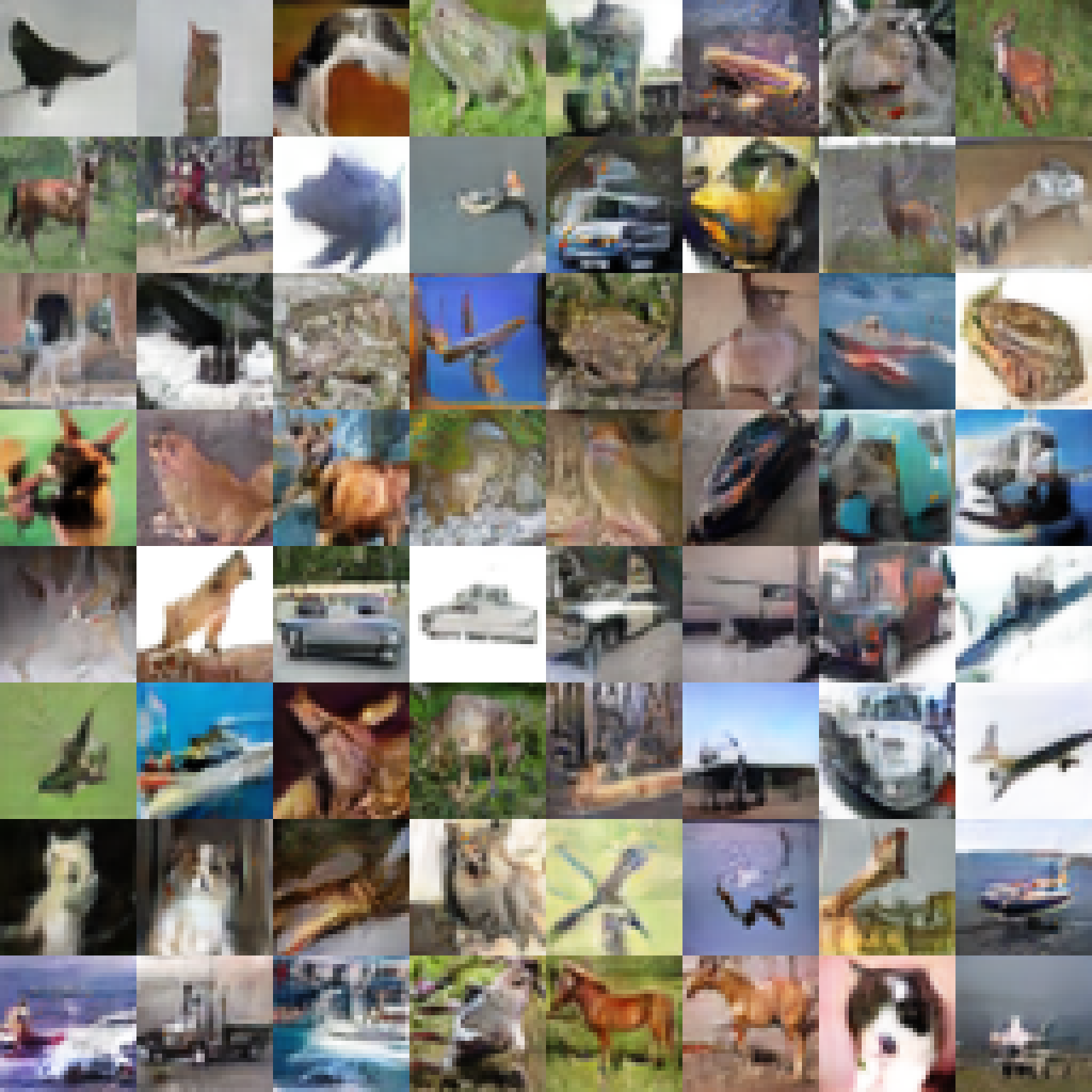}
		\vspace{-3pt}
		\caption{$x+\sqrt{x^2+1}$}
	\end{subfigure}	
	\begin{subfigure}{0.427\linewidth}
		\centering	
		\includegraphics[width=0.95\columnwidth]{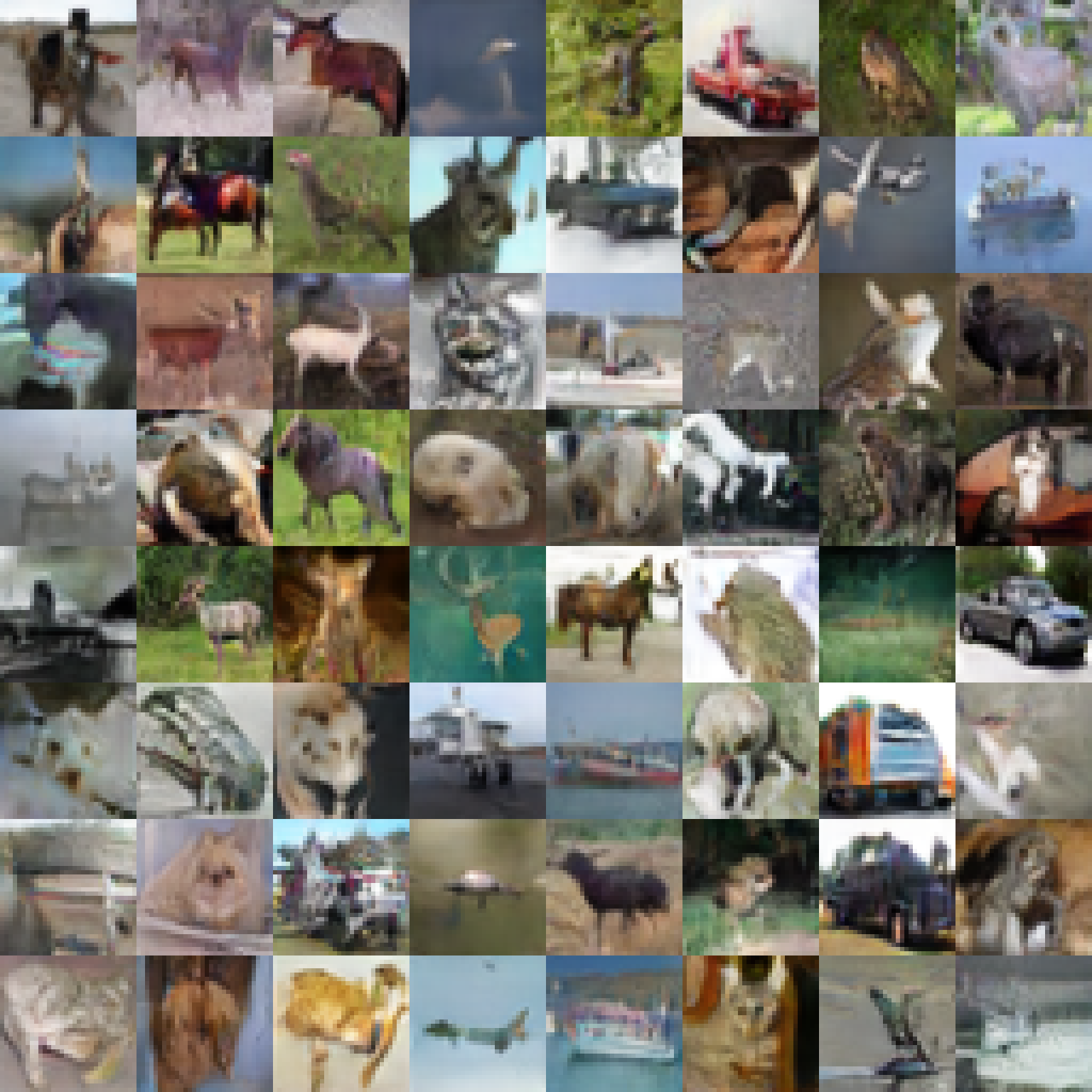}
		\vspace{-3pt}
		\caption{$(x+1.0)^2$}
	\end{subfigure}	
	\begin{subfigure}{0.427\linewidth}
		\centering	
		\includegraphics[width=0.95\columnwidth]{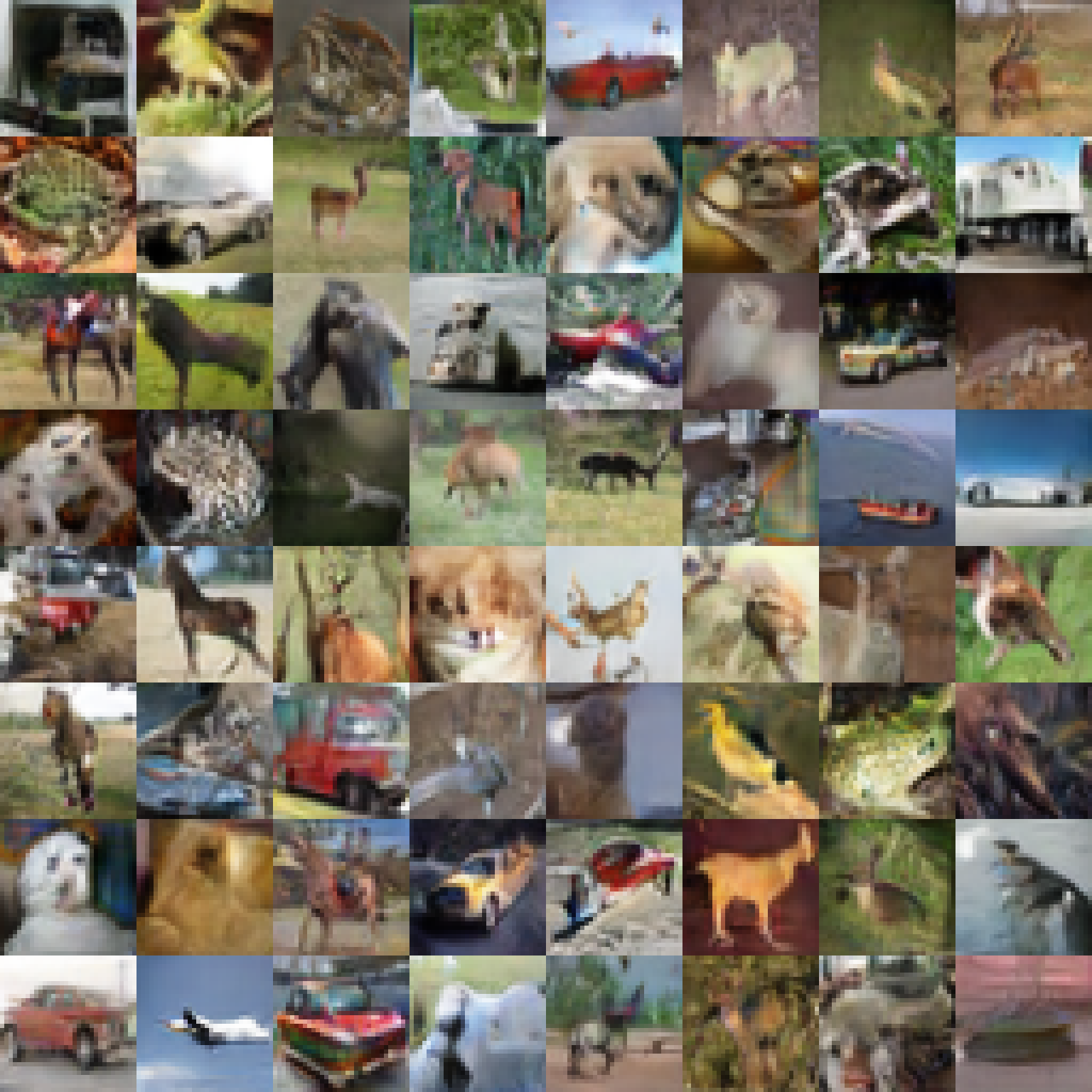}
		\vspace{-3pt}
		\caption{$\max(0, x+1.0)$}
	\end{subfigure}	
	\vspace{-5pt}
	\caption{Random samples of LGANs with different loss metrics on CIFAR-10.}
	\label{fig_cifar10}
\end{figure}

\begin{figure}[!htbp]
\vspace{5pt}
	\centering
	\begin{subfigure}{0.427\linewidth}
	    \centering
	    \includegraphics[width=0.95\columnwidth]{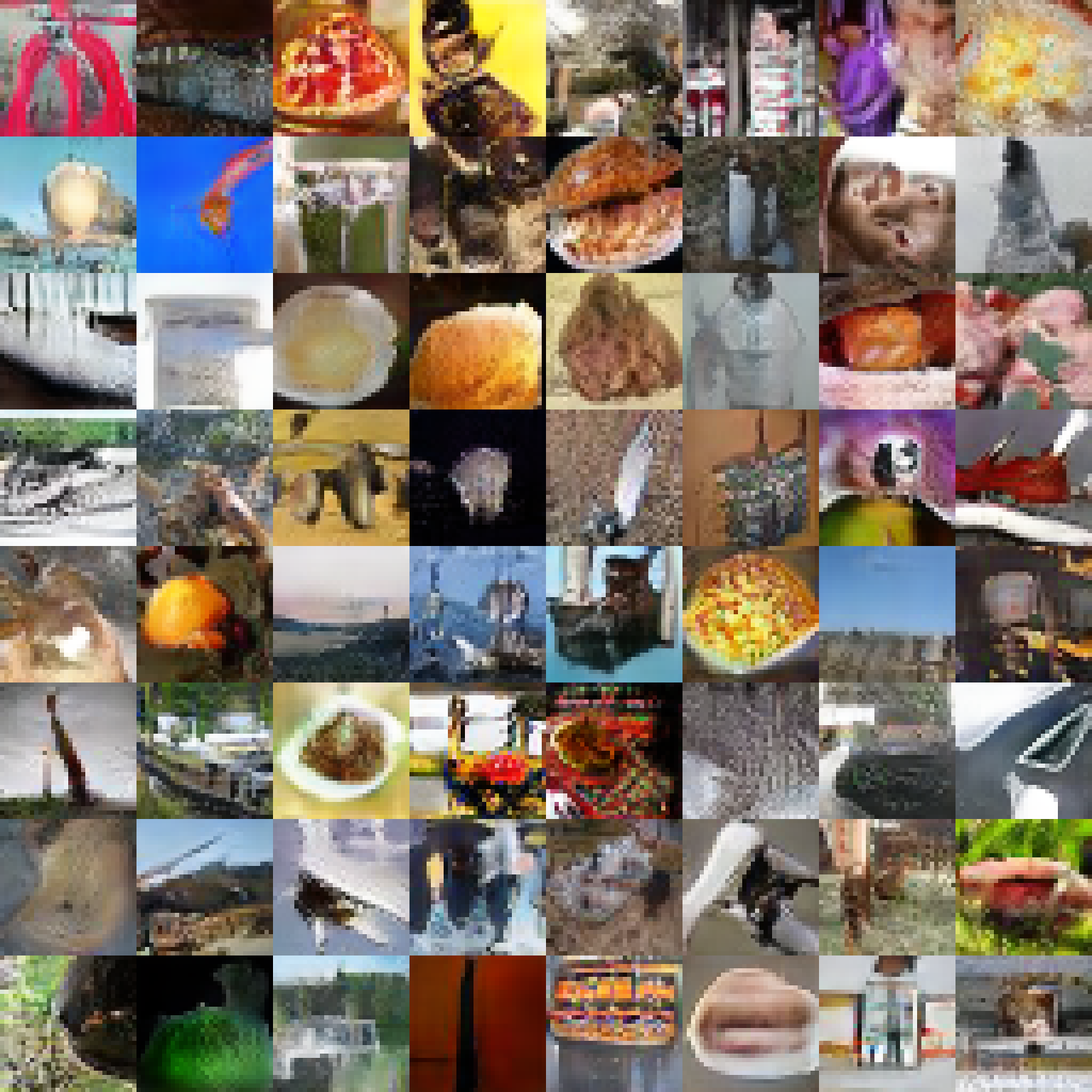}
	    \vspace{-3pt}
	    \caption{$x$}
	\end{subfigure}
	\begin{subfigure}{0.427\linewidth}
		\centering	
		\includegraphics[width=0.95\columnwidth]{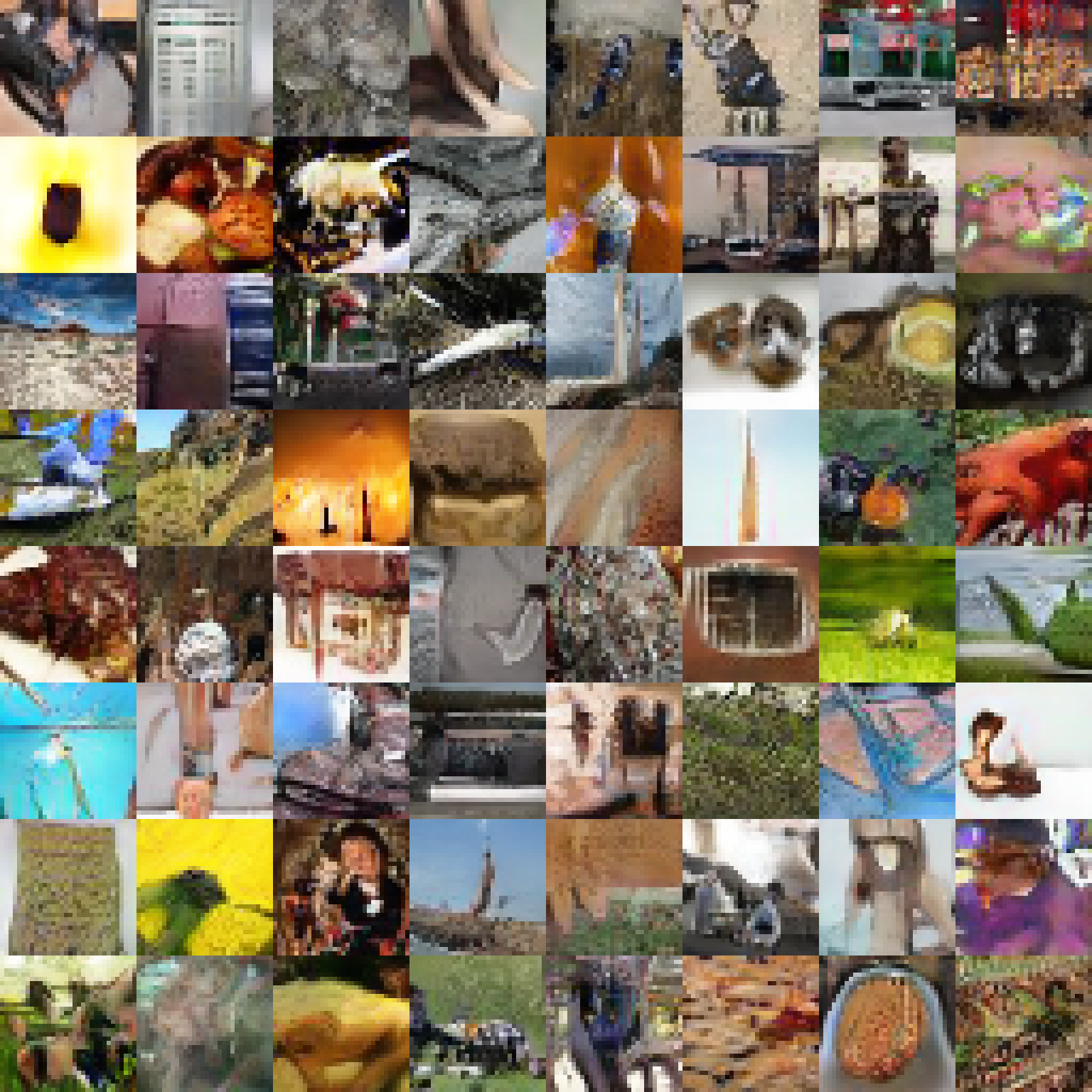}
		\vspace{-3pt}
		\caption{$\exp(x)$}
	\end{subfigure}	
	\begin{subfigure}{0.427\linewidth}
		\centering
		\includegraphics[width=0.95\columnwidth]{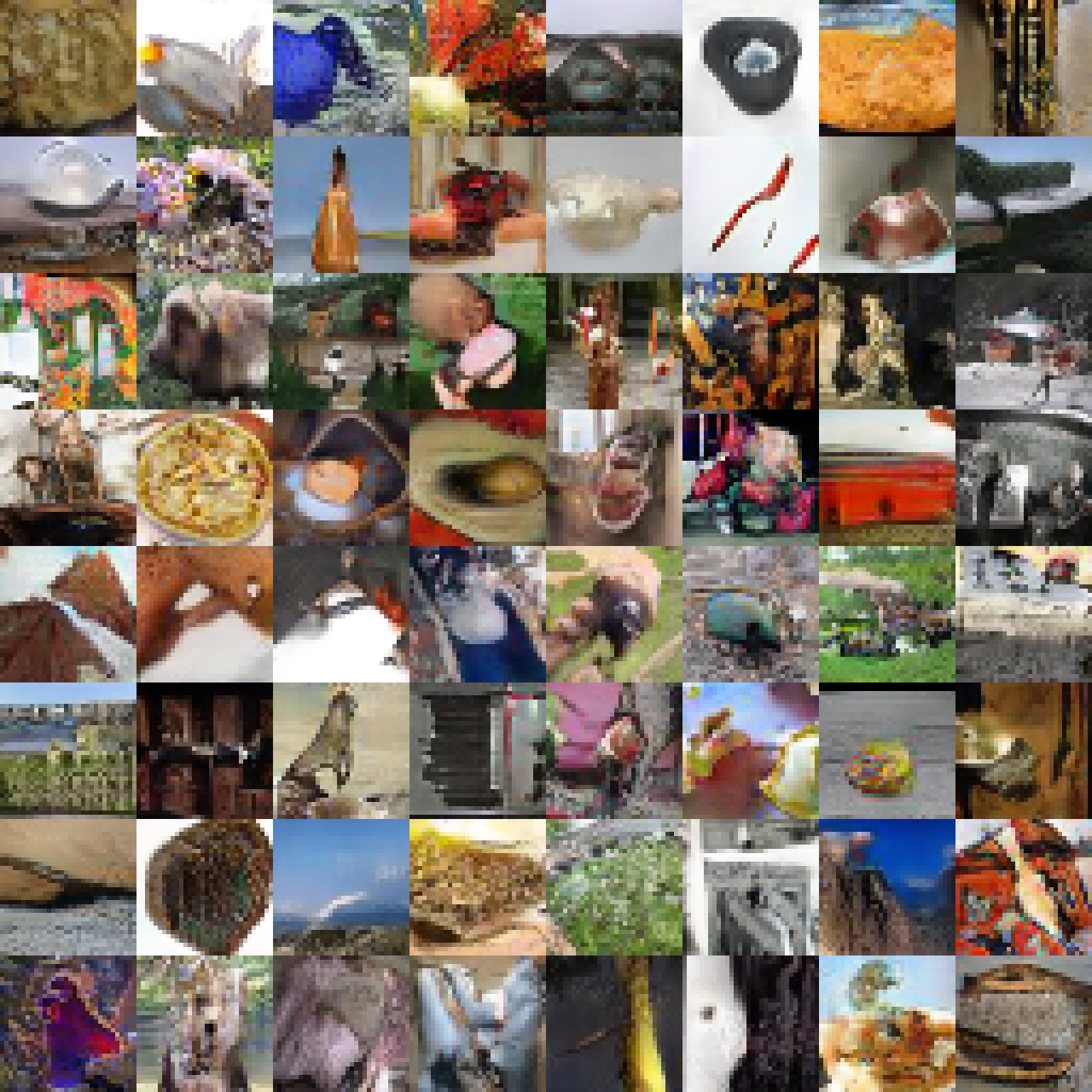}
		\vspace{-3pt}
		\caption{$-\log(\sigma(-x))$}
	\end{subfigure}
	\begin{subfigure}{0.427\linewidth}
		\centering	
		\includegraphics[width=0.95\columnwidth]{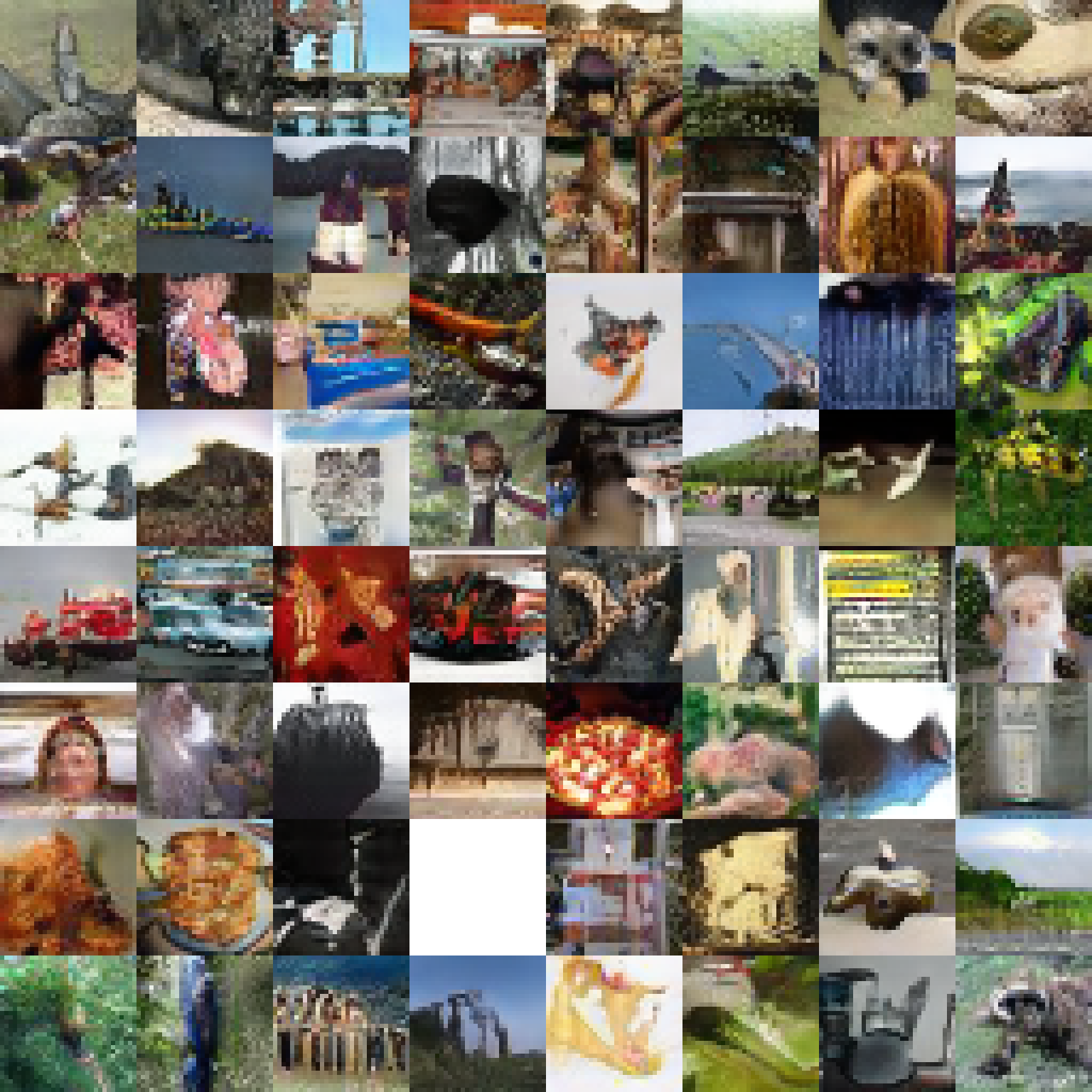}
		\vspace{-3pt}
		\caption{$x+\sqrt{x^2+1}$}
	\end{subfigure}	
	\begin{subfigure}{0.427\linewidth}
    	\centering
	    \includegraphics[width=0.95\columnwidth]{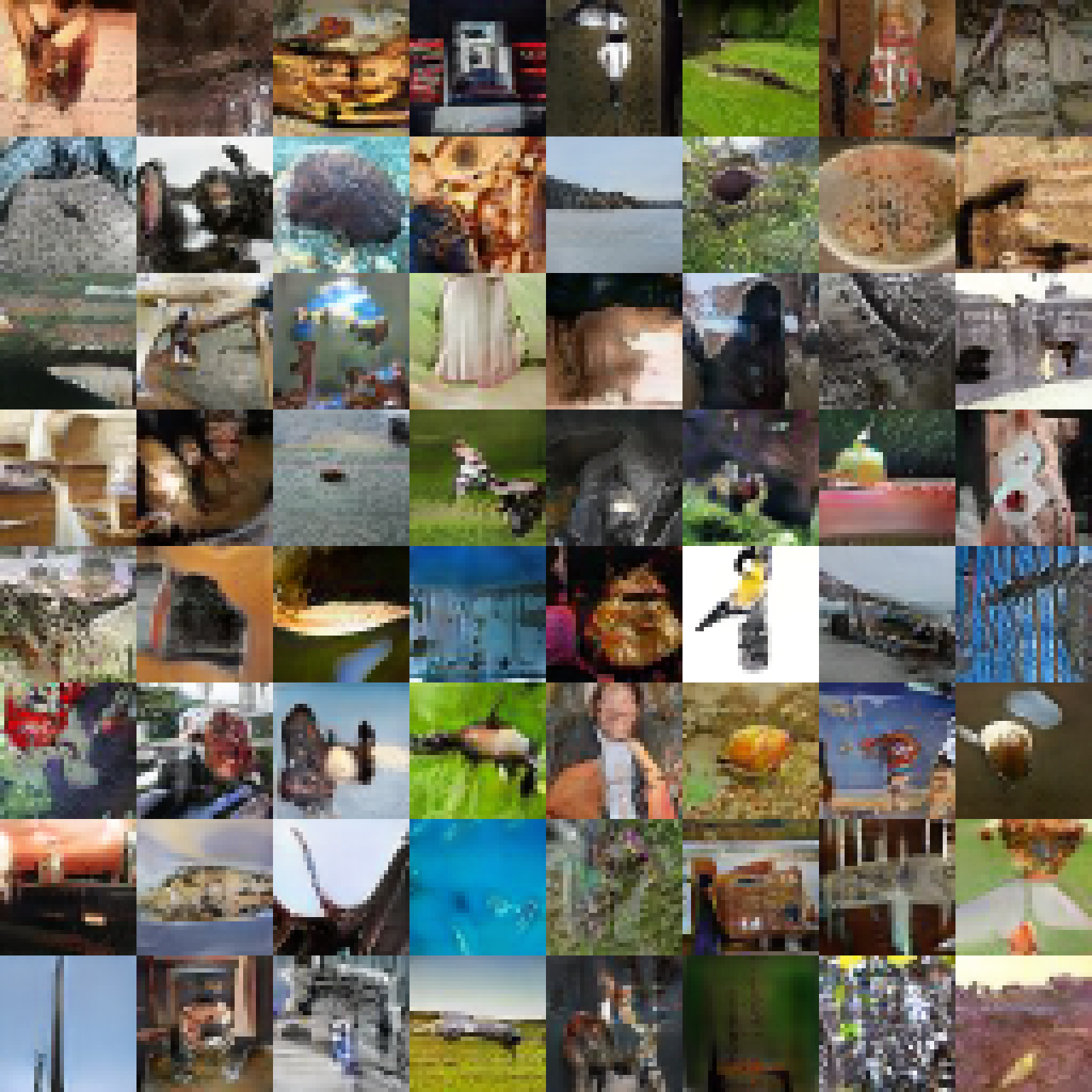}
	    \vspace{-3pt}
	    \caption{$(x+1.0)^2$}
	\end{subfigure}
	\begin{subfigure}{0.427\linewidth}
		\centering	
		\includegraphics[width=0.95\columnwidth]{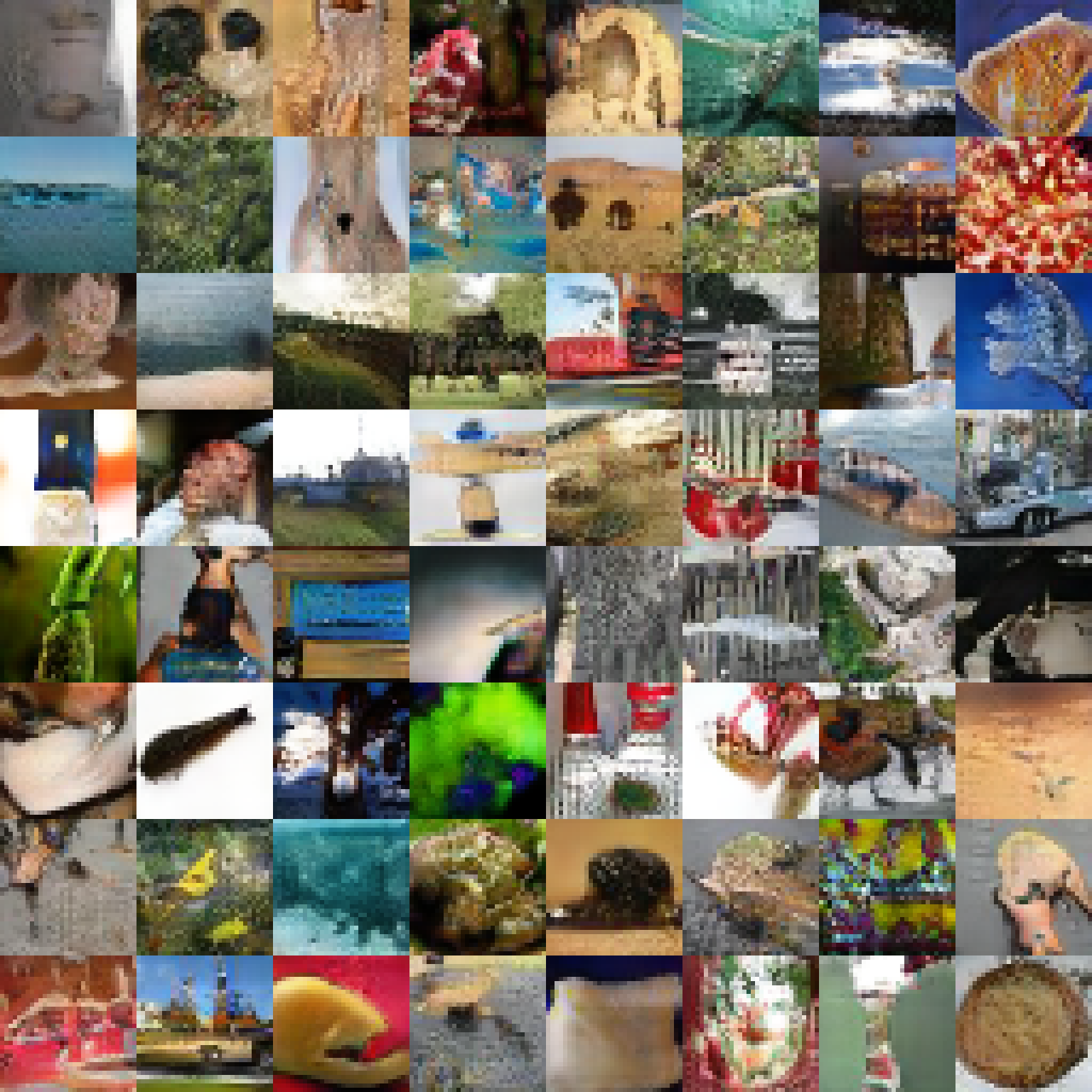}
		\vspace{-3pt}
		\caption{$\max(0, x+1.0)$}
	\end{subfigure}	
	\vspace{-5pt}
	\caption{Random samples of LGANs with different loss metrics on Tiny Imagenet.}
	\label{fig_tiny}
\end{figure}

\begin{figure}[!htbp]
\vspace{5pt}
	\centering
	\begin{subfigure}{0.427\linewidth}
	    \centering
	    \includegraphics[width=0.95\columnwidth]{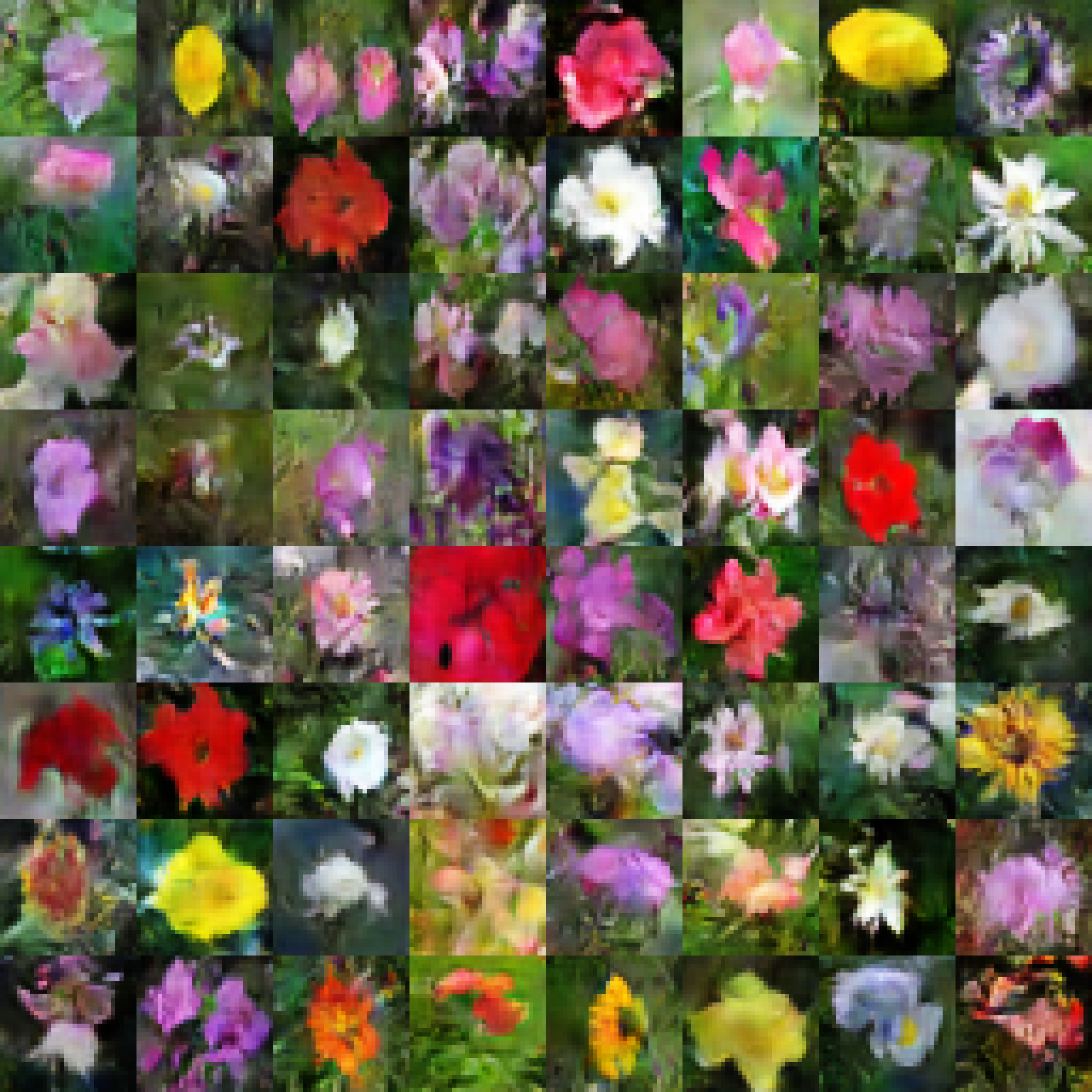}
	    \vspace{-3pt}
	    \caption{$x$}
	\end{subfigure}
	\begin{subfigure}{0.427\linewidth}
		\centering
		\includegraphics[width=0.95\columnwidth]{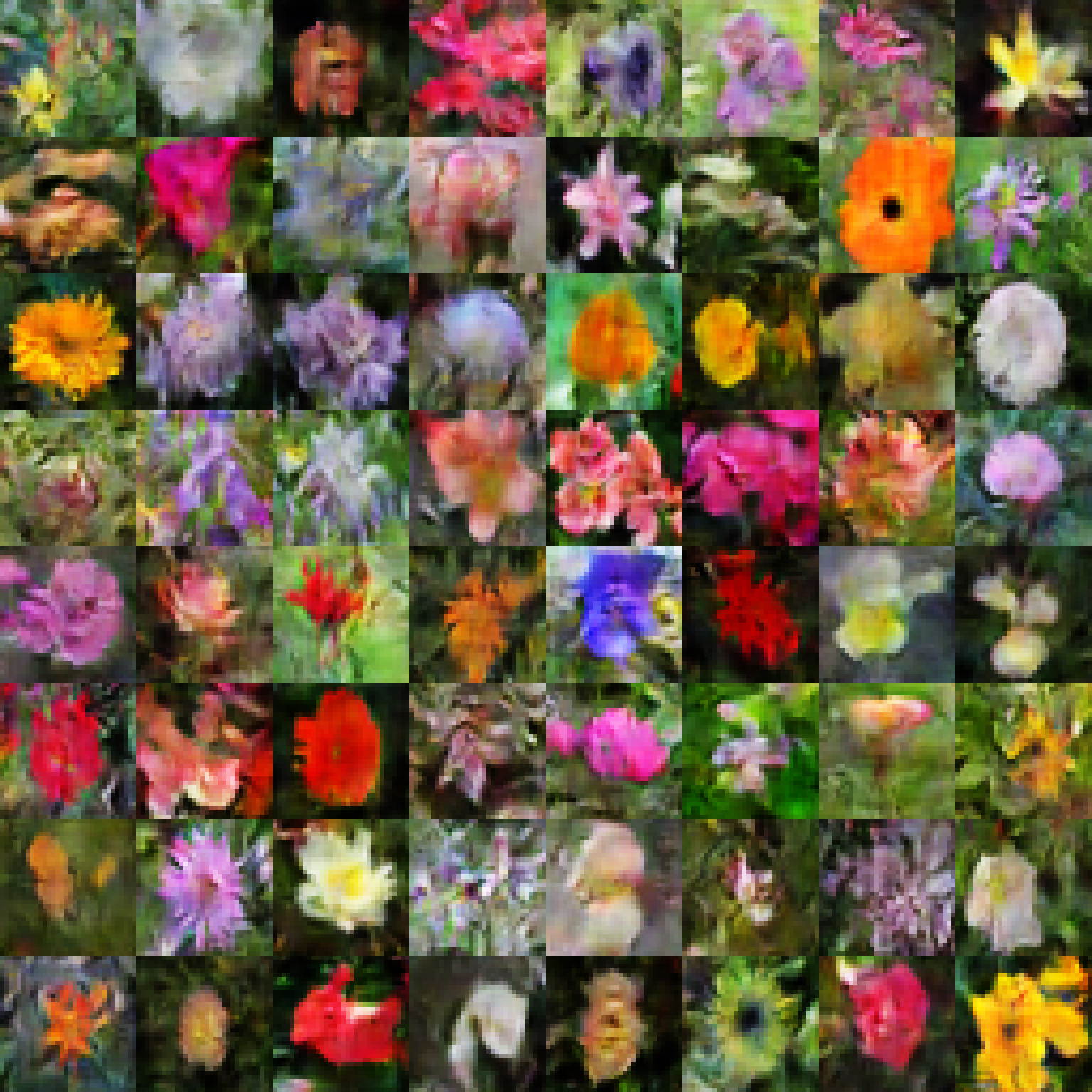}
		\vspace{-3pt}
		\caption{$\exp(x)$}
	\end{subfigure}
	\begin{subfigure}{0.427\linewidth}
		\centering	
		\includegraphics[width=0.95\columnwidth]{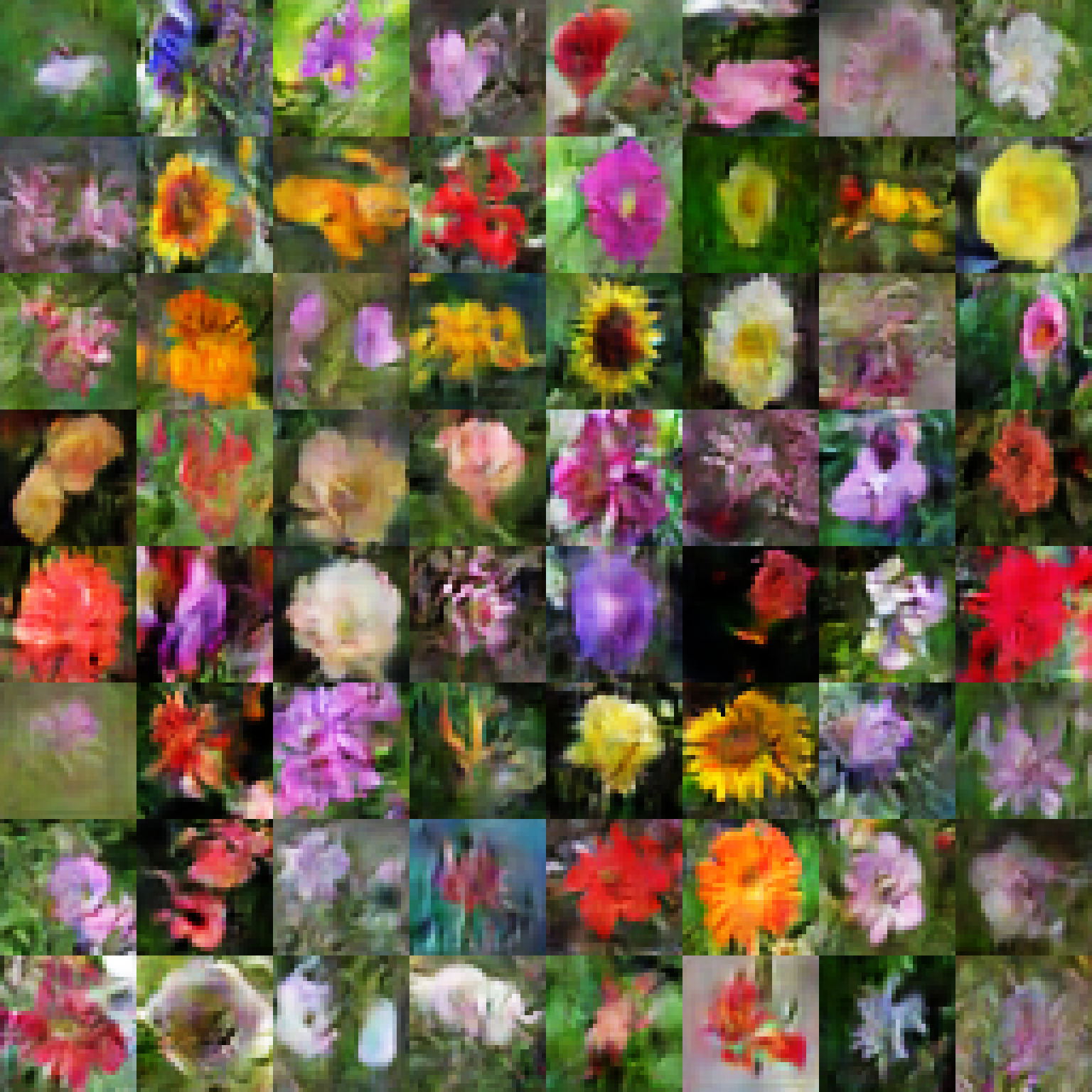}
		\vspace{-3pt}
		\caption{$-\log(\sigma(-x))$}
	\end{subfigure}	
	\begin{subfigure}{0.427\linewidth}
		\centering	
		\includegraphics[width=0.95\columnwidth]{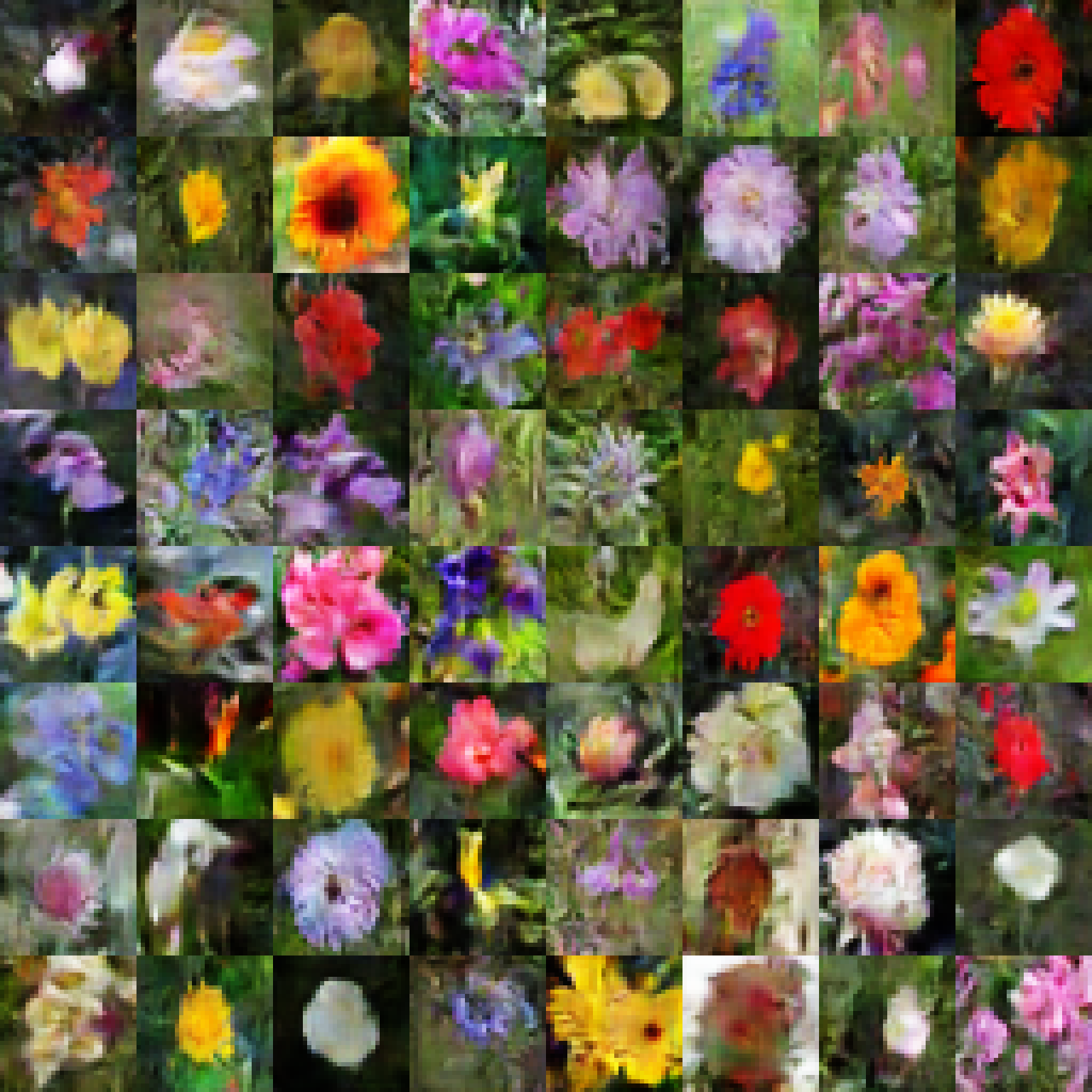}
		\vspace{-3pt}
		\caption{$x+\sqrt{x^2+1}$}
	\end{subfigure}	
	\begin{subfigure}{0.427\linewidth}
    	\centering
	    \includegraphics[width=0.95\columnwidth]{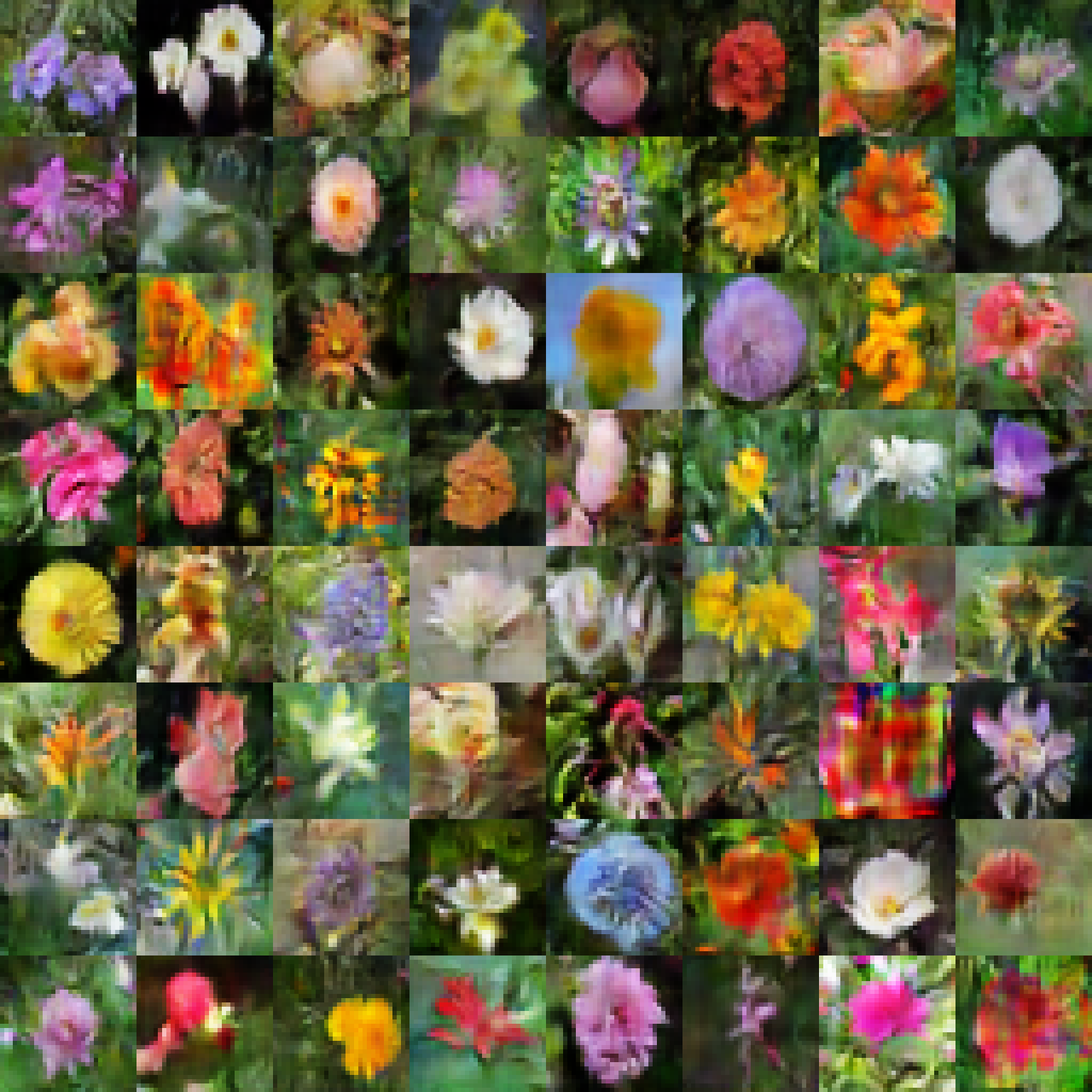}
	    \vspace{-3pt}
	    \caption{$(x+1.0)^2$}
	\end{subfigure}
	\begin{subfigure}{0.427\linewidth}
		\centering	
		\includegraphics[width=0.95\columnwidth]{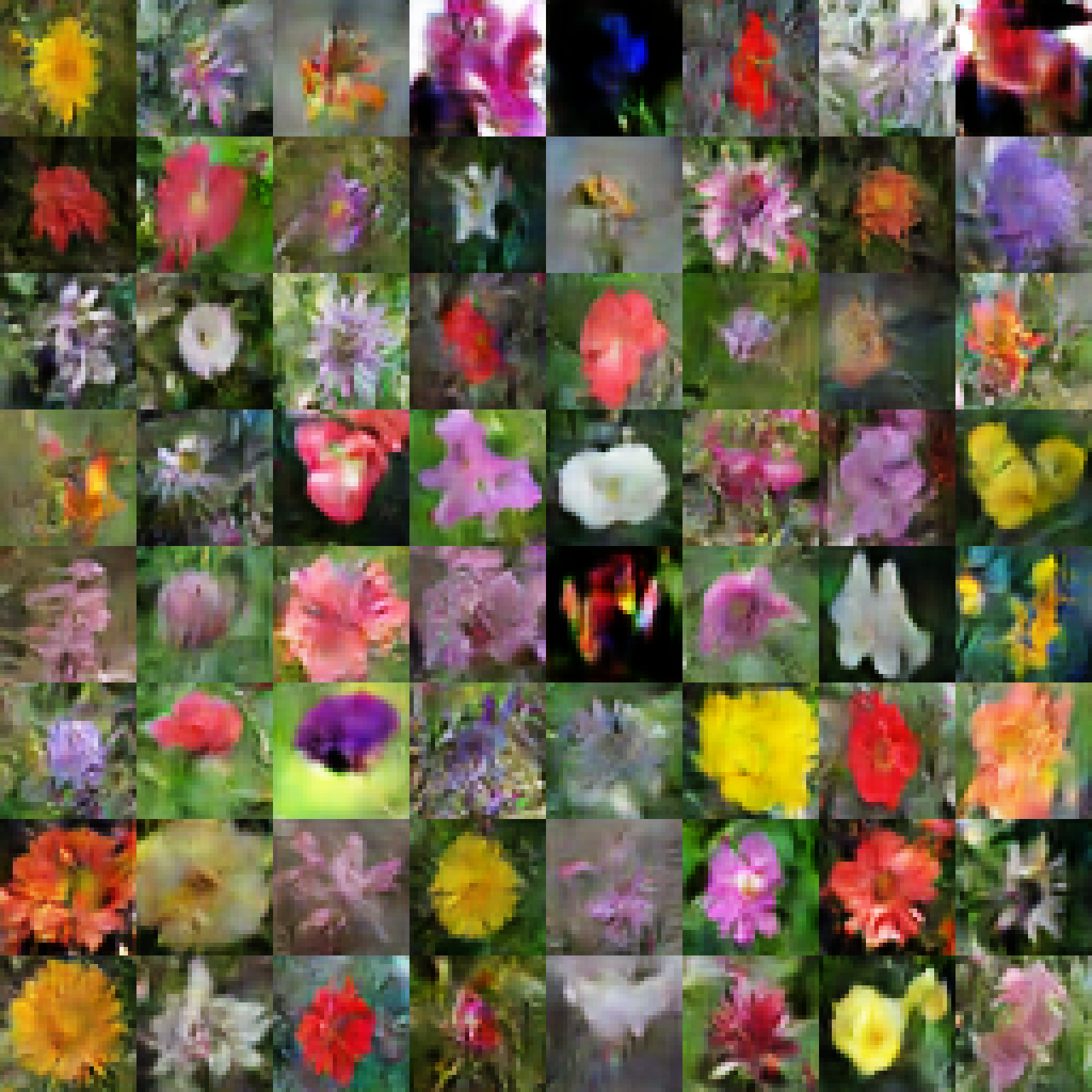} 
		\vspace{-3pt}
		\caption{$\max(0, x+1.0)$}
	\end{subfigure}	
	\vspace{-5pt}
	\caption{Random samples of LGANs with different loss metrics on Oxford 102.} 
	\label{fig_flowers}
\end{figure}

\begin{figure}[!htbp]
\vspace{5pt}
\centering
\includegraphics[width=0.75\columnwidth]{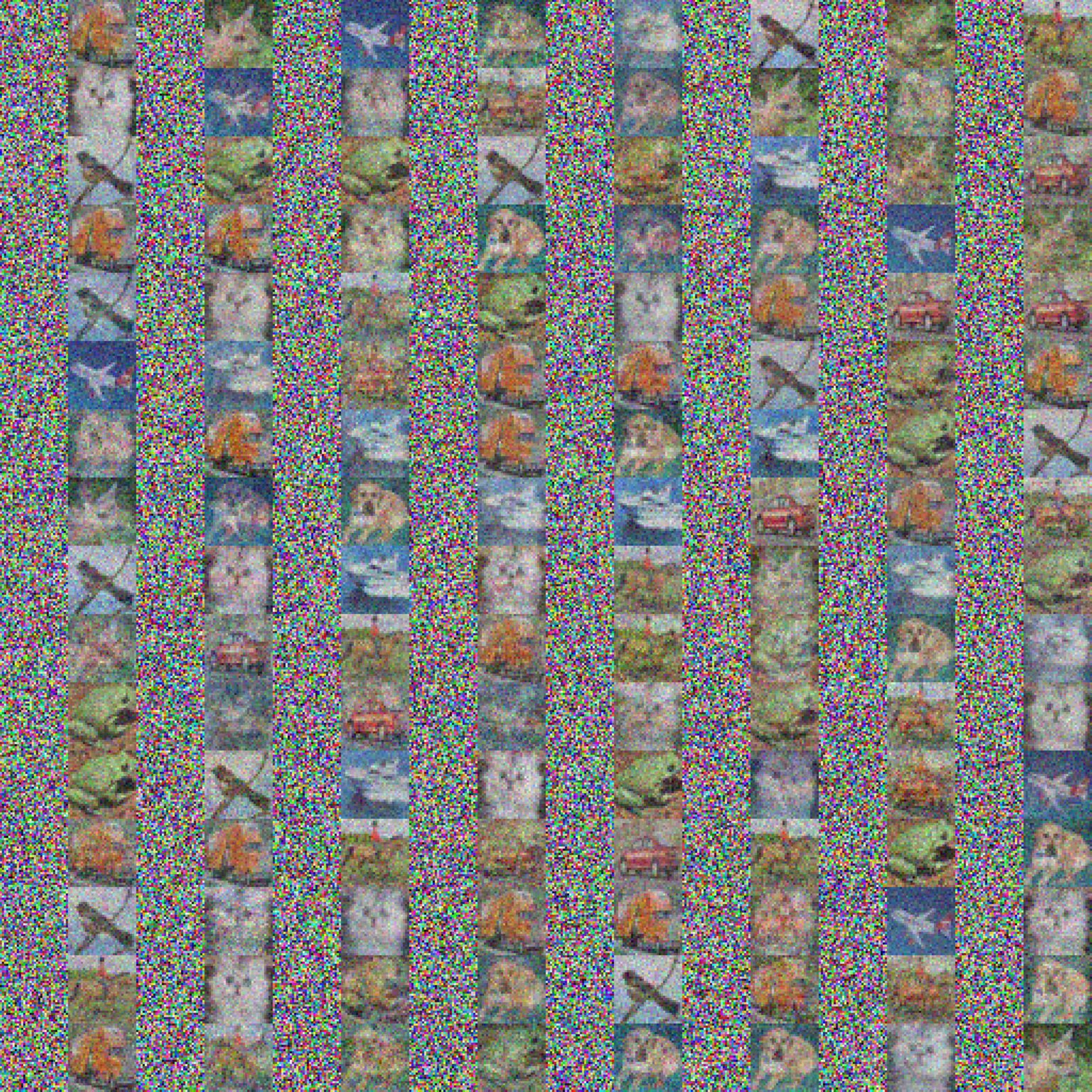}

\vspace{10pt}

\includegraphics[width=0.75\columnwidth]{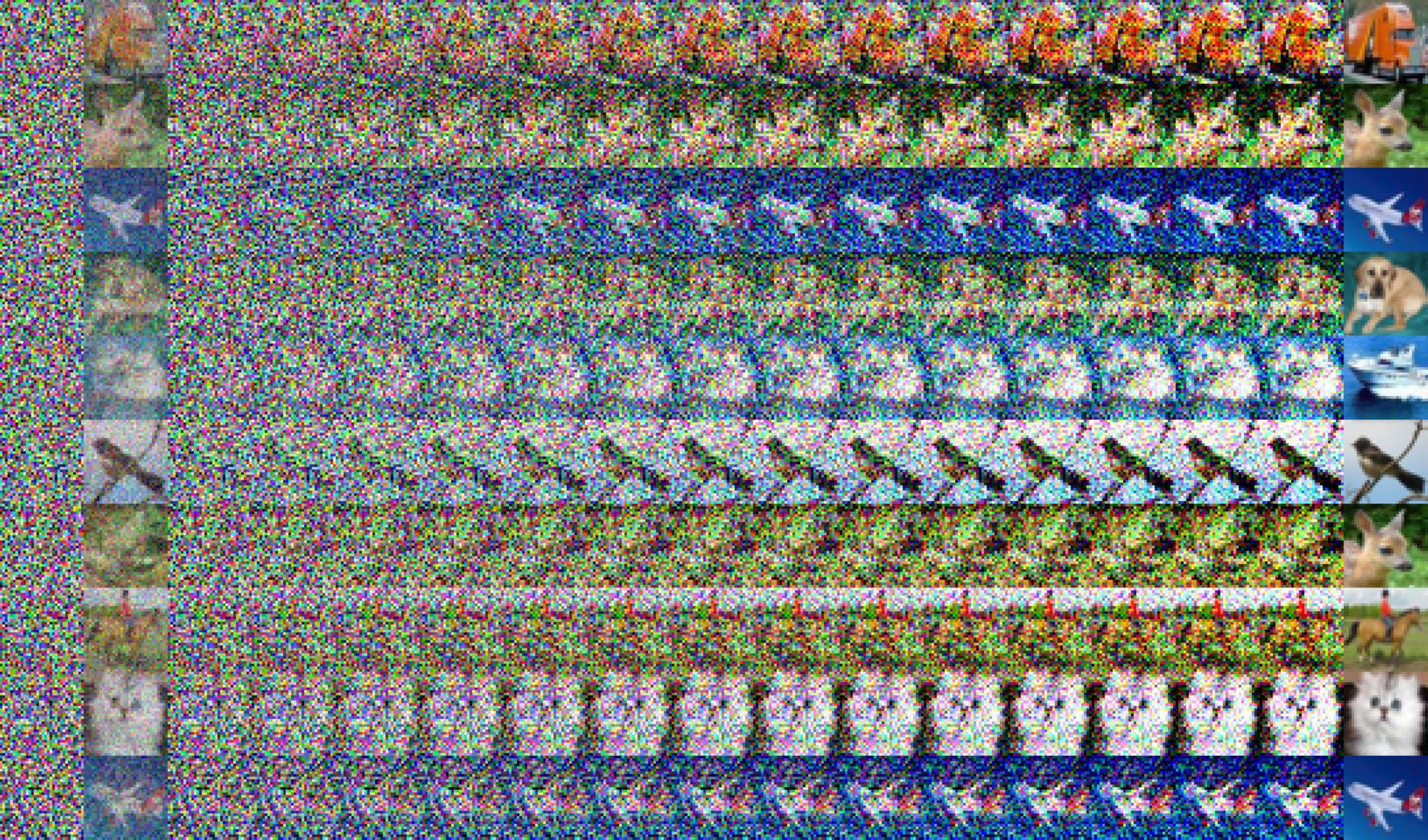}
\vspace{3pt}
\caption{The gradient of LGANs with real world data, where $\CP_r$ consists of ten images and $\CP_g$ is Gaussian noise. Up: Each odd column are $x \in \CS_g$ and the nearby column are their gradient $\nabla_{\!x} \ff(x)$. Down: the leftmost in each row is $x \in \CS_g$, the second are their gradients $\nabla_{\!x} \ff(x)$, the interiors are $x+\epsilon\cdot\nabla_{\!x} \ff(x)$ with increasing $\epsilon$, and the rightmost is the nearest $y \in \CS_r$.} 
\label{fig_gradient_direction_continuous}
\end{figure}

\end{document}